\documentclass{article}
\usepackage{arxiv_preprint}

\usepackage[utf8]{inputenc}
\usepackage[T1]{fontenc}
\usepackage{amsmath,amssymb,amsthm,mathtools}
\usepackage{booktabs}
\usepackage{multirow}
\usepackage{graphicx}
\usepackage{xcolor}
\usepackage{microtype}
\usepackage[colorlinks=true,linkcolor=blue!60!black,citecolor=blue!60!black,urlcolor=blue!60!black]{hyperref}

\newcommand{\Sm}[1]{\mathcal{S}^{#1}}
\newcommand{\clo}{\mathcal{C}}
\DeclareMathOperator{\clr}{clr}
\DeclareMathOperator{\ilr}{ilr}
\newcommand{\dA}{d_{A}}
\newcommand{\KL}{\mathrm{KL}}
\newcommand{\Hb}{H_{b}}
\newcommand{\scafS}{\mathcal{S}}
\newcommand{\contC}{\mathcal{C}}
\newcommand{\pS}{p^{\scafS}}
\newcommand{\pC}{p^{\contC}}
\newcommand{\HS}{H^{\scafS}}
\newcommand{\HC}{H^{\contC}}
\newcommand{\lek}{\Lambda}
\newcommand{\JS}{\mathrm{JS}}
\newcommand{\dvec}{\boldsymbol{\delta}}
\usepackage{nicefrac}
\usepackage{enumitem}
\setlist[enumerate]{itemsep=1pt,topsep=2pt,parsep=0pt,leftmargin=2.2em}
\setlist[itemize]{itemsep=1pt,topsep=2pt,parsep=0pt,leftmargin=1.4em}
\AtBeginDocument{\setlength{\abovedisplayskip}{5pt}\setlength{\belowdisplayskip}{5pt}%
  \setlength{\abovedisplayshortskip}{2pt}\setlength{\belowdisplayshortskip}{3pt}}

\newtheorem{theorem}{Theorem}
\newtheorem{proposition}{Proposition}
\newtheorem{lemma}{Lemma}
\newtheorem{corollary}{Corollary}
\theoremstyle{definition}
\newtheorem{definition}{Definition}
\theoremstyle{remark}
\newtheorem{remark}{Remark}

\title{What Does Chain-of-Thought Entropy Measure?\\ A Channel Audit of Scaffolding, Routing, and Content}

\author{
Marios Papamichalis\thanks{\raggedright Human Nature Lab, Yale University, New Haven, CT 06511, USA. \href{mailto:marios.papamichalis@yale.edu}{\nolinkurl{marios.papamichalis@yale.edu}}}
\and Regina Ruane\thanks{\raggedright Department of Statistics and Data Science, The Wharton School, University of Pennsylvania, Philadelphia, PA, USA. \href{mailto:ruanej@wharton.upenn.edu}{\nolinkurl{ruanej@wharton.upenn.edu}}}}

\begin{document}
\maketitle

\begin{abstract}
Entropy over chain-of-thought tokens decides which tokens receive the policy gradient, which get pruned, and whether a run has collapsed, yet each such statistic reads a next-token distribution mixing three choices: whether to emit connective scaffolding, which connective, and what the substantive continuation should be. Designating a scaffold vocabulary subset separates the three, exactly, for entropy, Kullback--Leibler divergence, and the first-order entropy velocity of a softmax policy. We prove the raw and content conventions disagree about which position is the larger fork on an explicit open region, and bound answer diversity by the content channel plus a leakage term a measured witness certifies. Across twenty-three configurations the scaffold side carries up to $41\%$ of the raw high-entropy set; on a matched-tokenizer ladder, coupling changes only at the math-corpus step while the scaffold's entropy share keeps growing through distillation; a closed-form forecast from one channel correlation tracks selection retention over a $54$-point range to five points, unfitted. On compression, the content convention beats raw surprisal in every cell; an answer-leakage audit then corrects our own headline control: re-fed chains earn a quarter to a half of their accuracy from restated answers, and once stripped, no token scorer beats a random contiguous block.
\end{abstract}

\section{Introduction}

A reasoning model that answers ``\emph{Wait, let me reconsider: the total is 12}'' has made three different kinds of choice. It chose to emit a connective instead of content (``Wait''), it chose \emph{which} connective, and it chose the substantive continuation (``12'' instead of ``7''). The next-token distribution at every position of a chain of thought (CoT) mixes these three choices into one point on the probability simplex, and the statistics that currently drive the analysis and training of reasoning models are computed on that mixture.

The mixture is consequential because the diagnostics built on it feed decisions. \citet{cui2025entropy} tie reinforcement-learning performance to policy entropy through a covariance mechanism. \citet{wang2025beyond} find that roughly $20\%$ of CoT tokens carry high entropy, that these ``forking tokens'' are overwhelmingly connectives such as ``wait'' and ``thus'', and that restricting policy gradients to them matches full-gradient training. Selection also prunes: TokenSkip ranks CoT tokens for compression \citep{xia2025tokenskip}, R-KV evicts KV-cache entries \citep{cai2025rkv}, and step attribution ranks sentences by attention received \citep{bogdan2025anchors}. Each time, a scalar computed on the mixed distribution decides which tokens get gradient, which get pruned, and which get called important.

\begin{figure}[t]
\centering
\includegraphics[width=0.76\textwidth]{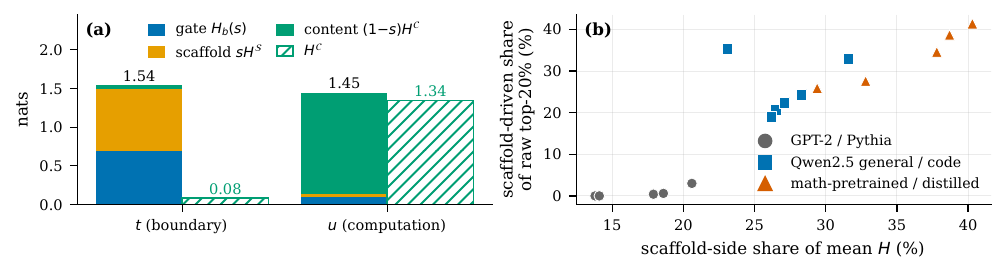}
\caption{\textbf{The conflation, exactly and at model scale.} (a) A two-position witness (closed form, Theorem~\ref{thm:reversal}): at a thought boundary $t$ the gate is a coin flip and the connective choice is open while the continuation is decided ($\HC=0.08$ nats); mid-computation at $u$ the scaffolding is quiet and the content forks ($\HC=1.34$). Raw entropy ranks $t$ above $u$, selecting the position whose content uncertainty is $16\times$ smaller. (b) Seventeen measured cells, \textsc{fixed} definition (Table~\ref{tab:frozen} lists fourteen): as the scaffold side takes more of mean $H$, the raw top-$20\%$ set becomes correspondingly scaffold-driven.}
\label{fig:punchline}
\end{figure}

Recent observations point at the mixture itself: entropy selection stops beating random pruning once formatting and numeric confounds are purified \citep{candussio2026demystifying}; formatting tokens destabilize importance rankings \citep{skipkv2025}; a positional heuristic matches a learned scorer in KV-cache eviction \citep{steele2026limits}; and realized trace tokens split into structural scaffolding and organic content well enough to drive compression \citep{zhao2026shorthand}, in apparent tension with connectives carrying \emph{high} entropy \citep{wang2025beyond}. The channels reconcile the tension: a boundary slot can have an uncertain \emph{gate} while its within-scaffold continuation is formulaic and its content is decided. These works classify tokens after the fact; none factorizes the predictive distribution itself, and none touches its training dynamics.

\textbf{Three estimands.} Designate a scaffold set $\scafS\subset V$ (connectives, discourse markers, formatting; \S\ref{sec:setup}) and write any next-token distribution as $p=(s\cdot\pS,\ (1-s)\cdot\pC)$, with $s=p(\scafS)$ the scaffold mass and $\pS,\pC$ the renormalized within-block distributions. Every question one asks of $p$ is then a question about the \emph{gate} (scaffold now, or content?), the \emph{scaffold channel} (which connective?), or the \emph{content channel} (which substantive continuation?). Raw entropy, top-$k$ sets, margins, and divergences answer an unstated mixture of the three. Our claim is that the three questions separate exactly, that the separation is measurable, and that it changes decisions on deployed pipelines.

\textbf{Contributions.}
\begin{enumerate}
\item \emph{Exact identities and an explicit disagreement region.} Entropy, Kullback--Leibler divergence, the Aitchison metric, and every weighted aggregate of CoT entropy split into gate, scaffold, and content parts (Proposition~\ref{thm:chain}), with sharp signed bounds on the surplus $G=H-\HC$. Raw and content entropy disagree about which of two positions is the larger fork exactly on the open wedge $\delta(\delta+\gamma)<0$, top-$\rho$ selections can be disjoint, and the retention and Jaccard overlap of the two selections are explicit functionals of the joint $(\HC,G)$ law (Theorem~\ref{thm:reversal}).
\item \emph{Leakage: what the scaffold can do for exploration.} In content-censored coordinates, trajectory entropy is exactly $\sum_t\mathbb{E}[\Hb(s_t)+(1-s_t)\HC_t]+\lek$ with $0\le\lek\le\sum_t\mathbb{E}[s_t\HS_t]$, and answer diversity obeys a matching bound with both sides attainable (Theorem~\ref{prop:exploration}). This converts the strongest objection to any scaffold/content split, that forks \emph{are} connectives, into a measurable quantity. We supply a one-sided witness $\kappa_t$ and measure it at $0.81$ nats on a distilled reasoning model, which bounds the inert-scaffold reading away.
\item \emph{Channel-resolved dynamics.} The covariance mechanism of \citet{cui2025entropy} splits exactly into a between-block term driven by the advantage gap between scaffold and content blocks plus mass-weighted within-block covariances. Under the exact categorical natural-gradient flow an advantage pattern constant on each block moves only the between-block term; under vanilla REINFORCE the same pattern additionally sharpens the scaffold channel and spreads the content channel (Proposition~\ref{prop:dynamics}, Corollary~\ref{cor:reinforce}). These describe the local logit velocity of a directly parameterized categorical policy; App.~\ref{app:proofs} separates that scope from shared-parameter network updates.
\item \emph{Measurement, a baseline-compared forecast, and a control that beats every scorer.} Across twenty-three configurations spanning thirteen checkpoints, two datasets, two evaluation modes, three scaffold definitions, and three seeds, the scaffold side carries $26$ to $41\%$ of the raw top-$20\%$ set on math-pretrained and distilled checkpoints against $0$ to $3\%$ on GPT-2 and Pythia (\textsc{fixed} definition; wider under \textsc{pos}, \S\ref{sec:stage1}). Four single-factor checkpoint contrasts at one tokenizer separate two axes occupancy statistics cannot: channel \emph{coupling} changes only at the math-corpus release, while the scaffold's entropy \emph{share} keeps rising through distillation (\S\ref{sec:ladder}). A zero-parameter forecast from the channel correlation halves a constant predictor's error and rank-orders the cells at Spearman $0.91$, with its one flip constant frozen and timestamped before three cells were measured (\S\ref{sec:forecast}). On CoT compression, scoring the content channel gains $2$ to $33$ points in five of five cells (clustered intervals), a factorial shows the gain is an interaction of score and eligibility, and an answer-leakage audit then corrects the positional control that beat every scorer: a quarter to a half of re-fed accuracy is restated-answer copying, and once stripped, no token scorer beats a random contiguous block (\S\ref{sec:decisions}). Two preregistered training interventions return nulls, reported as such (\S\ref{sec:training}).
\end{enumerate}

\section{Related work}
\label{sec:related}

\textbf{Entropy in RLVR.} \citet{cui2025entropy} give the collapse law and covariance mechanism; \citet{deng2025decomposing} decompose the entropy--performance exchange by stage, instance, and position; \citet{wang2025beyond} partition tokens by an entropy \emph{threshold}; \citet{tigerlab2025hierarchical} warn that most high-entropy tokens are not planning tokens; \citet{wu2025leash} show token entropy and answer diversity can diverge under RLVR; DASD routes a self-distillation loss by token entropy on top of GRPO \citep{zhang2026dasd}. None writes a mass/content identity; thresholding the mixture is what Theorem~\ref{thm:reversal} audits. Semantic entropy marginalizes surface variation at the answer level before measuring uncertainty \citep{kuhn2023semantic,farquhar2024detecting}; ours is the per-position exact counterpart, with the scaffold block as the surface class. \textbf{Structural and organic token splits.} \citet{zhao2026shorthand} classify realized tokens into structural scaffolding and organic content; the distinction here is the estimand, since they label emitted tokens while we factorize the next-token \emph{distribution}, its geometry, and its training-time velocity, then quantify through $\lek$ and $\kappa_t$ when the scaffold carries strategy. Concurrent lines: an explicit reasoning scaffold distills \citep{reasoningscaffold2025}; controlling logical connectives selects reasoning paths \citep{wherereasoningbreaks2026}, interventional evidence for the routing $\kappa_t$ measures; format constraints alone induce diversity collapse \citep{priceformat2025}, which localizes to post-training \citep{karouzos2026diversity}. \textbf{Scaffold tokens and intra-thought sinks.} Reflection tokens are information peaks \citep{qian2025thinking}; suppressing ``wait'' trades length for accuracy \citep{wang2025nowait}; the think delimiter is a reasoning attention sink \citep{li2026syncthink}, the BOS sink survives distillation \citep{zhang2025reasoning}, and the lineage runs back to pause and filler tokens \citep{goyal2024pause,pfau2024dot}. These document the phenomenon token by token; we treat it as a composition. \textbf{Compression, eviction, attribution.} TokenSkip \citep{xia2025tokenskip} prunes with a learned LLMLingua-2 importance classifier \citep{pan2024llmlingua2}; \citet{li2026stepentropy} prune whole steps by step entropy at $7$ to $14$B; R-KV evicts by redundancy-aware attention \citep{cai2025rkv}; SkipKV shows formatting destabilizes the rankings \citep{skipkv2025}; \citet{candussio2026demystifying} show entropy selectors collapse to random once confounds are purified, \citet{singh2026functional} ask whether attention encodes functional importance at all, and none of these evaluations strips restated answers from the re-fed chain, the inflation \S\ref{sec:decisions} measures. Thought Anchors attribute at sentence level and name the positional pressure on those scores while correcting for it only minimally \citep{bogdan2025anchors}. Positional baselines exist in neighbouring settings with opposite verdicts: recency windows are standard in KV-cache eviction \citep{xiao2024streamingllm,zhang2023h2o,li2024snapkv}, where \citet{steele2026limits} finds a positional heuristic matching a learned scorer, while EPiC finds head keeps above random above tail in training-data condensation \citep{jia2025epic}; \citet{garcia2026lastword} shows terminal answer statements confound suffix sensitivity. \S\ref{sec:decisions} runs the positional control in inference-time compression, where these pipelines do not. \textbf{Compositional data analysis.} Log-ratio geometry originates with \citet{aitchison1982}, partition balances with \citet{egozcue2005balances}, subcompositional coherence with \citet{aitchison1992,greenacre2011incoherence}; concurrent work finds log-ratio structure in attention \citep{lee2026invariants,zhu2026aitchisonattention,yamada2026polyilr}. The CoT chain rule, dynamics decomposition, exploration bound, and re-audit protocol are new here. \textbf{RLVR context.} GRPO \citep{shao2024deepseekmath}, R1-style distillation \citep{deepseek2025r1}, open small-scale recipes \citep{zeng2025simplerl,tinyzero2025}, and the pass@$k$ debate \citep{yue2025limit} fix the setting; GSM8K and MATH \citep{cobbe2021gsm8k,hendrycks2021math} the tasks.

\section{Setup: three channels of a chain-of-thought distribution}
\label{sec:setup}

Let $V$ be the vocabulary, $|V|=D$, and let $p_t=\pi_\theta(\cdot\mid x,y_{<t})$ be the next-token distribution at CoT position $t$. Fix a nonempty proper \emph{scaffold set} $\scafS\subset V$, write $\contC=V\setminus\scafS$, $m=|\scafS|$, $k=|\contC|$. Every $p$ with $0<p(\scafS)<1$ factors uniquely as
\begin{equation}
p \;=\; \big(s\cdot \pS,\ (1-s)\cdot \pC\big),\qquad s=p(\scafS),\quad \pS=\clo(p|_{\scafS}),\quad \pC=\clo(p|_{\contC}),
\label{eq:factor}
\end{equation}
where $\clo$ renormalizes. We call $s$ the \emph{gate mass}, $\pS$ the \emph{scaffold channel}, and $\pC$ the \emph{content channel}; \emph{content} names the complement block $\contC$ and carries no semantic claim that every non-scaffold token is substantive (Remark~\ref{rem:partition}). For softmax policies $p=\clo(e^{z})$ this factorizes the logits, with gate log-odds $\log\frac{s}{1-s}=\mathrm{lse}(z|_{\scafS})-\mathrm{lse}(z|_{\contC})$. Entropy is in nats with $0\log 0=0$; $\Hb(s)=-s\log s-(1-s)\log(1-s)$; $\HS=H(\pS)$, $\HC=H(\pC)$, and the \emph{scaffold surplus} is $G=H-\HC$.

\begin{remark}[Any finite partition]
\label{rem:partition}
Nothing below is specific to two blocks: any finite partition $V=\bigsqcup_b B$ gives $H(p)=H(\text{block})+\sum_b m_b H(p^{B})$, with divergence, geometry, and dynamics analogues. The scaffold list is tiered; App.~\ref{app:scaffold} reports the three-block refinement \emph{formatting / strategic-control / content}, which answers the objection that ``Wait'' can be strategy while a newline is not.
\end{remark}

\textbf{The scaffold set is a modelling choice, and the mathematics does not depend on it.} Every identity below holds for \emph{any} $\scafS$; what depends on $\scafS$ is which question the channels answer. We consider three constructions (App.~\ref{app:scaffold}): \textsc{fixed}, a curated list of $71$ reasoning connectives plus formatting tokens, matched across tokenizer variants, mathematical operators excluded; \textsc{pos}, adding a closed-class function-word list; \textsc{data}, the most frequent non-numeric token types of the model's own CoTs. Representative cells run under all three and principal numbers carry their sensitivity (App.~\ref{app:additional}).

\section{Theory}
\label{sec:theory}

Proofs, formal statements with equality cases, and the remaining instrument decompositions are in Apps.~\ref{app:proofs} and~\ref{app:instruments}. Statements (i) and (ii) below are the classical grouping and chain identities \citep{cover2006elements}, which is why they are a proposition; the load-bearing results are the reversal geometry and the leakage bound.

\begin{proposition}[Channel identities and sharp surplus bounds]
\label{thm:chain}
Let $p,q$ have gate masses $s,t\in(0,1)$. Then
\begin{enumerate}
\item[\rm(i)] $H(p)=\Hb(s)+s\,\HS(p)+(1-s)\,\HC(p)$;
\item[\rm(ii)] $\KL(p\|q)=\KL_{\mathrm b}(s\|t)+s\,\KL(\pS\|q^{\scafS})+(1-s)\,\KL(\pC\|q^{\contC})$ as an identity in $[0,\infty]$;
\item[\rm(iii)] $G(p)=\Hb(s)+s(\HS(p)-\HC(p))$, and over all $p$ with $s$ fixed, $\min G=\Hb(s)-s\log k$, attained exactly when $\pS$ is a point mass and $\pC$ is uniform, and $\max G=\Hb(s)+s\log m$, attained exactly when $\pS$ is uniform and $\pC$ is a point mass;
\item[\rm(iv)] every weighted sum $\sum_j\alpha_jH(p_j)$, and every stopped weighted sum $\mathbb{E}\sum_{j\le N}W_jH(P_j)$ with $\mathbb{E}\sum_{j\le N}|W_j|<\infty$, splits channelwise by (i), with no independence or stopping-time assumption.
\end{enumerate}
\end{proposition}

Part (iv) makes the decomposition usable on training curves: any collapse curve or weighted aggregate splits without further assumptions, and the bounds in (iii) put the surplus at a few nats, the scale forking-token thresholds live on. The same partition splits the Aitchison geometry into a balance coordinate and two within-block coordinates (Lemma~\ref{lem:pyth}), and at a fixed prefix $\pC$ is the maximal invariant of scaffold-supported logit perturbations, so a per-position statistic ignores those perturbations if and only if it factors through $\pC$ (Proposition~\ref{prop:invariance}). Trajectory-level independence is a separate question, and it is exactly the leakage of Theorem~\ref{prop:exploration}.

\begin{definition}[Exact collapse attribution, and its scope]
\label{def:attr}
Between two snapshots $(s_0,\HS_0,\HC_0)$ and $(s_1,\HS_1,\HC_1)$ of the \emph{same context}, set $\Delta H_{\mathrm{cont}}:=(1-\tfrac{s_0+s_1}{2})(\HC_1-\HC_0)$ and $\Delta H_{\mathrm{scaf}}:=\Delta H-\Delta H_{\mathrm{cont}}$: the average of the two update orders of the exact telescoping of Proposition~\ref{thm:chain}(i), exact per matched context and on average over any fixed context set, \emph{not} exact on separately averaged $(\bar s,\overline{\HS},\overline{\HC})$; contexts are unmatched across on-policy checkpoints. The protocol therefore attributes on a fixed teacher-forced context set and reports, on-policy, only the linear aggregates $\Delta\mathbb{E}[\Hb]$, $\Delta\mathbb{E}[s\HS]$, $\Delta\mathbb{E}[(1-s)\HC]$, which sum to $\Delta\mathbb{E}[H]$ identically.
\end{definition}

\subsection{Ranking reversals and the selection wedge}

Forking-token selection, entropy-based pruning, and step attribution all rank positions. Ranking by raw $H$ and ranking by $\HC$ are different orders, and the disagreement region is explicit.

\begin{theorem}[Fork-verdict reversal]
\label{thm:reversal}
Let $t,u$ be positions with $|\contC|\ge2$, and set $\delta:=\HC_t-\HC_u$ and $\gamma:=G_t-G_u$, so that $H_t-H_u=\delta+\gamma$. Away from ties,
\begin{enumerate}
\item[\rm(i)] the raw and content conventions disagree about which position is the larger fork exactly on the wedge $\delta(\delta+\gamma)<0$, an open set of positive Lebesgue measure in the joint parameters;
\item[\rm(ii)] there are position populations for which the top-$\rho$ sets under the two conventions are disjoint for every $\rho\le1/2$;
\item[\rm(iii)] for any population of positions, the pairwise flip probability and the expected retention and Jaccard overlap of the two top-$\rho$ sets are explicit functionals of the joint law of $(\HC,G)$.
\end{enumerate}
\end{theorem}

Part (iii) is what makes the theorem predictive as well as diagnostic. If $(H,\HC)$ is jointly Gaussian across positions with correlation $r$, then the pairwise flip probability is $\tfrac12-\tfrac{1}{\pi}\arcsin r$ and top-$\rho$ retention is a bivariate orthant probability, so both are determined by $r$ alone. We register that forecast in \S\ref{sec:forecast}, test it against twenty-three measured configurations, and score its one constant on three of them that were held out.

\subsection{What the scaffold can and cannot do for exploration}

The strongest objection to any scaffold/content split is that forks \emph{are} connectives, so the scaffold channel is where the strategy lives. The objection is correct in principle, and the next result says how much diversity the scaffold can carry, through what mechanism, and how to bound it. Let $Z_t$ be the content-censored coordinate at step $t$ (the token if it is content, a single symbol recording ``scaffold'' otherwise), and $\alpha$ the answer.

\begin{theorem}[Exploration bound with exact leakage]
\label{prop:exploration}
With $\lek:=\sum_t I(Z_t;Y_{<t}\mid Z_{<t})$,
\begin{equation}
H(Z_{1:T})=\sum_t\mathbb{E}\big[\Hb(s_t)+(1-s_t)\HC_t\big]+\lek,
\qquad
0\le\lek\le\sum_t\mathbb{E}\big[s_t\HS_t\big],
\label{eq:leak}
\end{equation}
and $H(\alpha)\le\sum_t\mathbb{E}[\Hb(s_t)+(1-s_t)\HC_t]+\lek+H(\alpha\mid Z_{1:T})$. The residual $H(\alpha\mid Z_{1:T})$ vanishes exactly when the answer is a function of the content-censored trajectory, every equality case is characterized, and both bounds are simultaneously attainable.
\end{theorem}

The scaffold-internal channel buys answer diversity only through history leakage, and a coin flipped \emph{inside} the scaffold can carry a full bit of it. Estimating $\lek$ is open; a one-sided witness is measurable now. Let $\kappa_t$ be the mutual information between the connective emitted at $t$ and the first content token that follows: $\kappa_t>0$ at any prefix implies $\lek>0$, and the converse fails, since routing acting only beyond the first content token leaves $\kappa_t=0$ with $\lek>0$ (App.~\ref{app:proofs}). So $\kappa_t$ certifies leakage without estimating it, the direction the inertness question needs.

\subsection{Channel-resolved entropy dynamics}

\begin{proposition}[Channel dynamics and two optimizer regimes]
\label{prop:dynamics}
For a softmax policy the first-order entropy velocity splits exactly into a \emph{between-block} term proportional to the advantage gap between the scaffold and content blocks plus mass-weighted \emph{within-block} covariance terms, with the gate, block, and block-entropy velocities in closed form.
\end{proposition}

\begin{corollary}[Which channels an advantage pattern moves]
\label{cor:reinforce}
Under the exact categorical natural-gradient flow, an advantage that is constant on each block moves only the between-block term, and the sign of that term is characterized by the advantage gap. Under vanilla REINFORCE the same advantage pattern additionally sharpens the scaffold channel and applies a first-order \emph{spreading} pressure to the content channel, by exact comonotone-covariance formulas. A reward that depends only on the block sequence need not induce a block-constant conditional advantage, so these are statements about the advantage pattern at a position; the scope remark in App.~\ref{app:proofs} separates both regimes from shared-parameter network updates.
\end{corollary}

Corollary~\ref{cor:reinforce} has a practical reading. A format reward moves measured entropy through channels exploration does not live in, so a collapse curve computed on the mixture can fall while content exploration is untouched; it also predicts arm-indifference for channel-selected gradients under sparse reward, which \S\ref{sec:training} measures. A second corollary (App.~\ref{app:instruments}) carries the accounting into the RLVR KL penalty: a fixed coefficient charges content drift at weight $1-s$, so hardening gates linearly weaken the content trust region.

\section{Experiments}
\label{sec:experiments}
\label{sec:protocol}

Every number below was computed with fixed seeds, and the identities are machine-checked ($215$ checks, App.~\ref{app:expdetails}). Numerics matter: in float32 a gate mass that rounds to $1$ makes the naive content estimator return $\infty-\infty$, so the estimator sums the content block directly, accumulates in float64, and identifies positions whose content channel is undefined; those are excluded from rankings and their fraction reported (at most $0.2\%$ teacher-forced, $4.9\%$ generated).

\subsection{Stage 1: where CoT entropy lives}
\label{sec:stage1}

\begin{table}[t]
\centering
\scriptsize\setlength{\tabcolsep}{3.4pt}\renewcommand{\arraystretch}{0.84}
\caption{\textbf{Channel decomposition of CoT entropy on frozen models} (mean$\pm$sem over $3$ seeds, \textsc{fixed} scaffold; \textsc{pos}/\textsc{data} sensitivity in App.~\ref{app:additional}). The middle block holds the tokenizer and the context set fixed, so each indented row differs from Qwen2.5-1.5B by the single intervention named and the last two share the Math-1.5B parent. Shares are of mean $H$; flip is the fork-verdict flip rate on random position pairs; ret@$20$ is top-$20\%$ retention between conventions; scaf-drv is the scaffold-driven share of the raw top-$20\%$ set; real-$\scafS$ is its realized-scaffold share, the \citet{wang2025beyond} connective observation quantified per channel.}
\label{tab:frozen}
\begin{tabular}{lccccccccc}
\toprule
model (data, mode) & $\bar s$ & $\bar H$ & gate\% & scaf\% & cont\% & flip\% & ret@20\% & scaf-drv\% & real-$\scafS$\% \\
\midrule
\multicolumn{10}{l}{\emph{base models, teacher-forced on GSM8K reference solutions} (Pythia-70M, -160M, and -1B in App.~\ref{app:additional})}\\
GPT-2 & $0.16$ & $2.89$ & $7.8{\scriptstyle\pm 0.0}$ & $6.3{\scriptstyle\pm 0.0}$ & $86.0{\scriptstyle\pm 0.0}$ & $13.1{\scriptstyle\pm 0.2}$ & $69.1{\scriptstyle\pm 0.2}$ & $0.0{\scriptstyle\pm 0.0}$ & $13.2{\scriptstyle\pm 0.1}$ \\
Pythia-410M & $0.18$ & $2.25$ & $9.8{\scriptstyle\pm 0.0}$ & $8.1{\scriptstyle\pm 0.0}$ & $82.0{\scriptstyle\pm 0.0}$ & $12.5{\scriptstyle\pm 0.1}$ & $70.1{\scriptstyle\pm 0.1}$ & $0.4{\scriptstyle\pm 0.0}$ & $13.4{\scriptstyle\pm 0.2}$ \\
Pythia-1.4B & $0.19$ & $1.86$ & $10.9{\scriptstyle\pm 0.0}$ & $9.7{\scriptstyle\pm 0.0}$ & $79.4{\scriptstyle\pm 0.0}$ & $12.9{\scriptstyle\pm 0.1}$ & $65.6{\scriptstyle\pm 0.2}$ & $3.0{\scriptstyle\pm 0.0}$ & $15.9{\scriptstyle\pm 0.2}$ \\
\midrule
\multicolumn{10}{l}{\emph{one tokenizer, one teacher-forced context set: corpus, scale, tuning, and distillation vary}}\\
Qwen2.5-1.5B & $0.19$ & $0.49$ & $18.8{\scriptstyle\pm 0.1}$ & $7.7{\scriptstyle\pm 0.0}$ & $73.5{\scriptstyle\pm 0.1}$ & $16.1{\scriptstyle\pm 0.2}$ & $52.1{\scriptstyle\pm 0.2}$ & $20.7{\scriptstyle\pm 0.1}$ & $18.6{\scriptstyle\pm 0.2}$ \\
\quad + code corpus (Coder-1.5B) & $0.19$ & $0.54$ & $18.3{\scriptstyle\pm 0.1}$ & $7.9{\scriptstyle\pm 0.0}$ & $73.8{\scriptstyle\pm 0.1}$ & $15.9{\scriptstyle\pm 0.2}$ & $52.0{\scriptstyle\pm 0.3}$ & $19.0{\scriptstyle\pm 0.2}$ & $19.0{\scriptstyle\pm 0.3}$ \\
\quad + $2\times$ parameters (3B) & $0.19$ & $0.44$ & $19.4{\scriptstyle\pm 0.1}$ & $7.7{\scriptstyle\pm 0.0}$ & $72.9{\scriptstyle\pm 0.2}$ & $16.2{\scriptstyle\pm 0.1}$ & $53.9{\scriptstyle\pm 0.3}$ & $22.4{\scriptstyle\pm 0.3}$ & $19.2{\scriptstyle\pm 0.2}$ \\
\quad + instruction tuning & $0.19$ & $0.45$ & $19.4{\scriptstyle\pm 0.1}$ & $8.9{\scriptstyle\pm 0.0}$ & $71.7{\scriptstyle\pm 0.1}$ & $15.4{\scriptstyle\pm 0.1}$ & $54.2{\scriptstyle\pm 0.3}$ & $24.3{\scriptstyle\pm 0.1}$ & $21.1{\scriptstyle\pm 0.1}$ \\
\quad + math corpus (Math-1.5B) & $0.20$ & $0.71$ & $18.1{\scriptstyle\pm 0.1}$ & $14.7{\scriptstyle\pm 0.0}$ & $67.2{\scriptstyle\pm 0.1}$ & $13.9{\scriptstyle\pm 0.2}$ & $58.6{\scriptstyle\pm 0.5}$ & $27.6{\scriptstyle\pm 0.3}$ & $30.5{\scriptstyle\pm 0.1}$ \\
\quad\quad + instruction tuning & $0.19$ & $0.77$ & $15.8{\scriptstyle\pm 0.1}$ & $13.6{\scriptstyle\pm 0.1}$ & $70.6{\scriptstyle\pm 0.2}$ & $13.1{\scriptstyle\pm 0.1}$ & $58.7{\scriptstyle\pm 0.5}$ & $25.9{\scriptstyle\pm 0.3}$ & $32.9{\scriptstyle\pm 0.2}$ \\
\quad\quad + R1 distillation & $0.23$ & $0.79$ & $19.4{\scriptstyle\pm 0.1}$ & $20.9{\scriptstyle\pm 0.0}$ & $59.7{\scriptstyle\pm 0.1}$ & $13.7{\scriptstyle\pm 0.0}$ & $60.0{\scriptstyle\pm 0.1}$ & $41.4{\scriptstyle\pm 0.2}$ & $28.5{\scriptstyle\pm 0.1}$ \\
\midrule
\multicolumn{10}{l}{\emph{models generating their own chains of thought}}\\
Qwen2.5-1.5B (GSM8K) & $0.33$ & $0.09$ & $10.9{\scriptstyle\pm 0.3}$ & $12.2{\scriptstyle\pm 0.2}$ & $76.9{\scriptstyle\pm 0.5}$ & $22.7{\scriptstyle\pm 0.6}$ & $16.1{\scriptstyle\pm 0.6}$ & $35.3{\scriptstyle\pm 1.2}$ & $32.9{\scriptstyle\pm 0.7}$ \\
Qwen2.5-1.5B-It (GSM8K) & $0.26$ & $0.26$ & $20.0{\scriptstyle\pm 0.1}$ & $11.6{\scriptstyle\pm 0.1}$ & $68.5{\scriptstyle\pm 0.1}$ & $17.9{\scriptstyle\pm 0.2}$ & $47.6{\scriptstyle\pm 0.3}$ & $32.9{\scriptstyle\pm 0.2}$ & $22.8{\scriptstyle\pm 0.1}$ \\
R1-Distill-1.5B (GSM8K) & $0.28$ & $0.56$ & $20.3{\scriptstyle\pm 0.1}$ & $17.5{\scriptstyle\pm 0.1}$ & $62.2{\scriptstyle\pm 0.1}$ & $15.5{\scriptstyle\pm 0.1}$ & $57.1{\scriptstyle\pm 0.3}$ & $34.6{\scriptstyle\pm 0.2}$ & $31.8{\scriptstyle\pm 0.1}$ \\
R1-Distill-1.5B (MATH-500) & $0.28$ & $0.45$ & $19.4{\scriptstyle\pm 0.1}$ & $19.3{\scriptstyle\pm 0.1}$ & $61.2{\scriptstyle\pm 0.2}$ & $15.0{\scriptstyle\pm 0.0}$ & $58.8{\scriptstyle\pm 0.1}$ & $38.7{\scriptstyle\pm 0.2}$ & $33.4{\scriptstyle\pm 0.1}$ \\
\bottomrule
\end{tabular}
\end{table}

We measure Pythia $70$M to $1.4$B and GPT-2 teacher-forced on GSM8K reference solutions, and six Qwen2.5 checkpoints with R1-Distill-Qwen-1.5B, teacher-forced and generating, on GSM8K and MATH-500 (Table~\ref{tab:frozen}); predictions were fixed before measurement (App.~\ref{app:expdetails}).

Two of the three pre-stated predictions hold, with a scope correction the data force. The scaffold-driven share of the raw top-$20\%$ set comes in three tiers: $0$ to $3\%$ on teacher-forced GPT-2 and Pythia, $19$ to $24\%$ on the Qwen2.5 general, code, and instruct checkpoints, $26$ to $41\%$ on the math-pretrained and distilled ones. The cross-family contrast confounds family and era with tuning, so the controlled comparison is the within-family ladder of \S\ref{sec:ladder}, a $20.7\to41.4\%$ climb from base to distilled at a gate share near $20\%$ of mean $H$ against $8$ to $11\%$ for GPT-2 and Pythia. The third prediction, that $\HC$ would be flatter across positions than $H$, is false in the informative direction on the reasoning models generating their own chains: there $\operatorname{Var}(G)$ and $\operatorname{Var}(\HC)$ each exceed $\operatorname{Var}(H)$ by $2.0$ to $4.1\times$ with $\operatorname{corr}(G,\HC)$ in $[-0.87,-0.77]$, while on teacher-forced base models only $\operatorname{Var}(\HC)$ exceeds $\operatorname{Var}(H)$ (App.~\ref{app:additional}). Since $G\equiv H-\HC$ forces the three terms to sum to $\operatorname{Var}(H)$, near-cancellation is automatic once they are large; the measurement is their magnitude, and it says the raw statistic is the smaller residual of two larger opposed channels.

\subsection{Coupling and allocation are two axes, and four checkpoint contrasts separate them}
\label{sec:ladder}

\begin{figure}[t]
\centering
\includegraphics[width=0.70\textwidth]{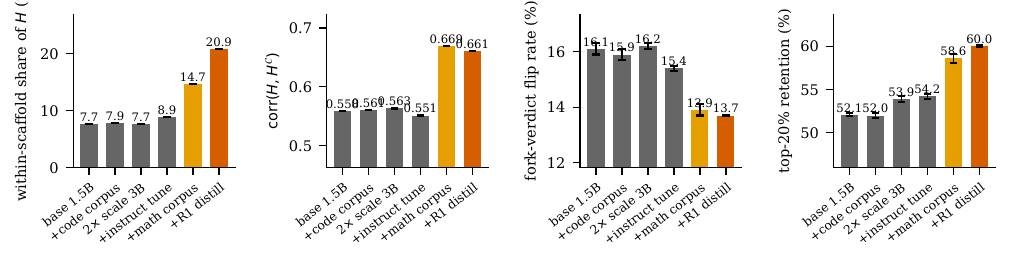}
\caption{\textbf{Four checkpoint contrasts on one tokenizer, and only two move the channels.} Teacher-forced GSM8K, three seeds, identical context set and token budget, mean$\pm$sem; each bar is the released checkpoint differing from Qwen2.5-1.5B by the factor named, uncontrolled in whatever else that release changed. Code data, $2\times$ parameters, and instruction tuning leave \emph{coupling} (panels 2 and 3) inside noise and move retention (panel 4) by at most $2.1$ points against $6.5$ for the math corpus; distillation on the math corpus moves only \emph{allocation} (panel 1).}
\label{fig:ladder}
\end{figure}

Comparing a distilled reasoning model to a base model varies many things at once. Holding the tokenizer and the context set fixed and changing one named factor per released checkpoint narrows it (Figure~\ref{fig:ladder}): four contrasts against Qwen2.5-1.5B, adding code data (Coder-1.5B), doubling parameters (Qwen2.5-3B), instruction tuning, and adding math data (Math-1.5B), then R1 distillation on the math corpus. These are released checkpoints, not randomized interventions, so each contrast carries whatever else its release changed.

Three of the four leave the channel geometry where it was. The channel correlation $\operatorname{corr}(H,\HC)$ moves by $+0.003$ under the code corpus, $+0.005$ under $2\times$ scale, and $-0.007$ under instruction tuning, against $+0.111$ under the math corpus; the flip rate moves by at most $0.7$ points against $2.2$; retention moves by $-0.1$, $+1.8$, and $+2.1$ points against $+6.5$. The one contrast that changes what the wedge geometry of Theorem~\ref{thm:reversal} predicts is the math-corpus release, and neither the scale nor the tuning contrast substitutes for it.

Distillation then moves the other axis alone. From Math-1.5B to its R1-distilled student the correlation goes $0.669\to0.661$, flip $13.9\to13.7\%$, retention $58.6\to60.0\%$ (Welch $p=0.71$, $p=0.08$), while the scaffold's share of mean $H$ climbs $14.7\to20.9\%$ ($p<10^{-4}$); instruction tuning on the same parent moves that share the other way, $14.7\to13.6\%$, so the rise appears at the distillation step and not at post-training in general. A parent and student indistinguishable in flip and retention differ by $6.2$ points of scaffold share, the sharpest evidence that the decomposition is not redundant with occupancy statistics. One quantity cuts the other way and we report it: the scaffold-driven share of the raw top-$20\%$ set runs $20.7\to27.6\to41.4\%$, weighting distillation more heavily than the corpus.

\subsection{A forecast with no fitted parameters}
\label{sec:forecast}

Before running the selection audit we registered the forecast of Theorem~\ref{thm:reversal}(iii): feed each configuration's Stage-1 channel correlation $r$ through the Gaussian orthant formulas and predict its flip rate and top-$20\%$ retention, using no quantity from the audit it predicts.

Retention lands within $5.1$ points on average (Spearman $0.91$) across cells spanning two orders of magnitude in model size, two datasets, teacher-forced and generated modes, and three scaffold definitions (App.~\ref{app:additional}, Figure~\ref{fig:forecast}a). The flip rate is rank-calibrated at Spearman $0.82$ and over-predicted in level in every cell, by a stable ratio of mean $0.59$ and standard deviation $0.10$. The uniform sign and stable ratio point at the Gaussian copula: real position populations, with heavier tails and a mass of near-ties, swap order less often than a Gaussian of the same correlation.

\textbf{Trivial baselines, and a frozen constant.} A forecast is only as good as the baseline it beats, so we score two constants on the same cells. Predicting every cell at the development mean gives retention error $12.1$ points over the twenty development cells against the forecast's $5.1$; the median-flip constant gives $4.0$ against the recalibrated $2.8$; and a constant ranks nothing, while the forecast rank-orders the cells at Spearman $0.91$ and $0.82$. The one fitted quantity, the flip constant $c=0.5988$, was frozen and timestamped with the development cells on 2026-08-24; three sibling cells measured afterwards (Qwen2.5-Coder-1.5B, -3B, -Math-1.5B-Instruct) score $2.9$ and $5.5$ points with nothing refitted (Figure~\ref{fig:forecast}, diamonds). Because those cells sit near the development centre, the constants are accidentally closer on them ($1.2$ and $2.4$); three same-family in-range cells are a calibration check, and the claim is level-and-rank calibration over all twenty-three cells against baselines, not external validation.

\subsection{Decisions: compression, attribution, eviction}
\label{sec:decisions}

\begin{table}[t]
\centering
\scriptsize\setlength{\tabcolsep}{4pt}\renewcommand{\arraystretch}{0.80}
\caption{\textbf{Post-hoc CoT compression: answer accuracy (\%) at matched token budgets}, mean$\pm$sem over three seeds ($512$ problems per seed on GSM8K, $256$ on MATH-500); every arm keeps the same number of tokens in every cell. \textsc{k} is raw surprisal, the convention beneath LLMLingua-style importance scores \citep{pan2024llmlingua2} (TokenSkip itself deploys a learned LLMLingua-2 classifier \citep{xia2025tokenskip}); \textsc{k}$^{\contC}$ and \textsc{k}$^{g}$ are the factorial arms (raw score with content-only eligibility; gate-adjusted score with all tokens eligible); \textsc{a} is the receiver-attention scorer \citep{bogdan2025anchors}, \textsc{d} the content-channel scorer, \textsc{last} the final tokens of the chain. The lower panel reruns every arm after deleting each restatement of the gold answer and the chain's final sentence; raw and content entropy track their surprisal twins within a point and are omitted there.}
\label{tab:compress}
\begin{tabular}{llcccccc}
\toprule
& & \multicolumn{3}{c}{R1-Distill-1.5B, GSM8K} & R1-Distill, MATH & Qwen2.5-1.5B-It \\
\cmidrule(lr){3-5}\cmidrule(lr){6-6}\cmidrule(lr){7-7}
& arm & keep $.25$ & keep $.50$ & keep $.75$ & keep $.50$ & keep $.50$ \\
\midrule
\multirow{5}{*}{\rotatebox{90}{\emph{scorers}}}
& \textsc{k} raw surprisal & $12.2{\scriptstyle\pm0.3}$ & $20.2{\scriptstyle\pm0.4}$ & $46.2{\scriptstyle\pm0.6}$ & $15.9{\scriptstyle\pm0.7}$ & $25.8{\scriptstyle\pm1.4}$ \\
& \textsc{e} raw entropy & $12.1{\scriptstyle\pm0.3}$ & $19.9{\scriptstyle\pm0.8}$ & $46.1{\scriptstyle\pm0.9}$ & $15.4{\scriptstyle\pm0.3}$ & $25.7{\scriptstyle\pm1.2}$ \\
& \textsc{a} attention received & $14.5{\scriptstyle\pm1.0}$ & $25.3{\scriptstyle\pm0.3}$ & $44.5{\scriptstyle\pm1.3}$ & $19.1{\scriptstyle\pm1.0}$ & $11.5{\scriptstyle\pm0.5}$ \\
& \textsc{k}$^{\contC}$ raw score, content-only & $12.9{\scriptstyle\pm0.7}$ & $25.5{\scriptstyle\pm0.7}$ & $74.4{\scriptstyle\pm0.7}$ & $21.9{\scriptstyle\pm1.4}$ & $16.0{\scriptstyle\pm0.2}$ \\
& \textsc{k}$^{g}$ gate-adj score, all & $12.4{\scriptstyle\pm0.7}$ & $25.1{\scriptstyle\pm1.5}$ & $70.1{\scriptstyle\pm1.2}$ & $22.0{\scriptstyle\pm1.3}$ & $15.7{\scriptstyle\pm0.6}$ \\
& \textsc{d} content surprisal & $14.2{\scriptstyle\pm0.3}$ & $52.0{\scriptstyle\pm0.8}$ & $73.0{\scriptstyle\pm0.7}$ & $25.3{\scriptstyle\pm1.1}$ & $39.7{\scriptstyle\pm0.8}$ \\
& \textsc{e}$^{\contC}$ content entropy & $13.4{\scriptstyle\pm0.4}$ & $52.2{\scriptstyle\pm0.1}$ & $73.0{\scriptstyle\pm0.7}$ & $26.2{\scriptstyle\pm0.8}$ & $40.1{\scriptstyle\pm0.9}$ \\
\midrule
\multirow{4}{*}{\rotatebox{90}{\emph{controls}}}
& \textsc{rand} random scatter & $17.2{\scriptstyle\pm0.7}$ & $38.7{\scriptstyle\pm0.5}$ & $58.9{\scriptstyle\pm0.4}$ & $21.1{\scriptstyle\pm0.7}$ & $17.8{\scriptstyle\pm0.3}$ \\
& \textsc{first} head block & $14.3{\scriptstyle\pm0.4}$ & $40.8{\scriptstyle\pm0.4}$ & $61.9{\scriptstyle\pm0.4}$ & $18.5{\scriptstyle\pm0.5}$ & $6.4{\scriptstyle\pm0.6}$ \\
& \textsc{blk} random contiguous & $49.2{\scriptstyle\pm1.2}$ & $61.6{\scriptstyle\pm0.9}$ & $69.9{\scriptstyle\pm1.0}$ & $28.0{\scriptstyle\pm1.7}$ & $33.8{\scriptstyle\pm1.5}$ \\
& \textsc{last} recency & $\mathbf{70.6}{\scriptstyle\pm0.3}$ & $\mathbf{75.1}{\scriptstyle\pm0.3}$ & $\mathbf{75.1}{\scriptstyle\pm0.9}$ & $\mathbf{34.4}{\scriptstyle\pm0.5}$ & $\mathbf{67.7}{\scriptstyle\pm0.3}$ \\
\midrule
& uncompressed chain & $75.5{\scriptstyle\pm0.8}$ & $75.5{\scriptstyle\pm0.8}$ & $75.5{\scriptstyle\pm0.8}$ & $34.2{\scriptstyle\pm0.8}$ & $68.8{\scriptstyle\pm0.3}$ \\
\midrule
\multicolumn{7}{l}{\emph{after stripping every gold-answer restatement and the final sentence}}\\
& \textsc{k} raw surprisal & $10.7{\scriptstyle\pm0.5}$ & $8.3{\scriptstyle\pm0.5}$ & $10.6{\scriptstyle\pm1.1}$ & $9.5{\scriptstyle\pm0.6}$ & $6.9{\scriptstyle\pm0.1}$ \\
& \textsc{k}$^{\contC}$ raw score, content-only & $9.8{\scriptstyle\pm0.7}$ & $12.6{\scriptstyle\pm1.6}$ & $35.3{\scriptstyle\pm0.5}$ & $11.2{\scriptstyle\pm0.3}$ & $14.2{\scriptstyle\pm0.9}$ \\
& \textsc{k}$^{g}$ gate-adj score, all & $10.7{\scriptstyle\pm0.8}$ & $11.4{\scriptstyle\pm0.5}$ & $31.8{\scriptstyle\pm0.5}$ & $10.0{\scriptstyle\pm0.9}$ & $14.1{\scriptstyle\pm0.9}$ \\
& \textsc{a} attention received & $10.6{\scriptstyle\pm0.7}$ & $16.5{\scriptstyle\pm1.1}$ & $26.4{\scriptstyle\pm0.9}$ & $12.5{\scriptstyle\pm1.3}$ & $11.1{\scriptstyle\pm1.1}$ \\
& \textsc{d} content surprisal & $10.4{\scriptstyle\pm0.1}$ & $12.0{\scriptstyle\pm0.8}$ & $35.3{\scriptstyle\pm0.5}$ & $10.5{\scriptstyle\pm0.5}$ & $14.8{\scriptstyle\pm0.7}$ \\
& \textsc{rand} random scatter & $7.5{\scriptstyle\pm0.3}$ & $12.4{\scriptstyle\pm1.1}$ & $21.5{\scriptstyle\pm0.8}$ & $8.2{\scriptstyle\pm0.6}$ & $12.2{\scriptstyle\pm1.3}$ \\
& \textsc{first} head block & $15.4{\scriptstyle\pm0.4}$ & $34.4{\scriptstyle\pm0.7}$ & $\mathbf{43.2}{\scriptstyle\pm1.6}$ & $13.9{\scriptstyle\pm1.8}$ & $6.9{\scriptstyle\pm0.6}$ \\
& \textsc{blk} random contiguous & $30.0{\scriptstyle\pm0.6}$ & $\mathbf{40.5}{\scriptstyle\pm0.2}$ & $42.2{\scriptstyle\pm1.8}$ & $15.5{\scriptstyle\pm0.9}$ & $24.2{\scriptstyle\pm0.4}$ \\
& \textsc{last} recency & $\mathbf{33.3}{\scriptstyle\pm0.7}$ & $39.1{\scriptstyle\pm1.0}$ & $39.7{\scriptstyle\pm1.6}$ & $\mathbf{19.9}{\scriptstyle\pm0.6}$ & $\mathbf{48.5}{\scriptstyle\pm1.2}$ \\
& uncompressed chain, stripped & $41.2{\scriptstyle\pm0.9}$ & $41.2{\scriptstyle\pm0.9}$ & $41.2{\scriptstyle\pm0.9}$ & $19.9{\scriptstyle\pm0.5}$ & $49.5{\scriptstyle\pm1.4}$ \\
\bottomrule
\end{tabular}
\end{table}

\textbf{Protocol.} Post-hoc compression in the TokenSkip pipeline shape \citep{xia2025tokenskip}: the model generates its own chain, the \verb|\boxed{}| span is deleted so no gold answer reaches a scorer, an arm selects which tokens survive a keep budget, and the model is re-prompted from the surviving text to answer. Every arm keeps the same number of tokens, so arms differ only in which. Three audits accompany the arms: a $2\times2$ factorial separating the content scorer's two ingredients, an answer-survival audit measuring how often each arm's kept text still contains the gold answer, and a stripped rerun of every cell with each restatement of the gold answer and the final sentence deleted before compression. Seeds resample problems, so alongside pooled exact McNemar tests every margin below carries a problem-clustered bootstrap interval ($3\times512$ records collapse to $1{,}020$ problem clusters on GSM8K, $3\times256$ to $441$ on MATH-500); square brackets are clustered $95\%$ intervals in points.

\textbf{The channel gain is real in every cell, and it is an interaction.} Content-channel scoring beats raw surprisal in all five cells: $+33.1$ $[30.2,36.1]$, $+27.6$ $[25.0,30.3]$, $+14.6$ $[12.0,17.1]$, $+10.1$ $[6.9,13.4]$, and $+2.3$ $[0.5,4.2]$ at the tightest budget, which survives clustering. The factorial says the gain is attributable to neither ingredient alone: at the distilled model's half budget, content-only eligibility with the raw score gains $+5.7$ $[3.2,8.3]$ and the gate-adjusted score on all tokens gains $+5.3$ $[2.7,7.8]$ while the combination gains $+33.1$; at the loosest budget either alone recovers the full effect; on the instruct model each alone \emph{loses} about $10$ points while the combination gains $+14.6$. Raw surprisal spends $25$ to $34\%$ of its budget on scaffold tokens, the content scorer none, and the keep sets agree on $47$ to $59\%$ of tokens; content entropy tracks content surprisal within a point everywhere, so the moving axis is the convention, not the statistic.

\textbf{Position dominates the scorer family, before the audit.} Keeping the last half of the chain reaches $75.1\pm0.3\%$ against $75.5\pm0.8$ uncompressed on GSM8K and $34.4$ against $34.2$ on MATH-500; a quarter of the chain still reaches $70.6\%$, and every scorer sits $23$ to $56$ points below. Two controls split the advantage: a random \emph{contiguous} block beats a random \emph{scattered} keep in all five cells ($+33.0$ $[30.0,36.0]$ down to $+7.5$ $[4.5,10.4]$), so a compressed chain has to read as continuous text, and recency then beats the contiguous block ($+21.3$ $[18.5,24.1]$ down to $+5.0$ $[3.3,6.7]$, $+35.4$ $[32.5,38.3]$ on the instruct model), while the head block runs far below it.

\textbf{The answer-survival audit finds a leak} (Figure~\ref{fig:decisions}). Only the \verb|\boxed{}| span was deleted, and chains restate the final answer in prose before boxing it, the confound \citet{garcia2026lastword} names in corruption studies. Measuring it directly: at a quarter budget on the distilled model, the recency arm's kept text still contains the gold answer for $79\%$ of chains, against $53\%$ for the contiguous block, $37\%$ for the scatter, $15\%$ for the head, and $84\%$ for the full chain; at half budget recency's survival matches the full chain's ($83$ vs $84\%$). The tail almost always carries the restated answer, so the recency arm was answering, in part, by copying.

\textbf{Stripping the restatements halves the ceiling and rewrites the ordering.} With every gold-answer restatement and the final sentence deleted, the uncompressed chain itself falls from $75.5$ to $41.2\%$ on distilled GSM8K, $68.8$ to $49.5$ on the instruct model, and $34.2$ to $19.9$ on MATH-500: a quarter to a half of re-fed accuracy rides on answer text the evaluation hands back, a bound loose in the strict direction since stripping also deletes genuine final-step content. Against the stripped ceiling, contiguity survives everywhere ($+12.1$ to $+28.5$ over scatter), and the tail premium becomes model-dependent: on the distilled model recency collapses onto the contiguous block ($-2.3$ $[-4.9,+0.3]$ at half budget, $-3.2$ $[-5.7,-0.8]$ at three quarters, $+3.5$ $[0.6,6.4]$ at one quarter, the head block drawing level at three quarters), while a real tail premium remains on the instruct model ($+24.1$ $[21.1,26.9]$) and a modest one on MATH-500 ($+4.6$ $[2.1,7.1]$). The channel gain survives stripping where the budget allows a choice, $+24.7$ $[22.1,27.5]$ at three-quarters budget on the distilled model and $+8.5$ $[6.5,10.6]$ on the instruct model, and shrinks to $+3.2$ $[1.3,5.0]$ or nothing at tight budgets. One ordering is uniform: in every stripped cell, and in nine of ten cells overall, no token scorer beats the random contiguous block.

\textbf{Consequence: the audit corrected its own headline.} Before the audit the finding read ``keep the tail''; after it, three corrected claims stand. First, re-feeding evaluations substantially measure answer copying unless restatements are stripped: the ceiling drops of $34$, $19$, and $14$ points quantify, for the TokenSkip-shaped evaluation family, the inflation \citet{garcia2026lastword} diagnosed in corruption studies, and no compression paper we cite strips restatements. Second, the robust positional prior is contiguity: no token scorer beats a random contiguous block in any stripped cell, sharpening \citet{candussio2026demystifying} from ``entropy does not beat random scatter'' to ``no token-level score beats random contiguous''. Third, what looked like one tail effect was two: copyable answer text on the distilled model, and genuine late-chain state on the instruct model. EPiC's opposite ordering in training-data condensation \citep{jia2025epic} and KV-cache eviction in between \citep{xiao2024streamingllm,zhang2023h2o,li2024snapkv,steele2026limits} then read as this paper's thesis applied to position: a token-importance claim is uninterpretable until the estimand, learn-from, answer-from, or attend-from, is stated. Theorem~\ref{thm:reversal} concerns conventions at a position and says nothing about where the position sits or what the evaluation hands back; the audit protocol, run against our own headline, caught both.

\textbf{Attribution and eviction.} Receiver-style sentence attribution \citep{bogdan2025anchors} under raw versus content-renormalized attention agrees at Kendall $\tau=0.75\pm0.01$ yet flips its top-1 anchor on $40.4\pm3.3\%$ of problems; KV-eviction keep sets \citep{cai2025rkv} overlap at Jaccard $0.64$/$0.70$ at $30\%$/$50\%$ budgets. The ranking is stable; the single item read off it is not.

\subsection{Training dynamics: two nulls and one decomposed decline}
\label{sec:training}

\textbf{Two nulls, reported as such} (arm table in App.~\ref{app:additional}). Six GRPO runs on Qwen2.5-0.5B and its instruct variant (three seeds each, $300$ to $400$ steps) give a pooled raw-entropy change of $-0.05\pm0.03$ nats while the policies demonstrably move; the attribution of Definition~\ref{def:attr} stays below $0.008$ nats in every run, and at pass@$1$ of $3$ to $9\%$ most GRPO groups carry zero advantage. Three matched-update arms (no bonus, content bonus, raw bonus) raise final on-policy $H$ to $0.580\pm0.025$, $0.655\pm0.026$, $0.610\pm0.029$ and leave pass@$8$ at $39.6\pm5.1$, $38.0\pm4.5$, $43.2\pm8.1\%$, no pair separating ($p=0.29$ to $0.66$); Corollary~\ref{cor:reinforce} predicts arm-indifference here, since a zero-advantage step zeroes every channel mask. One confirmed sign runs against the naive expectation: the \emph{content} bonus raised total $H$ more than the \emph{raw} bonus, as the accounting predicts, since a raw bonus can be paid in cheap gate and scaffold entropy while a content bonus must move the expensive channel. A dense-reward pair on Qwen2.5-Math-1.5B declines $0.13\pm0.08$ nats, under the pre-stated $0.15$-nat threshold; $90\pm3\%$ of each decline lies in the content channel and pass@$8$ \emph{rises} on the declining seed.

\textbf{The routing witness is large; leakage is certified, not estimated.} On R1-Distill-1.5B the first-content routing information is $\hat\kappa=0.81$ nats ($1.16$ bits), bootstrap interval $[0.76,0.85]$: which connective the model emits carries more than a bit about the first content token that follows. By Theorem~\ref{prop:exploration} this certifies $\lek>0$ and bounds the one-step inert-scaffold reading away; it does not estimate $\lek$, for which no numerical relation to $\hat\kappa$ exists (App.~\ref{app:proofs}). Lifting the outcome to (waiting time, first content token) gives $1.01\pm0.02$ nats, dominating the first-token value as data processing requires; the eventual-content entropy at scaffold positions measures $H(q_t)=0.600$ nats $[0.496,0.725]$ (App.~\ref{app:instruments}).

\section{Discussion and limitations}
\label{sec:discussion}

\textbf{What is claimed.} The results concern \emph{measurement} and first-order \emph{mechanism}: raw CoT statistics answer convention-dependent mixtures of three questions that separate exactly, the optimizer determines which channels an advantage pattern moves, and the scaffold's contribution to exploration is exactly the leakage it routes. We claim neither that models ``use'' channels nor that the complement block is validated ``reasoning''; the claim is that the question should be stated, that stating it moves compression accuracy by up to $33$ points, and that the same audit discipline, applied to our own positional control, exposed an answer-copying inflation the evaluation family shares.

\textbf{Limits.} Everything is measured at or below $3$B on GSM8K and MATH-500; the $7$B rows are marked to run. The dissociation compares released checkpoints, not randomized interventions; the design controls tokenizer, context set, token budget, and lineage, and the causal reading stops there. The forecast's held-out cells are three same-family in-range checkpoints, so that claim is baseline-beating calibration, not external validity. The stripping protocol is string-based: paraphrased restatements survive it and genuine final-step content does not, which brackets the copying estimate without pinning it. The positional findings concern post-hoc text re-feeding, not KV-cache eviction, where attention scoring beats recency. The training nulls are theory-consistent under sparse reward; the discriminating placement, pass@$1$ above $30\%$, five seeds per arm, remains unrun. The scaffold set is chosen, not canonical; the identities hold for any $\scafS$, we evaluate three constructions and the tiered refinement, and $\kappa_t$ measures whether connectives carry strategy.

\textbf{Directions.} Estimating $\lek$; longer horizons and scale; a randomized corpus intervention; the decomposition on answer distributions and MoE routers.

\label{endbody}
\label{endmain}

\bibliographystyle{plainnat}
\bibliography{refs}

\appendix

\section{Full formal statements}
\label{app:full}

The main text states each result in the form used there. This section
restates the same results in full, with every domain condition, equality
case, and attainer; the proofs in App.~\ref{app:proofs} refer to these
statements. Nothing here differs in content from the main-text version.

\subsection{Exact channel identities}

\begin{theorem}[Channel chain rules and sharp signed scaffold-surplus bounds]
\label{thm:chain-full}
Let \(V\) be a finite set with a partition $V=\scafS\mathbin{\dot\cup}\contC, \scafS\neq\varnothing, \contC\neq\varnothing$, and put $m:=|\scafS|, k:=|\contC|$. For every nonempty finite set \(A\), define
\[
\overline{\Delta}(A)
:=
\left\{
r\in[0,1]^A:
\sum_{a\in A}r_a=1
\right\}.
\]
For \(r\in\overline{\Delta}(A)\), define $H(r):=-\sum_{a\in A}r_a\log r_a, 0\log0:=0$. For \(r,u\in\overline{\Delta}(A)\), write \(r\ll u\) if $u_a=0\ \Longrightarrow\ r_a=0 ,\ (a\in A)$, and define
\[
\KL(r\|u)
:=
\begin{cases}
\displaystyle
\sum_{\substack{a\in A\\r_a>0}}
r_a\log\frac{r_a}{u_a},
& r\ll u,\\[2ex]
+\infty,
& r\not\ll u.
\end{cases}
\]
Define the gate-interior simplex
\[
\Delta_{\mathrm{gi}}(V)
:=
\left\{
r\in\overline{\Delta}(V):
0<r(\scafS)<1
\right\},
\qquad
r(\scafS):=\sum_{v\in\scafS}r_v.
\]
For every \(r\in\Delta_{\mathrm{gi}}(V)\), write $s_r:=r(\scafS)$, and define its conditional block distributions by $r_v^{\scafS}:=\frac{r_v}{s_r} \quad(v\in\scafS), r_v^{\contC}:=\frac{r_v}{1-s_r} \quad(v\in\contC)$. Set $\HS(r):=H(r^{\scafS}), \HC(r):=H(r^{\contC}), G(r):=H(r)-\HC(r)$. For \(u\in[0,1]\), define $\Hb(u):=-u\log u-(1-u)\log(1-u)$, using \(0\log0=0\). Let \(p,q\in\Delta_{\mathrm{gi}}(V)\), and write $s:=s_p=p(\scafS), t:=s_q=q(\scafS)$. Define $\KL_{\mathrm b}(s\|t) := s\log\frac{s}{t} + (1-s)\log\frac{1-s}{1-t}$. Then the following statements hold.
\begin{enumerate}
\item[\rm(i)]
For every \(r\in\Delta_{\mathrm{gi}}(V)\), $H(r) = \Hb(s_r) +s_r\,\HS(r) +(1-s_r)\,\HC(r)$. \item[\rm(ii)]
As an identity in \([0,+\infty]\), $\KL(p\|q) = \KL_{\mathrm b}(s\|t) +s\,\KL(p^{\scafS}\|q^{\scafS}) +(1-s)\,\KL(p^{\contC}\|q^{\contC})$. \item[\rm(iii)]
For every \(r\in\Delta_{\mathrm{gi}}(V)\), $G(r) = \Hb(s_r) +s_r\bigl(\HS(r)-\HC(r)\bigr)$. Fix \(\sigma\in(0,1)\) and
\(c\in\overline{\Delta}(\contC)\), and define
\[
\mathcal P_{\sigma,c}
:=
\left\{
r\in\Delta_{\mathrm{gi}}(V):
s_r=\sigma,\ 
r^{\contC}=c
\right\}.
\]
Then $\min_{r\in\mathcal P_{\sigma,c}}G(r) = \Hb(\sigma)-\sigma H(c)$, and $\max_{r\in\mathcal P_{\sigma,c}}G(r) = \Hb(\sigma) +\sigma\bigl(\log m-H(c)\bigr)$. A distribution \(r\in\mathcal P_{\sigma,c}\) attains the minimum if
and only if \(r^{\scafS}\) is a point mass. It attains the maximum if
and only if \(r^{\scafS}\) is uniform on \(\scafS\).
For fixed \(\sigma\in(0,1)\), define
\[
\mathcal P_{\sigma}
:=
\left\{
r\in\Delta_{\mathrm{gi}}(V):
s_r=\sigma
\right\}.
\]
Then $\min_{r\in\mathcal P_{\sigma}}G(r) = \Hb(\sigma)-\sigma\log k$, and $\max_{r\in\mathcal P_{\sigma}}G(r) = \Hb(\sigma)+\sigma\log m$. A distribution \(r\in\mathcal P_{\sigma}\) attains the minimum if
and only if \(r^{\scafS}\) is a point mass and \(r^{\contC}\) is
uniform on \(\contC\). It attains the maximum if and only if
\(r^{\scafS}\) is uniform on \(\scafS\) and \(r^{\contC}\) is a
point mass. These characterizations include the singleton-block
cases.
\item[\rm(iv)]
Let \(n\geq1\), let
\(\alpha_1,\ldots,\alpha_n\in\mathbb R\), and let
\(p_1,\ldots,p_n\in\Delta_{\mathrm{gi}}(V)\). Then
\[
\sum_{j=1}^{n}\alpha_jH(p_j)
=
\sum_{j=1}^{n}\alpha_j\Hb(s_{p_j})
+
\sum_{j=1}^{n}\alpha_js_{p_j}\HS(p_j)
+
\sum_{j=1}^{n}\alpha_j(1-s_{p_j})\HC(p_j).
\]
More generally, let \(N\) be a nonnegative integer-valued random
variable satisfying \(N<\infty\) almost surely. For every \(j\geq1\),
let \(P_j\) be a measurable
\(\Delta_{\mathrm{gi}}(V)\)-valued random vector and let \(W_j\) be a
real-valued measurable random variable. Put $S_j:=s_{P_j}=P_j(\scafS)$. Assume $\mathbb E\left[ \sum_{j=1}^{N}|W_j| \right] <\infty$, with an empty sum interpreted as zero. Then all four stopped weighted
sums below are absolutely integrable, and
\begin{align*}
\mathbb E\left[
\sum_{j=1}^{N}W_jH(P_j)
\right]
&=
\mathbb E\left[
\sum_{j=1}^{N}W_j\Hb(S_j)
\right]\\
&\quad+
\mathbb E\left[
\sum_{j=1}^{N}W_jS_j\HS(P_j)
\right]\\
&\quad+
\mathbb E\left[
\sum_{j=1}^{N}W_j(1-S_j)\HC(P_j)
\right].
\end{align*}
No independence or stopping-time assumption is required.
\end{enumerate}
\end{theorem}

The surplus bounds calibrate the stakes: with a modest scaffold set the surplus is at most a few nats, but a few nats is exactly the scale on which forking-token thresholds and collapse curves live, so the confound is quantitatively live and is not a tail effect. The same partition splits the Aitchison geometry the companion develops for attention rows:

\begin{lemma}[Three-way Pythagorean split]
\label{lem:pyth}
With the partition balance $b(p)=\sqrt{\tfrac{|\scafS||\contC|}{D}}\,\log\tfrac{g(p|_{\scafS})}{g(p|_{\contC})}$ ($g$ the geometric mean), the map $p\mapsto\big(b(p),\ilr(\pS),\ilr(\pC)\big)$ is an isometry, hence
$\dA(p,q)^2=\big(b(p)-b(q)\big)^2+\dA(\pS,q^{\scafS})^2+\dA(\pC,q^{\contC})^2$.
The entropy identity uses the arithmetic mass $s$, the geometric identity the balance $b$; they are different coordinates for the same three questions, and we use each where it is exact (Remark~\ref{rem:amalg}).
\end{lemma}

\subsection{The content channel is the maximal invariant}

\begin{proposition}[Content maximal invariance and temperature behavior]
\label{prop:invariance}
Let \(V\) be finite with \(|V|=D\ge2\), and let $V=\scafS\mathbin{\dot\cup}\contC$, $\scafS,\contC\neq\varnothing$.
Write $P(z):=\clo(e^{z})$ for the softmax map on $\mathbb R^V$, with
$P^{B}(z):=\clo\bigl(P(z)|_{B}\bigr)$ its block channels
($B\in\{\scafS,\contC\}$), so that $P^{\contC}(z)$ is the content
channel of \eqref{eq:factor} as a function of the logits. Let
\[
G_{\scafS}
:=
\left\{
g_{w,c}:
g_{w,c}(z)=z+w+c\mathbf1_V,\ 
w|_{\contC}=0,\ 
c\in\mathbb R
\right\}.
\]
For \(w|_{\contC}=0\), define the induced action on
\(\Delta_V^\circ\) by $\bigl(\Phi_w(p)\bigr)_v := \frac{e^{w_v}p_v} {\sum_{u\in V}e^{w_u}p_u}, v\in V$. \begin{enumerate}
\item[\rm(i)]
The content subcomposition is a maximal invariant. Specifically, $P^{\contC}(g_{w,c}z)=P^{\contC}(z)$ for every \(z\in\mathbb R^V\) and \(g_{w,c}\in G_{\scafS}\), and
\[
P^{\contC}(z)=P^{\contC}(z')
\quad\Longleftrightarrow\quad
z'=g_{w,c}(z)
\quad
\text{for some }g_{w,c}\in G_{\scafS}.
\]
Equivalently, for \(p,q\in\Delta_V^\circ\),
\[
p^{\contC}=q^{\contC}
\quad\Longleftrightarrow\quad
q=\Phi_w(p)
\quad
\text{for some }w\in\mathbb R^V
\text{ satisfying }w|_{\contC}=0.
\]
Consequently, for any set \(\mathcal Y\) and any map
\(F:\Delta_V^\circ\to\mathcal Y\), the following are equivalent:
\begin{enumerate}
\item[\rm(a)] $F(\Phi_w(p))=F(p)$ for every \(p\in\Delta_V^\circ\) and every \(w|_{\contC}=0\);
\item[\rm(b)]
there exists a unique map $\widetilde F:\Delta_{\contC}^\circ\to\mathcal Y$ such that $F(p)=\widetilde F(p^{\contC}),\ \text{for every }p\in\Delta_V^\circ$. \end{enumerate}
\item[\rm(ii)]
For \(x\in\Delta_A^\circ\), define $\operatorname{VarEnt}(x) := \sum_{a\in A} x_a\bigl(-\log x_a-H(x)\bigr)^2$. When \(|A|\ge2\), write
\(x_{(1)}\ge x_{(2)}\ge\cdots\) and define $\operatorname{mar}(x):=x_{(1)}-x_{(2)}$. For \(1\le k\le |A|\), let
\(\operatorname{Top}_k(x)\) denote the \(k\) selected indices under
a fixed deterministic tie-breaking rule.
Under the induced action \(\Phi\), the maps $p\longmapsto H(p), p\longmapsto\operatorname{mar}(p), p\longmapsto\operatorname{VarEnt}(p)$ are not invariant. For every
\(k\in\{1,\ldots,D-1\}\), the map $p\longmapsto\operatorname{Top}_k(p)$ is not invariant. Moreover, for every
\(h\in(0,\log D)\), the decision $E_h(p):=\mathbf1\{H(p)>h\}$ is not invariant. Separate witnesses may be used for the different
statistics.
In contrast, $z\longmapsto H(P^{\contC}(z)), z\longmapsto\operatorname{VarEnt}(P^{\contC}(z))$, and $z\longmapsto \mathbf1\{H(P^{\contC}(z))>h\}, h\in\mathbb R$, are \(G_{\scafS}\)-invariant. For
\(1\le\ell\le|\contC|\), $z\longmapsto\operatorname{Top}_\ell(P^{\contC}(z))$ is \(G_{\scafS}\)-invariant. If \(|\contC|\ge2\), then $z\longmapsto\operatorname{mar}(P^{\contC}(z))$ is also \(G_{\scafS}\)-invariant.
\item[\rm(iii)]
For every \(\tau>0\), every \(z\in\mathbb R^V\), and either block
\(B\in\{\scafS,\contC\}\), $P^B(z/\tau)=(1/\tau)\odot P^B(z)$. Consequently, if the same temperature \(\tau\) is applied to all
compared logit vectors and to both arguments of every distance, then
\[
d_{\mathrm A}\!\left(
P^{\contC}(z/\tau),P^{\contC}(z'/\tau)
\right)
=
\frac1\tau
d_{\mathrm A}\!\left(
P^{\contC}(z),P^{\contC}(z')
\right).
\]
Hence every ranking and every tie among pairwise content-channel
Aitchison distances is preserved. At each fixed position, the
coordinate ordering, including ties, within \(\contC\) is also
preserved.
Define $\mathsf H(z):=H(P(z)), s(z):=\sum_{v\in\scafS}P_v(z), G(z):=\mathsf H(z)-H(P^{\contC}(z))$. If \(D\ge3\), then for every $\Psi\in\{\mathsf H,s,G\}$ there exist finite logit vectors
\(z_\Psi,z'_\Psi\in\mathbb R^V\) and finite temperatures $0<\tau_{\Psi,-}<\tau_{\Psi,+}<\infty$ such that
\[
\begin{aligned}
&
\bigl[
\Psi(z_\Psi/\tau_{\Psi,-})
-
\Psi(z'_\Psi/\tau_{\Psi,-})
\bigr]
\\
&\hspace{25mm}\times
\bigl[
\Psi(z_\Psi/\tau_{\Psi,+})
-
\Psi(z'_\Psi/\tau_{\Psi,+})
\bigr]
<0.
\end{aligned}
\]
Thus raw entropy, gate mass, and scaffold surplus do not possess a
universal cross-position rank-preservation law under a common
temperature change.
\end{enumerate}
\end{proposition}

Part (i) is a per-position measurement charter with an explicitly limited writ: it governs perturbations of the current logits at a fixed partition, and it does \emph{not} deliver trajectory-level convention independence, because emitting a different scaffold token changes the hidden state and hence future distributions. That trajectory-level question is precisely the inertness/leakage question, which \S\ref{sec:firstpass} makes measurable; and the partition itself remains a modelling choice audited in App.~\ref{app:scaffold}. Within that writ the charter has teeth: any conclusion advertised as being about the \emph{content} of the next decision must be computable from $\pC$, or it is partly a statement about scaffolding conventions.

\subsection{Channel-resolved entropy dynamics}
\label{sec:theory-dyn}

\citet{cui2025entropy} trace entropy collapse to a covariance: for a softmax policy, the first-order entropy change under a logit step is $-\operatorname{Cov}_{a\sim p}(\log p_a,\Delta z_a)$. That covariance splits exactly along the partition.

\begin{proposition}[Channel dynamics]
\label{prop:dynamics-full}
Let \(V=\scafS\sqcup\contC\) be a finite partition with
\(\scafS,\contC\neq\varnothing\). Let \(I\subset\mathbb R\) be an open
interval and let \(z:I\to\mathbb R^V\) be continuously differentiable.
For \(a\in V\), define $p_a(t) := \frac{\exp(z_a(t))} {\sum_{v\in V}\exp(z_v(t))}, u_a(t):=\dot z_a(t)$. For a probability vector \(r=(r_i)_{i\in E}\), define $H(r):=-\sum_{i\in E}r_i\log r_i, 0\log0:=0$, and
\[
\operatorname{Cov}_{i\sim r}(f_i,g_i)
:=
\sum_{i\in E}r_i f_i g_i
-
\left(\sum_{i\in E}r_i f_i\right)
\left(\sum_{i\in E}r_i g_i\right).
\]
Also define $\Hb(s):=-s\log s-(1-s)\log(1-s)$. For \(B\in\{\scafS,\contC\}\), set $s(t):=\sum_{a\in\scafS}p_a(t), m_{\scafS}(t):=s(t), m_{\contC}(t):=1-s(t)$, and $p_a^B(t):=\frac{p_a(t)}{m_B(t)} \quad(a\in B)$. Write $H^B(t):=H(p^B(t)), \bar u_B(t) := \mathbb E_{a\sim p^B(t)}[u_a(t)]$, and define $K(p(t)) := \log\frac{1-s(t)}{s(t)} +H^{\scafS}(t)-H^{\contC}(t)$. Then, for every \(t\in I\),
\[
\frac{d}{dt}H(p(t))
=
-\operatorname{Cov}_{a\sim p(t)}
\bigl(\log p_a(t),u_a(t)\bigr)
=
\Gamma_{\mathrm{gate}}(t)
+\Gamma_{\scafS}(t)
+\Gamma_{\contC}(t),
\]
where $\Gamma_{\mathrm{gate}}(t) = K(p(t))\,s(t)(1-s(t)) \bigl(\bar u_{\scafS}(t)-\bar u_{\contC}(t)\bigr)$ and
\[
\Gamma_B(t)
=
-m_B(t)
\operatorname{Cov}_{a\sim p^B(t)}
\bigl(\log p_a^B(t),u_a(t)\bigr),
\qquad
B\in\{\scafS,\contC\}.
\]
Moreover, $\dot s(t) = s(t)(1-s(t)) \bigl(\bar u_{\scafS}(t)-\bar u_{\contC}(t)\bigr)$, $\dot p_a^B(t) = p_a^B(t) \bigl(u_a(t)-\bar u_B(t)\bigr)$, and $\dot H^B(t) = -\operatorname{Cov}_{a\sim p^B(t)} \bigl(\log p_a^B(t),u_a(t)\bigr)$. Consequently,
\[
\Gamma_{\mathrm{gate}}(t)
=
K(p(t))\dot s(t)
=
\frac{d}{dt}\Hb(s(t))
+
\bigl(H^{\scafS}(t)-H^{\contC}(t)\bigr)\dot s(t).
\]
If, at some \(t\in I\), \(u_a(t)\) is constant over \(a\in B\), then $\dot p_a^B(t)=0\quad(a\in B), \Gamma_B(t)=0$. Thus, if \(u(t)\) is constant on both blocks, only
\(\Gamma_{\mathrm{gate}}(t)\) can be nonzero.
\medskip
\noindent
\textbf{Local categorical natural-gradient specialization.}
Suppose, in addition, that the categorical policy at the fixed prefix
is parameterized by its own unrestricted logit vector. Let
\(A:I\to\mathbb R^V\) be a finite conditional action-score vector.
In a policy-gradient application, \(A(t)\) may be the exact conditional
advantage; for the standard centered advantage, $\mathbb E_{a\sim p(t)}[A_a(t)]=0$. More generally, adding the same scalar baseline to every component of
\(A(t)\) leaves all conclusions below unchanged.
Define $\bar A(t) := \mathbb E_{a\sim p(t)}[A_a(t)], \bar A_B(t) := \mathbb E_{a\sim p^B(t)}[A_a(t)]$, and $\langle A(t)\rangle_{\mathrm{unif}} := \frac{1}{|V|}\sum_{a\in V}A_a(t)$. For every \(t\in I\), let $F(t) := \operatorname{diag}(p(t))-p(t)p(t)^\top$ and $g(t) := \mathbb E_{a\sim p(t)} \bigl[A_a(t)(e_a-p(t))\bigr]$, where \(e_a\) is the \(a\)-th standard basis vector.
Assume that the actual logit path satisfies the exact undamped local
natural-gradient ascent equation $u(t)=\dot z(t) = u^{\mathrm{NG}}(t) := \eta F(t)^\dagger g(t), t\in I$, for some \(\eta>0\), where \(F(t)^\dagger\) is the Moore--Penrose
pseudoinverse. Then $u^{\mathrm{NG}}(t) = \eta \left( A(t)-\langle A(t)\rangle_{\mathrm{unif}}\mathbf1 \right)$. Equivalently, $u^{\mathrm{NG}}(t)=\eta A(t)+c(t)\mathbf1, c(t):=-\eta\langle A(t)\rangle_{\mathrm{unif}}$. The induced probability velocities satisfy $\dot p_a(t) = \eta p_a(t)\bigl(A_a(t)-\bar A(t)\bigr)$, $\dot p_a^B(t) = \eta p_a^B(t)\bigl(A_a(t)-\bar A_B(t)\bigr)$, and $\dot s(t) = \eta s(t)(1-s(t)) \bigl(\bar A_{\scafS}(t)-\bar A_{\contC}(t)\bigr)$. The entropy contributions are $\Gamma_B(t) = -\eta m_B(t) \operatorname{Cov}_{a\sim p^B(t)} \bigl(\log p_a^B(t),A_a(t)\bigr)$ and $\Gamma_{\mathrm{gate}}(t) = \eta K(p(t))s(t)(1-s(t)) \bigl(\bar A_{\scafS}(t)-\bar A_{\contC}(t)\bigr)$. If, at some \(t\in I\), there exists a scalar \(\alpha_B(t)\) such
that $A_a(t)=\alpha_B(t) ,\ (a\in B)$, then $\dot p_a^B(t)=0\quad(a\in B), \Gamma_B(t)=0$. If, for every \(t\) in an interval \(J\subset I\), there exists
\(\alpha_B(t)\) such that $A_a(t)=\alpha_B(t) ,\ (a\in B)$, then \(p^B(t)\) is constant on \(J\).
For every \(t\in I\), $\dot s(t)\neq0 \quad\Longleftrightarrow\quad \bar A_{\scafS}(t)\neq\bar A_{\contC}(t)$, whereas
\[
\Gamma_{\mathrm{gate}}(t)\neq0
\quad\Longleftrightarrow\quad
K(p(t))
\bigl(\bar A_{\scafS}(t)-\bar A_{\contC}(t)\bigr)
\neq0.
\]
Moreover,
\[
\operatorname{sgn}\bigl(\Gamma_{\mathrm{gate}}(t)\bigr)
=
\operatorname{sgn}\!\left(
K(p(t))
\bigl(\bar A_{\scafS}(t)-\bar A_{\contC}(t)\bigr)
\right).
\]
\end{proposition}

The optimizer matters, and the vanilla case is \emph{not} gate-only:

\begin{corollary}[Two optimizer regimes]
\label{cor:reinforce-full}
The expected REINFORCE/GRPO logit update is $\mathbb{E}[\Delta z_a]=\eta\,p_a(\hat A_a-\bar{\hat A})$, which is proportional to $p^{B}$ within each block even when $\hat A$ is block-constant. For a block-constant advantage gap $\Delta A=\bar A_{\scafS}-\bar A_{\contC}$ the within-block terms are exactly
\[
\Gamma_{\scafS}=-\eta\,s^{2}(1-s)\,\Delta A\;V_{\scafS},\qquad
\Gamma_{\contC}=+\eta\,s(1-s)^{2}\,\Delta A\;V_{\contC},\qquad
V_{B}:=\operatorname{Cov}_{a\sim p^{B}}\!\big(\log p^{B}_a,\,p^{B}_a\big)\ \ge 0
\]
($V_B\ge0$ by comonotonicity, $=0$ iff $p^B$ is uniform). A format reward ($\Delta A>0$) therefore \emph{sharpens the scaffold channel} (the connective habit) and, to first order, \emph{spreads the content channel}, on top of the between-block movement; only under natural gradient, or exactly blockwise-constant velocities, is the effect confined to the gate. (Proof, with a numerical counterexample to the gate-only reading: App.~\ref{app:proofs}.)
\end{corollary}

This turns the decomposition from an accounting identity into a statement about mechanism with the optimizer in the loop: it predicts which training signals move which channel, and in which direction, prospectively. \S\ref{sec:insilico} confirms the predictions in silico under the sampled REINFORCE update itself; the GRPO protocol tests them at model scale and adds the interventions they motivate (\S\ref{sec:protocol}).

\subsection{Ranking reversals and the selection wedge}

\begin{theorem}[Fork-verdict reversal]
\label{thm:reversal-full}
Let $V=\scafS\sqcup\contC, m:=|\scafS|\geq1, k:=|\contC|\geq2, D:=m+k$. For \(p\in\Delta_V^\circ\), put
\[
p(\contC):=\sum_{v\in\contC}p_v,
\qquad
p^{\contC}
:=
\clo(p|_{\contC})
=
\left(
\frac{p_v}{p(\contC)}
\right)_{v\in\contC},
\]
and define $\HC(p):=H(p^{\contC}), G(p):=H(p)-\HC(p)$. Thus \(G\) is the raw-minus-content residual. For
\(p,q\in\Delta_V^\circ\), define $\delta(p,q):=\HC(p)-\HC(q), \gamma(p,q):=G(p)-G(q)$. \begin{enumerate}
\item[{\rm(i)}]
For every \(p,q\in\Delta_V^\circ\), $H(p)-H(q)=\delta(p,q)+\gamma(p,q)$. Consequently, \(H\) and \(\HC\) rank \(p\) and \(q\) in opposite
strict orders if and only if $\delta(p,q) \bigl(\delta(p,q)+\gamma(p,q)\bigr)<0$. If $\delta(p,q)=0 ,\ \text{or},\  \delta(p,q)+\gamma(p,q)=0$, at least one score ties, so \((p,q)\) is not a strict reversal.
Define
\[
\mathcal R
:=
\left\{
(p,q)\in(\Delta_V^\circ)^2:
\bigl(\HC(p)-\HC(q)\bigr)
\bigl(H(p)-H(q)\bigr)<0
\right\},
\]
and let
\[
\mathcal A_V
:=
\operatorname{aff}(\Delta_V)
=
\left\{
x\in\mathbb R^V:
\sum_{v\in V}x_v=1
\right\}.
\]
Then \(\mathcal R\) is nonempty and relatively open in
\(\mathcal A_V\times\mathcal A_V\). Consequently, it has positive
induced \(2(D-1)\)-dimensional Lebesgue measure, equivalently
positive \(2(D-1)\)-dimensional Hausdorff measure.
\item[{\rm(ii)}]
There exist \(p_\star,q_\star\in\Delta_V^\circ\) such that $H(p_\star)<H(q_\star) ,\ \text{and},\  \HC(p_\star)>\HC(q_\star)$. \item[{\rm(iii)}]
For every \(\rho\in(0,\tfrac12]\) and every even positive integer \(N\)
with \(\rho N\) a positive integer, there exist
\(p_1,\ldots,p_N\in\Delta_V^\circ\) such that
\[
\operatorname{Top}_{\rho N}
\bigl((H(p_i))_{i=1}^N\bigr)
\cap
\operatorname{Top}_{\rho N}
\bigl((\HC(p_i))_{i=1}^N\bigr)
=
\varnothing.
\]
Thus the two exact top-\(\rho\) sets, in particular the top-\(20\%\)
sets used throughout, can have Jaccard similarity zero
(the exact top-\(K\) rule, larger scores first with lexicographic
tie-breaking, is fixed in App.~\ref{app:proofs}).
More generally, let \(P\) be a Borel-measurable
\(\Delta_V^\circ\)-valued random vector and define $X:=\HC(P), W:=G(P)=H(P)-\HC(P), Y:=X+W=H(P)$. Let $\nu:=\mathcal L(X,W)$ denote the joint law of \((X,W)\). For a real-valued random variable
\(Z\), define
\[
F_Z(z):=\mathbb P(Z\leq z),
\qquad
F_Z^{-1}(u)
:=
\inf\left\{
z\in\mathbb R:
F_Z(z)\geq u
\right\}.
\]
Fix \(\rho\in(0,1)\), set $q_C:=F_X^{-1}(1-\rho), q_H:=F_Y^{-1}(1-\rho)$, and assume $\mathbb P(X=q_C)=0, \mathbb P(Y=q_H)=0$. Define $A_C:=\{X\geq q_C\}, A_H:=\{Y\geq q_H\}$. Then $\mathbb P(A_C)=\mathbb P(A_H)=\rho$. Their intersection mass is
\[
\begin{aligned}
\omega_\rho
&:=
\mathbb P(A_C\cap A_H) \\
&=
\nu\left(
\left\{
(x,w)\in\mathbb R^2:
x\geq q_C,\quad x+w\geq q_H
\right\}
\right).
\end{aligned}
\]
The retention probabilities in the two directions coincide and equal $\operatorname{Ret}_\rho := \mathbb P(A_C\mid A_H) = \mathbb P(A_H\mid A_C) = \frac{\omega_\rho}{\rho}$. The population Jaccard similarity is $J_\rho := \frac{\mathbb P(A_C\cap A_H)} {\mathbb P(A_C\cup A_H)} = \frac{\omega_\rho}{2\rho-\omega_\rho}$. Moreover, $\max\{0,2\rho-1\} \leq \omega_\rho \leq \rho$, so every displayed denominator is strictly positive.
\end{enumerate}
\end{theorem}

Figure~\ref{fig:punchline}'s witness realizes the strict-reversal
condition of (i) with $\delta=-1.260$, $\gamma=+1.353$; the
$\nicefrac{999}{2000}$/$\nicefrac{1}{18000}$ witness realizes (ii)
($H=0.535<1.248$ yet $\HC=2.303>0.010$); and exact $\HC$ ties are
common in practice (many positions admit one plausible continuation),
so flip rates are always reported on non-tied pairs with the tie
fraction alongside. The population objects $\omega_\rho$,
$\operatorname{Ret}_\rho$, $J_\rho$ of (iii) are exactly what the
protocol estimates on any measured population; \emph{prevalence}
on real models is an empirical question the protocol answers,
alongside the companion estimands $(1-s)\HC$ and $\Hb+(1-s)\HC$
(\S\ref{sec:protocol}).

The $(\delta,\gamma)$ wedge (Figure~\ref{fig:wedge}a) is the paper's regime map: flip rates are geometry, populations decide how much of the geometry is occupied, and \S\ref{sec:protocol} measures the occupancy on real reasoning models, on RLVR selection, and on compression pipelines.

\subsection{What the scaffold can and cannot do for exploration}

Entropy interventions in RLVR are justified as protecting exploration,
ultimately answer diversity. The channels resolve that
justification into three exact identities, stated
below in the content-censored coordinates the proposition introduces
(the block-flag-plus-content-token view $Z_t$).

\begin{proposition}[Exploration bound with exact leakage]
\label{prop:exploration-full}
Fix a prompt \(x\), and let \(\mathbb P\) denote the probability law
governing the generated trajectory and answer at that prompt. Let $V=\scafS\mathbin{\dot\cup}\contC, \scafS\neq\varnothing, \contC\neq\varnothing$, be a finite vocabulary partition, and fix \(T\geq1\).
If variable-length trajectories are represented by padding, assume
that a distinguished symbol \(\mathrm{EOS}\in\scafS\) has already
been included in \(V\). Assume that EOS is absorbing: $\mathbb P(Y_t=\mathrm{EOS}\mid Y_{<t}=h)=1$ whenever \(h\) contains \(\mathrm{EOS}\). If all trajectories have
length \(T\), no EOS assumption is required.
Let \(Y_{1:T}\in V^T\). For every \(t\) and \(h\in V^{t-1}\), choose a
version of the actual conditional sampling law $p_t(\,\cdot\mid h) := \mathbb P(Y_t\in\cdot\mid Y_{<t}=h)$. On positive-probability histories this is the law actually generating
\(Y_t\). On null histories it may be chosen arbitrarily, subject to
the absorbing-EOS convention when applicable.
Let \(\alpha\) be a discrete random variable jointly distributed with
\(Y_{1:T}\), and assume \(H(\alpha)<\infty\). For \(h\in V^{t-1}\), set $s_t(h) := \sum_{v\in\scafS}p_t(v\mid h)$. When \(s_t(h)>0\), define $p_t^{\scafS}(v\mid h) := \frac{p_t(v\mid h)}{s_t(h)}, v\in\scafS$, and, when \(s_t(h)<1\), define $p_t^{\contC}(v\mid h) := \frac{p_t(v\mid h)}{1-s_t(h)}, v\in\contC$. If \(s_t(h)=0\), choose an arbitrary probability distribution
\(p_t^{\scafS}(\cdot\mid h)\) on \(\scafS\). If \(s_t(h)=1\), choose
an arbitrary probability distribution \(p_t^{\contC}(\cdot\mid h)\)
on \(\contC\). Define
\[
H_t^{\scafS}(h)
:=
H\!\left(p_t^{\scafS}(\,\cdot\mid h)\right),
\qquad
H_t^{\contC}(h)
:=
H\!\left(p_t^{\contC}(\,\cdot\mid h)\right),
\]
with the weighted-zero conventions $s_t(h)H_t^{\scafS}(h):=0 \quad\text{when }s_t(h)=0$, and $(1-s_t(h))H_t^{\contC}(h):=0 \quad\text{when }s_t(h)=1$. The arbitrary zero-mass choices do not affect any quantity below.
Write $s_t:=s_t(Y_{<t}), H_t^{\scafS}:=H_t^{\scafS}(Y_{<t}), H_t^{\contC}:=H_t^{\contC}(Y_{<t})$. Let \(\dagger\notin V\) be an erasure symbol, put $\mathcal Z:=\contC\cup\{\dagger\}$, and define \(Z_t\in\mathcal Z\) by
\[
Z_t
:=
\begin{cases}
Y_t, & Y_t\in\contC,\\
\dagger, & Y_t\in\scafS.
\end{cases}
\]
Thus \(Z_t\) records whether the token is scaffold or content and
retains the content-token identity while erasing the scaffold-token
identity.
Define $B_{\mathrm{gc}} := \sum_{t=1}^{T} \mathbb E\!\left[ \Hb(s_t)+(1-s_t)H_t^{\contC} \right]$, $B_{\mathrm{sc}} := \sum_{t=1}^{T} \mathbb E\!\left[s_tH_t^{\scafS}\right]$, and $\lek := \sum_{t=1}^{T} I(Z_t;Y_{<t}\mid Z_{<t})$. The \(t=1\) summand is zero because \(Y_{<1}\) and \(Z_{<1}\) are
constant. Then
\begin{align}
H(Z_{1:T})
&=
B_{\mathrm{gc}}+\lek,
\label{eq:exploration-censored-entropy}\\
\lek
&=
B_{\mathrm{sc}}
-
H(Y_{1:T}\mid Z_{1:T}),
\label{eq:exploration-leakage-slack}\\
H(\alpha)
&=
B_{\mathrm{gc}}+\lek
+H(\alpha\mid Z_{1:T})
-H(Z_{1:T}\mid\alpha).
\label{eq:exploration-answer-identity}
\end{align}
Consequently, $0\leq\lek\leq B_{\mathrm{sc}}$ and $H(\alpha) \leq B_{\mathrm{gc}}+\lek+H(\alpha\mid Z_{1:T})$. In particular, under the answer-sufficiency condition $H(\alpha\mid Z_{1:T})=0$, one has $H(\alpha) \leq B_{\mathrm{gc}}+\lek \leq B_{\mathrm{gc}}+B_{\mathrm{sc}}$. The equality cases are
\[
\lek=B_{\mathrm{sc}}
\quad\Longleftrightarrow\quad
H(Y_{1:T}\mid Z_{1:T})=0
\quad\Longleftrightarrow\quad
Y_{1:T}
\text{ is almost surely a function of }Z_{1:T},
\]
\begin{align*}
H(\alpha)
&=
B_{\mathrm{gc}}+\lek+H(\alpha\mid Z_{1:T})\\
&\quad\Longleftrightarrow\quad
H(Z_{1:T}\mid\alpha)=0\\
&\quad\Longleftrightarrow\quad
Z_{1:T}
\text{ is almost surely a function of }\alpha,
\end{align*}
and $H(\alpha\mid Z_{1:T})=0 \quad\Longleftrightarrow\quad \alpha \text{ is almost surely a function of }Z_{1:T}$. Moreover, $\lek=0 \quad\Longleftrightarrow\quad Z_t\mathrel{\perp\!\!\!\perp}Y_{<t}\mid Z_{<t} \quad\text{for every }t$. Equivalently, for every \(t\), there exists a stochastic kernel
\(K_t\) from \(\mathcal Z^{t-1}\) to \(\mathcal Z\) such that, for every
\(z\in\mathcal Z\), $\mathbb P(Z_t=z\mid Y_{<t}) = K_t(z\mid Z_{<t}) ,\ \text{almost surely}$. When \(T=1\), one has \(\lek=0\) identically. If \(T\geq2\),
\(\contC\) contains two distinct symbols, and \(\scafS\) contains two
distinct nonterminal symbols, then the upper leakage bound and the
answer bound can hold with equality simultaneously, with
\(\lek=B_{\mathrm{sc}}>0\). Here ``nonterminal'' means any scaffold
symbol in the fixed-length case and any element of
\(\scafS\setminus\{\mathrm{EOS}\}\) in the absorbing-EOS case.
These simultaneous equalities remain sharp in the limiting sense
under full-support requirements: they can be approached by kernels
having full support at every pre-termination history while retaining
deterministic absorption after EOS.
\end{proposition}

The two-step ``coin in the scaffold'' policy (a fair coin flipped
between two connectives, converted later into different answers)
attains the equality pair: $H(\alpha)=\log2$ with $B_{\mathrm{gc}}=0$
and $\lek=B_{\mathrm{sc}}=\log2$, worked exactly in
App.~\ref{app:worked}. The residual $H(\alpha\mid Z_{1:T})$ is the
answer-sufficiency deficit: it is nonzero when answer-bearing words
(``yes''/``no'') sit in $\scafS$, which is why the protocol's
exploration analyses use $\scafS$ with such words excluded. Scaffold
inertness, block and content predictions depending on history only
through $Z_{<t}$, is exactly the displayed kernel condition, and
forces $\lek=0$.

The scaffold-internal channel can bankroll exploration, but only by laundering bits through history; absent that leakage, gate plus content is the whole exploration budget, and $s\HS$, however large, buys nothing. Crucially, $\lek$ is not left as an unmeasurable abstraction: the next subsection gives the sequence-level object it lives in and a substitution estimator for its per-step engine.

\subsection{Eventual content: the first-passage channel}
\label{sec:firstpass}

The per-position $\pC$ answers ``if content comes \emph{now}, what is it?''. The fork question practitioners actually ask is ``what will the model \emph{eventually} commit to?'', and the two differ exactly when the model can defer through scaffolding. Define the first content time and the \emph{eventual-content distribution}
\[
\tau_t=\inf\{j\ge t: Y_j\in\contC\},\qquad q_t(v)=\mathbb{P}\big(Y_{\tau_t}=v\mid h_t\big),
\]
with $h_t$ the prefix. One step of total probability gives the exact recursion
\begin{equation}
q_t \;=\;(1-s_t)\,\pC_t\;+\;s_t\sum_{w\in\scafS}\pS_t(w)\;q_{t+1}^{(w)},
\label{eq:firstpass}
\end{equation}
where $q_{t+1}^{(w)}$ is the eventual-content law after appending scaffold token $w$: eventual content is the content channel now, plus a scaffold-weighted mixture of eventual contents after each connective. Define the per-step \emph{scaffold-routing information}
\[
\kappa_t\;=\;I\big(Y_{\tau_{t+1}};\,Y_t\mid h_t,\;Y_t\in\scafS\big)
\;=\;H\Big(\textstyle\sum_w \pS_t(w)\,q^{(w)}_{t+1}\Big)-\sum_w \pS_t(w)\,H\big(q^{(w)}_{t+1}\big)\;\ge\;0 .
\]
$\kappa_t=0$ at every prefix is exactly one-step scaffold inertness (the mixture components coincide), which does \emph{not} imply $\lek=0$ (Remark~\ref{rem:kappa-not-lambda}); large $\kappa_t$ is a measured \emph{witness} of leakage: $\kappa_t>0$ at any prefix implies $\lek>0$, while the converse fails, so $\kappa_t$ certifies leakage without estimating it (App.~\ref{app:proofs}). Both objects are estimable by short truncated rollouts: $q_t$ by sampling to the first content token, $\kappa_t$ by substituting alternative connectives at scaffold positions and comparing the resulting first-content laws (with clustered intervals; App.~\ref{app:expdetails}). $H(q_t)$ is the reachability-aware fork estimand this paper recommends when $s_t$ is large, where ranking by $\HC_t$ alone would ignore that content may be deferred; the protocol compares $H$, $\HC$, $(1-s)\HC$, $\Hb+(1-s)\HC$, and $H(q_t)$ head-to-head as selection signals (\S\ref{sec:protocol}). (Proofs: App.~\ref{app:proofs}.)

\subsection{Beyond entropy: the full instrument calculus}
\label{sec:instruments}
\label{app:instruments}

Entropy is one instrument among many trained on CoT distributions: KL
penalties to a reference policy, JS drift between checkpoints or models,
Hellinger/Fisher--Rao and cosine similarities in analysis pipelines.

\begin{theorem}[Instrument calculus for CoT]
\label{thm:instruments}
Let $p=(s\,\pS,(1-s)\pC)$, $q=(t\,q^{\scafS},(1-t)q^{\contC})$,
$\bar s=\tfrac{s+t}2$.
\emph{(i)} $\JS(p,q)=\JS_b(s,t)+\bar s\,\JS_{\lambda_S}(\pS,q^{\scafS})
+(1-\bar s)\,\JS_{\lambda_C}(\pC,q^{\contC})$ with skews
$\lambda_S=\tfrac{s}{2\bar s}$, $\lambda_C=\tfrac{1-s}{2(1-\bar s)}$: a
\emph{three}-channel skew-JS law (the singleton law has two channels).
\emph{(ii)} At matched gate ($s=t$), every $f$-divergence obeys the
mixture law $D_f(p\|q)=s\,D_f(\pS\|q^{\scafS})+(1-s)\,D_f(\pC\|q^{\contC})$,
and R\'enyi an exact channel log-mix; Bhattacharyya splits blockwise and
Fisher--Rao obeys a two-dihedral spherical law of cosines.
\emph{(iii)} Under scaffold locking ($\pS=q^{\scafS}$, $s=t\to1$, the
regime RLVR produces), content sensitivity decays as $(1-s)^{\gamma}$
with $\gamma=2$ (cosine), $1$ ($f$-divergences, Euclidean),
$\nicefrac12$ (Hellinger, Fisher--Rao, $\sqrt{\JS}$), and exactly $0$
for the channel-resolved Aitchison report: every probability-side
instrument goes content-blind at a known polynomial rate while the
channels do not.
\end{theorem}

\begin{corollary}[The KL penalty is a mixed-channel instrument]
\label{cor:klpenalty}
The RLVR objective's penalty $\beta\KL(\pi_\theta\|\pi_{\mathrm{ref}})$
charges, by Proposition~\ref{thm:chain}(ii),
$\beta[\KL_b+s\KL^{\scafS}+(1-s)\KL^{\contC}]$. At matched, locked
scaffolding a content divergence $\varepsilon$ costs exactly
$\beta(1-s)\varepsilon$: as training hardens the gate, a
fixed-coefficient penalty constrains content drift \emph{linearly less},
precisely when exploration analyses care about content. Gate and
connective-habit drift can meanwhile dominate the recorded penalty.
Channel-resolved penalties (charge $\KL^{\contC}$ directly) restore an
$s$-free content trust region; Stage 2b's engine exposes this as a
one-line variant.
\end{corollary}

Matched gates are an idealization: two checkpoints, or two models,
generally sit at different gate masses. The unmatched analysis is
stated with a single sink coordinate carrying the scaffold mass, with
no loss: under scaffold locking ($p^{\scafS}=q^{\scafS}$) the whole
block lumps exactly into one coordinate, an exact identity for every
$f$-divergence (machine-verified to $10^{-16}$).

\begin{proposition}[Unmatched-sink regimes]
\label{prop:unmatched-sink-regimes}
Fix an integer \(D\geq3\), and index the \(D-1\) content coordinates
by $\mathcal I:=\{1,\ldots,D-1\}$. Write \(\mathcal S^{m}\) for the open simplex
\(\{x\in(0,1)^{m}:\mathbf1_m^\top x=1\}\).
For probability vectors \(x,y\) of the same finite dimension, define
the equal-weight, unrooted Jensen--Shannon divergence by
\[
\operatorname{JS}(x,y)
:=
\frac12
\operatorname{KL}\!\left(
x\middle\|\frac{x+y}{2}
\right)
+
\frac12
\operatorname{KL}\!\left(
y\middle\|\frac{x+y}{2}
\right).
\]
For \(\lambda\in(0,1)\), define $\mu_\lambda(x,y) := \lambda x+(1-\lambda)y$ and
\[
\operatorname{JS}_\lambda(x,y)
:=
\lambda
\operatorname{KL}\bigl(x\|\mu_\lambda(x,y)\bigr)
+
(1-\lambda)
\operatorname{KL}\bigl(y\|\mu_\lambda(x,y)\bigr).
\]
Let $h_b(\lambda) := -\lambda\log\lambda -(1-\lambda)\log(1-\lambda)$. For every \(n\geq1\), let
\[
u_n,v_n\in(0,1),
\qquad
\pi_n=(\pi_{n,i})_{i\in\mathcal I},
\quad
\rho_n=(\rho_{n,i})_{i\in\mathcal I}
\in\mathcal S^{D-1},
\]
and define $p_n=(1-u_n,u_n\pi_n), q_n=(1-v_n,v_n\rho_n)$. Assume $u_n\longrightarrow0, v_n\longrightarrow0$. Write $s_n:=1-u_n, t_n:=1-v_n, \bar s_n:=\frac{s_n+t_n}{2}$, and set $a_n := 1-\bar s_n = \frac{u_n+v_n}{2}, \lambda_n := \frac{u_n}{u_n+v_n}\in(0,1)$. Then the following identity is exact:
\begin{equation}
\operatorname{JS}(p_n,q_n)
=
a_n
\left[
\log 2-h_b(\lambda_n)
+
\operatorname{JS}_{\lambda_n}(\pi_n,\rho_n)
\right]
+
R_n,
\label{eq:unmatched-js-exact}
\end{equation}
where
\begin{equation}
R_n
:=
\frac12
\left[
(1-u_n)\log\frac{1-u_n}{1-a_n}
+
(1-v_n)\log\frac{1-v_n}{1-a_n}
\right].
\label{eq:unmatched-js-remainder}
\end{equation}
Whenever \(a_n<1/2\),
\begin{equation}
0
\leq
R_n
\leq
\frac{a_n^2(2\lambda_n-1)^2}
     {2(1-2a_n)}
\leq
\frac{a_n^2}{2(1-2a_n)}.
\label{eq:unmatched-js-remainder-bound}
\end{equation}
Consequently,
\begin{equation}
\operatorname{JS}(p_n,q_n)
=
(1-\bar s_n)
\left[
\log 2-h_b(\lambda_n)
+
\operatorname{JS}_{\lambda_n}(\pi_n,\rho_n)
\right]
+
O\!\left((1-\bar s_n)^2\right),
\label{eq:unmatched-js-first-order}
\end{equation}
where the \(O(\cdot)\) constant is absolute and independent of
\(\lambda_n,\pi_n,\rho_n\).
Define the exact Bernoulli sink channel and the content channel by
\[
B_n
:=
\operatorname{JS}
\bigl((1-u_n,u_n),(1-v_n,v_n)\bigr),
\qquad
C_n
:=
a_n\operatorname{JS}_{\lambda_n}(\pi_n,\rho_n).
\]
Then
\begin{equation}
\operatorname{JS}(p_n,q_n)=B_n+C_n,
\label{eq:unmatched-js-channel-decomposition}
\end{equation}
where $B_n = a_n\bigl[\log 2-h_b(\lambda_n)\bigr]+R_n$. Uniformly over all admissible content pairs, $B_n=O(a_n), C_n=O(a_n)$. Furthermore, for every \(n\), $B_n=0 \Longleftrightarrow \lambda_n=\frac12, C_n=0 \Longleftrightarrow \pi_n=\rho_n$. If the nonsink masses have unmatched rates in the sense that
\begin{equation}
\min\{\lambda_n,1-\lambda_n\}
\longrightarrow0,
\label{eq:unmatched-rate-condition}
\end{equation}
then, uniformly over all content sequences, $C_n=o(a_n), B_n=a_n\log 2+o(a_n)$, and hence
\begin{equation}
\operatorname{JS}(p_n,q_n)
=
(1-\bar s_n)\log 2
+
o(1-\bar s_n).
\label{eq:unmatched-js-rate-mismatch}
\end{equation}
Condition \eqref{eq:unmatched-rate-condition} holds, in particular,
when \(\lambda_n\to0\) or \(\lambda_n\to1\).
Under the separate comparable-rate condition that there exists
\(\varepsilon\in(0,1/2)\) such that, eventually, $\varepsilon \leq \lambda_n \leq 1-\varepsilon$, one has
\begin{equation}
B_n=\Theta(a_n)
\quad\Longleftrightarrow\quad
\liminf_{n\to\infty}
\left|
\lambda_n-\frac12
\right|
>0,
\label{eq:unmatched-sink-theta}
\end{equation}
and
\begin{equation}
C_n=\Theta(a_n)
\quad\Longleftrightarrow\quad
\liminf_{n\to\infty}
\operatorname{JS}_{\lambda_n}(\pi_n,\rho_n)
>0.
\label{eq:unmatched-content-theta}
\end{equation}
In particular, if $\pi_n\equiv\pi, \rho_n\equiv\rho, \pi\neq\rho$, then \(C_n=\Theta(a_n)\). Under these fixed-content conditions, both
channels are \(\Theta(a_n)\) if and only if $\liminf_{n\to\infty} \left| \lambda_n-\frac12 \right| >0$. \end{proposition}

The regime reading is the sharpest form of
Theorem~\ref{thm:instruments}(iii): when two runs lock their gates at
different rates ($\lambda_n\to0$ or $1$), JS drift between them reads
$(1-\bar s_n)\log2+o(1-\bar s_n)$ \emph{regardless of content}: the
instrument stops measuring content drift and starts measuring gate
mismatch, with an exact constant. Only at comparable rates do both
channels survive at order $1-\bar s_n$, and then
\eqref{eq:unmatched-sink-theta} and
\eqref{eq:unmatched-content-theta} say exactly when each does. The
channel-resolved report is immune by construction.

\begin{theorem}[What CoT distributions say downstream]
\label{thm:downstream}
Let $\dvec=p-q$ and $E\in\mathbb{R}^{|V|\times d}$ stack token
embeddings (or any value vectors bounded by $B$), $G=EE^{\top}$.
\emph{(i)} Attainable downstream distances over bounded values fill
$[0,\,B\|\dvec\|_1]$: channel statistics alone identify nothing about
downstream effect.
\emph{(ii)} Exactly, $\dvec=\Delta s\,\kappa+\bar s(\Delta^{\scafS},\mathbf 0)
+(1-\bar s)(\mathbf 0,\Delta^{\contC})$ with
$\kappa=(m^{\scafS},-m^{\contC})$ the gate direction and $m$ the midpoint
channels, so $d_G^2=\dvec^{\top}G\dvec$ expands into six channel terms
whose gate weight is
$q_G(\kappa)=\|\bar v_{\scafS}-\bar v_{\contC}\|_2^2$, the squared
distance between the mean scaffold and mean content embeddings, a
per-model constant computable exactly from model weights.
\emph{(iii)} With $d=1$ and the advantage as the value vector, (ii) is
the between/within-block structure of Proposition~\ref{prop:dynamics}:
one bilinear calculus underlies the geometry, the training dynamics, and
downstream-effect identification.
\end{theorem}

The tokenizer objection, and the instrument trilemma behind it, is
governed by a no-free-lunch theorem we now state in full; its proof,
and the companion notation it uses, are in App.~\ref{app:proofs}.

\begin{theorem}[No free lunch under compositional axioms]
\label{thm:no-free-lunch-corrected}
Fix an integer \(D\ge 3\).  For \(m\ge 2\), write
\[
\mathcal S^m
=
\left\{p\in(0,1)^m:\mathbf 1_m^\top p=1\right\},
\qquad
\overline{\mathcal S^m}
=
\left\{p\in[0,1]^m:\mathbf 1_m^\top p=1\right\}.
\]
For \(\pi\in\mathcal S^{D-1}\) and \(s\in(0,1)\), put $p_s^\pi=(s,(1-s)\pi)\in\mathcal S^D$. For integers \(m,M\ge2\), call
\(T:\mathcal S^m\to\mathcal S^M\) an admissible Markov map if $T\in[0,\infty)^{M\times m}, \mathbf1_M^\top T=\mathbf1_m^\top$, and \(T\) has no zero row.  This condition ensures that \(T\) maps
the open simplex into the open simplex.
\begin{enumerate}
\item[\textnormal{(i)}]
Let
\(\delta:\mathcal S^D\times\mathcal S^D\to[0,\infty)\)
satisfy \(\delta(p,p)=0\) for every \(p\).  If \(\delta\not\equiv0\)
and $\delta(c\oplus p,c\oplus q)=\delta(p,q) ,\ (c,p,q\in\mathcal S^D)$, then there is no finite jointly continuous extension $\bar\delta: \overline{\mathcal S^D}\times\overline{\mathcal S^D} \longrightarrow[0,\infty)$ that agrees with \(\delta\) on the open simplex.  Indeed, any such
extension would automatically vanish on the diagonal of the closed
simplex by continuity.
Independently of perturbation invariance, if a dissimilarity
vanishing on the interior diagonal admits such an extension, then $\lim_{s\uparrow1} \sup_{\pi,\rho\in\mathcal S^{D-1}} \delta(p_s^\pi,p_s^\rho)=0$. More precisely, there exists a metric \(\delta_0\) on \(\mathcal S^D\)
that admits a finite jointly continuous extension to
\(\overline{\mathcal S^D}\times\overline{\mathcal S^D}\) and such
that, for every \(\gamma>0\) and every fixed
\(\pi\ne\rho\), $\frac{\delta_0(p_s^\pi,p_s^\rho)}{(1-s)^\gamma} \longrightarrow\infty ,\ (s\uparrow1)$. \item[\textnormal{(ii)(a)}]
Let \(\delta:\mathcal S^D\times\mathcal S^D\to[0,\infty)\)
satisfy \(\delta(p,p)=0\) for every \(p\).  Suppose \(\delta\) is
sink-separable at coordinate \(0\): there are
finite real-valued functions \(\phi_0\) and \(\Psi_0\) such that $\delta(p,q)^2 = \phi_0\!\left(b_0(p),b_0(q)\right) + \Psi_0\!\left(p^{(-0)},q^{(-0)}\right)$ for all \(p,q\in\mathcal S^D\).
Then, for every \(\pi\in\mathcal S^{D-1}\) and every
\(\sigma\in\mathfrak S_{D-1}\), the map $s\longmapsto \delta(p_s^\pi,p_s^{\sigma\pi})$ is constant on \((0,1)\).  Consequently, if
\(\sigma\pi\ne\pi\), if $\delta(p_{s_0}^\pi,p_{s_0}^{\sigma\pi})>0$ for some \(s_0\in(0,1)\), and if $\delta(p_s^\pi,p_s^{\sigma\pi})\longrightarrow0 ,\ (s\uparrow1)$, then \(\delta\) is not sink-separable at coordinate \(0\), and
hence is not Sep-A.
Moreover, perturbation invariance alone implies $\delta(p_s^\pi,p_s^\rho) = \delta(p_t^\pi,p_t^\rho)$ for all \(\pi,\rho\in\mathcal S^{D-1}\) and all \(s,t\in(0,1)\).
In particular, the separability obstruction applies to every
\(f\)-divergence generated by a finite non-affine convex function
\(f:(0,\infty)\to\mathbb R\) with \(f(1)=0\), to
\((D_f)^\alpha\) for every \(\alpha>0\),
and to Euclidean distance, Fisher--Rao distance, and cosine
dissimilarity.
\item[\textnormal{(ii)(b)}]
Let \(\{\delta_m\}_{m\ge3}\) satisfy \textnormal{(S1)--(S4)} in
every dimension.  Then this family cannot be nonexpansive under
every admissible Markov map between dimensions at least \(3\).
More precisely, if
\(E_{D,k}^{(j)}\) splits
coordinate \(j\) into \(k\ge2\) equal parts, then
\[
\sup_{\substack{p,q\in\mathcal S^D\\p\ne q}}
\frac{
d_{A,D+k-1}\!\left(E_{D,k}^{(j)}p,E_{D,k}^{(j)}q\right)}
{d_{A,D}(p,q)}
=
\sqrt{\frac{kD}{D+k-1}}>1.
\]
The squared ratio \(kD/(D+k-1)\), not its square root, is the top
generalized eigenvalue of the corresponding quadratic-form pencil
restricted to \(\mathcal H_D\).  No family of positive
dimension-dependent rescalings
\(\delta_m=\kappa_m d_{A,m}\), \(\kappa_m>0\), restores the
data-processing inequality.  The incompatibility
already occurs under a strictly positive \(3\times3\)
column-stochastic map at fixed dimension \(3\).
For every finite convex generator \(f\) with \(f(1)=0\), the
functional \(D_f\) is invariant under every equal split and
nonexpansive under every admissible Markov map.  If \(f\) is
non-affine, then \(D_f\) and \((D_f)^\alpha\), \(\alpha>0\), are not
Sep-A by part~\textnormal{(ii)(a)}.  Fisher--Rao distance is likewise
equal-split invariant and Markov-nonexpansive, but is not Sep-A.
\item[\textnormal{(iii)}]
For every \(D\ge3\), Hilbert's projective metric
\[
d_{H,D}(p,q)
=
\operatorname{osc}\!\left(\log p-\log q\right),
\qquad
\operatorname{osc}(z)=\max_i z_i-\min_i z_i,
\]
is perturbation-invariant, powering-homogeneous, and
permutation-invariant.
For every \(\alpha\in\mathbb R\), $d_{H,D}(\alpha\odot p,\alpha\odot q) =|\alpha|d_{H,D}(p,q)$. A congruent Markov embedding is an admissible Markov map whose
column supports are pairwise disjoint.  For every \(M\ge D\), each
such map \(E:\mathcal S^D\to\mathcal S^M\) is an isometry.  Moreover,
\(d_H\) is nonexpansive under every admissible Markov map
\(T:\mathcal S^D\to\mathcal S^M\).  If a fixed admissible map
\(T\) is entrywise positive, then
there is a map-dependent coefficient \(\tau(T)<1\) such that $d_{H,M}(Tp,Tq)\le \tau(T)d_{H,D}(p,q) ,\ (p,q\in\mathcal S^D)$. Nevertheless, \(d_H\) is not sink-separable at coordinate \(0\),
and hence is not Sep-A, and the norm it induces on the CLR
hyperplane is not generated by an inner product.
\end{enumerate}
\end{theorem}

\begin{remark}[Tokenizers, by theorem]
\label{rem:tokenizer}
Retokenizing a vocabulary item into $k$ pieces is a congruent Markov
embedding. Theorem~\ref{thm:no-free-lunch-corrected}(ii)(b) gives the
exact expansion constant $\sqrt{kD/(D+k-1)}$ for channel-separating
instruments, and parts (i), (ii)(a), and (iii) the trichotomy behind
it: $f$-divergences are split-invariant and Markov-monotone but
provably non-separable; the separating Aitchison family cannot be
made Markov-nonexpansive at any dimension-dependent rescaling; and
the Hilbert metric buys every invariance except separation and an
inner product. Cross-tokenizer comparisons of channel statistics
therefore face a forced choice (separation or tokenization-robustness,
never both); our cross-model tables compare within tokenizer families
and report the choice explicitly. The common objection that
scaffold-set analyses ``depend on the tokenizer'' is an instance of a
theorem with an exact constant, not an oversight to hide.
\end{remark}

\section{Exact witnesses and in-silico validation}
\label{app:silico}

\begin{figure}[t]
\centering
\includegraphics[width=0.98\textwidth]{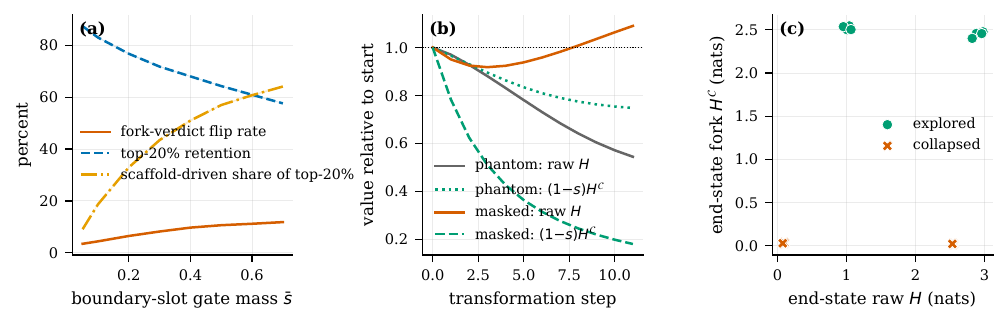}
\caption{\textbf{The three simulations.} (a) SYN-1: sweeping the boundary-slot
gate mass $\bar s$, the fork-verdict flip rate, the top-$20\%$ retention
between conventions, and the scaffold-driven share of the raw top-$20\%$
set (mean$\pm$sd, 3 seeds, $5\times10^{4}$ positions each). (b) SYN-2,
one fixed position population transformed per arm: in the phantom arm the
transformation moves only the nuisance group, so raw $H$ falls $46\%$
while the content channel is exactly invariant by
Proposition~\ref{prop:invariance}; in the masked arm a genuine content
collapse coexists with \emph{rising} raw $H$. (c) SYN-3, end states of the
in-silico GRPO $2\times2$: fork-content entropy separates exploring from
collapsed policies without overlap, and raw $H$ does not, because the
format-reward weight moves raw $H$ by up to $2.5$ nats at unchanged
exploration.}
\label{fig:wedge}
\label{fig:collapse}
\label{fig:insilico}
\end{figure}

All numbers below are produced by exact evaluations or seeded CPU simulations completed within minutes; the identities and both optimizer-regime formulas of Corollary~\ref{cor:reinforce} are additionally machine-checked, and Proposition~\ref{prop:dynamics} is checked against realized finite-difference entropy steps along the in-silico training runs (median relative error $0.0\%$ in the plain arm and $3.2\%$ in the larger-step intervention arm).

\subsection{Exact witnesses and calibrated flip rates}
\label{sec:synthetic}

\begin{table}[t]
\centering
\small
\caption{\textbf{Exact witness for Theorem~\ref{thm:reversal}(i)} (the Figure~\ref{fig:punchline} pair; all values in nats, computed in closed form by exact calculations). Raw entropy selects $t$; the content channel says $u$ forks $16\times$ harder. Boldface marks each convention's verdict.}
\label{tab:witness}
\begin{tabular}{lccccccc}
\toprule
position & $s$ & $H$ & $\Hb(s)$ & $s\HS$ & $(1{-}s)\HC$ & $\HC$ & $G$ \\
\midrule
$t$ (boundary) & $0.50$ & $\mathbf{1.540}$ & $0.693$ & $0.805$ & $0.042$ & $0.083$ & $1.456$ \\
$u$ (computation) & $0.02$ & $1.447$ & $0.098$ & $0.032$ & $1.316$ & $\mathbf{1.343}$ & $0.103$ \\
\bottomrule
\end{tabular}
\end{table}

Table~\ref{tab:witness} instantiates the reversal with realistic magnitudes, and Figure~\ref{fig:wedge} measures how often it happens in a stylized population whose stated target, a low-entropy majority with a boundary-concentrated high-entropy minority in the spirit of \citet{wang2025beyond,zhao2026shorthand}, is \emph{checked in-run} instead of asserted: at $\bar s{=}0.4$ the generator reports $65\%$ of positions below $0.7$ nats, a top-20\% threshold of $1.38$ nats, and boundary slots occupying $58\%$ of the raw top-20\% against a $20\%$ base rate ($2.8\times$ enrichment); these are simulation choices, not model measurements, and the free parameter $\bar s$ is swept precisely because no published statistic pins it (generator and targets: App.~\ref{app:expdetails}). Three seeds, $50{,}000$ positions each; flip rates are computed on non-tied pairs with the tie fraction reported. As $\bar s$ rises through $0.2$--$0.6$: the fork-verdict flip rate climbs $6.5\%\to11\%$ ($12\%$ at $\bar s{=}0.7$); top-20\% retention falls $77\%\to61\%$; and the scaffold-driven share of the raw top-20\% set ($\Hb+s\HS>(1-s)\HC$) grows $33\%\to61\%$, tracking the boundary share of the selected set. At the largest gate mass on the grid ($\bar s=0.7$) the sweep reaches $11.8\%$ flips, $57.6\%$ retention, and a $64.1\%$ scaffold-driven share, which calibrates the plausible range before any model is measured. The disjointness construction of Theorem~\ref{thm:reversal}(ii) is exercised exactly (measured Jaccard $0.00$).

Figure~\ref{fig:collapse} stages the two failure modes of collapse monitoring. The phantom arm moves the population inside the nuisance group only, so the content channel's silence is theorem-guaranteed; the masked arm applies a genuine content collapse, and raw entropy gets both cases wrong in opposite directions: it reports a $46\%$ collapse that is $100\%$ scaffold-side, and it reports $+9\%$ against a $79\%$ content collapse. These are constructions, not measurements; their job, as in the companion's synthetic section, is to isolate the mechanism the protocol then hunts in the wild.

\subsection{In-silico RLVR: which channel carries the collapse, and which predicts exploration}
\label{sec:insilico}

The toy testbed (spec in App.~\ref{app:expdetails}) is built so that ground truth about exploration exists: each content fork accepts two interchangeable tokens, and a policy explored iff it kept both alive, measured as exact pass@8 under the worst-case preference shift. Its role is \emph{mechanism validation}, not independent evidence: the environment is authored, the sample is 16 runs, and the successful intervention is of the same family the content channel measures, so the model-scale versions of these claims live in \S\ref{sec:protocol}. With that scope fixed, three results validate the machinery under the sampled REINFORCE update itself (the regime of Corollary~\ref{cor:reinforce}, not the natural-gradient idealization). \emph{First}, collapse is multi-channel and the attribution is exact: in the format-rewarded run of Figure~\ref{fig:insilico}a, $29\%$ of the $3.55$-nat collapse is gate$+$scaffold, and the per-step channel prediction matches realized entropy differences (median relative error $0.0\%$; $3.2\%$ in the larger-step intervention arm). \emph{Second}, the phantom is manufacturable by reward shaping alone: switching the format weight $\lambda$ from $0$ to $0.35$ moves end-state raw entropy by $2.45$ nats (no intervention) and $1.89$ nats (with intervention) while moving the exploration probe by $0.000$ and $0.035$ respectively. \emph{Third}, within this testbed the content channel is the cleaner exploration signal: end-state fork $\HC$ separates exploring from collapsed policies without overlap (AUC $1.00$ vs $0.75$ for raw $H$ over the 16 endpoints, Figure~\ref{fig:insilico}b), because the format knob moves raw $H$ orthogonally to exploration; the pattern echoes the finding that entropy interventions succeed by protecting a minority of decisive tokens \citep{wang2025beyond,cui2025entropy}, and its model-scale test is Stage 2b.

\section{Worked example: leakage, exactly}
\label{app:worked}

The leakage term $\lek$ of Theorem~\ref{prop:exploration} is the least
familiar object in the paper, so we compute it exactly on the smallest
possible example (four symbols, two steps; all trajectories are
enumerated exactly, without sampling).

\textbf{Alphabet.} $\scafS=\{\text{``Wait''},\text{``So''}\}$,
$\contC=\{a,b\}$; the answer $\alpha$ is the subsequence of content
tokens.

\textbf{Case 1 (leaky, and tight).} Step 1 emits a scaffold token,
uniformly: $p_1=(\tfrac12,\tfrac12,0,0)$. Step 2 converts the scaffold
identity into content deterministically: after ``Wait'' emit $a$, after
``So'' emit $b$. Then per-position gate$+$content budgets are $0$ at both
steps (step 1: $s=1$, $\Hb(1)=0$; step 2: deterministic given the
prefix), yet $H(\alpha)=\log 2$: the answer is a fair coin. The
decomposition locates the coin precisely:
\[
H(\alpha)=0.6931,\qquad
\textstyle\sum_t\mathbb{E}[\Hb+(1-s)\HC]=0,\qquad
\lek=0.6931=\textstyle\sum_t\mathbb{E}[s\HS],
\]
so the bound $H(\alpha)\le\text{gate+content}+\lek$ holds with equality
\emph{and} the envelope $\lek\le\sum_t\mathbb{E}[s\HS]$ holds with
equality: all the exploration was banked inside the scaffold channel at
step 1 (``which connective?'' $=\log 2$ nats) and withdrawn as content at
step 2 through history dependence. Scaffold-internal entropy is not
inert here, and no bound that ignores $\lek$ can be correct, which is
why Theorem~\ref{prop:exploration} carries the exact correction
in place of an assumption-free inequality.

\textbf{Case 2 (inert).} Same alphabet, with step-2 laws depending on the
past only through the censored view $Z_1$ (block flag $+$ content
identity): $p_1=(0.3,0.3,0.28,0.12)$, and $p_2$ is one distribution if
$Y_1\in\contC$ and another if $Y_1\in\scafS$, but never depends on
\emph{which} scaffold token occurred. Enumeration gives
\[
H(\alpha)=1.7304\;\le\;1.9993=\textstyle\sum_t\mathbb{E}[\Hb+(1-s)\HC],
\qquad \lek=0\ (\le 10^{-12}),
\]
with $\sum_t\mathbb{E}[s\HS]=0.5734$ entirely idle: under scaffold
inertness the scaffold-internal channel really does buy no answer
diversity, however large it is.

The pair of cases summarizes what the scaffold can do for
exploration: nothing, unless the model routes information through the
identity of its connectives, and the amount routed is exactly $\lek$.

\section{Proofs}
\label{app:proofs}

Throughout, $p\in\Sm{D}$ factors as
$p=(s\,\pS,(1-s)\,\pC)$ per \eqref{eq:factor}, $m=|\scafS|$,
$k=|\contC|$, $D=m+k$, and all logs are natural. Statements labeled
\emph{(classical)} are standard and included for completeness in this
setting. The main-text statements of Theorems~\ref{thm:chain},
\ref{thm:reversal}, \ref{thm:no-free-lunch-corrected} and
Propositions~\ref{prop:invariance}, \ref{prop:dynamics},
\ref{prop:exploration}, \ref{prop:unmatched-sink-regimes} are stated
in full generality with all conventions inline; the proofs below use
exactly those conventions, and each is followed by the remarks that
delimit its scope (boundary cases, random prompts, optimizer scope,
necessity of the interior domain). Every statement is additionally
machine-verified by 215 checks: exact joint-law
enumeration for Theorem~\ref{prop:exploration} and its equality
cases, reversal witnesses and population formulas, the natural-gradient
pseudoinverse solution, sharp surplus attainers, maximal invariance and
temperature reversals, the no-free-lunch split constant and Hilbert
claims, and the unmatched-sink law and its regimes. The numerical
witnesses and instrument calculus were also verified computationally.

\subsection{Proposition~\ref{thm:chain}: chain rules and sharp signed
  surplus bounds}

\begin{proof}
Fix \(r\in\Delta_{\mathrm{gi}}(V)\) and write
\[
\rho:=s_r\in(0,1).
\]
For \(v\in\scafS\) with \(r_v>0\),
\[
r_v=\rho r_v^{\scafS},
\]
and hence
\[
-r_v\log r_v
=
-\rho r_v^{\scafS}\log\rho
-\rho r_v^{\scafS}\log r_v^{\scafS}.
\]
This identity extends to \(r_v=0\) under the convention
\(0\log0=0\). Therefore
\begin{align*}
-\sum_{v\in\scafS}r_v\log r_v
&=
-\rho\log\rho
\sum_{v\in\scafS}r_v^{\scafS}
-\rho\sum_{v\in\scafS}
r_v^{\scafS}\log r_v^{\scafS}\\
&=
-\rho\log\rho+\rho\,\HS(r).
\end{align*}
Similarly,
\[
-\sum_{v\in\contC}r_v\log r_v
=
-(1-\rho)\log(1-\rho)
+(1-\rho)\,\HC(r).
\]
Adding the two disjoint block contributions proves \textup{(i)}.
We next prove \textup{(ii)}. Suppose first that \(p\ll q\).
Since \(s,t\in(0,1)\),
\[
p^{\scafS}\ll q^{\scafS},
\qquad
p^{\contC}\ll q^{\contC}.
\]
For \(v\in\scafS\) with \(p_v>0\), one has \(q_v>0\), and
\[
\log\frac{p_v}{q_v}
=
\log\frac{s}{t}
+
\log\frac{p_v^{\scafS}}{q_v^{\scafS}}.
\]
Consequently,
\begin{align*}
\sum_{\substack{v\in\scafS\\p_v>0}}
p_v\log\frac{p_v}{q_v}
&=
\sum_{\substack{v\in\scafS\\p_v>0}}
s p_v^{\scafS}
\left(
\log\frac{s}{t}
+
\log\frac{p_v^{\scafS}}{q_v^{\scafS}}
\right)\\
&=
s\log\frac{s}{t}
+
s\,\KL(p^{\scafS}\|q^{\scafS}).
\end{align*}
The analogous calculation on \(\contC\) gives
\[
\sum_{\substack{v\in\contC\\p_v>0}}
p_v\log\frac{p_v}{q_v}
=
(1-s)\log\frac{1-s}{1-t}
+
(1-s)\,\KL(p^{\contC}\|q^{\contC}).
\]
Adding these identities proves \textup{(ii)} when \(p\ll q\).
Suppose now that \(p\not\ll q\). Then some \(v\in V\) satisfies
\(p_v>0=q_v\), so
\[
\KL(p\|q)=+\infty.
\]
If \(v\in\scafS\), then
\[
p_v^{\scafS}>0=q_v^{\scafS},
\]
so
\[
\KL(p^{\scafS}\|q^{\scafS})=+\infty.
\]
Because \(s>0\), the right-hand side of the asserted chain rule is
\(+\infty\). If \(v\in\contC\), the corresponding content divergence
is \(+\infty\), and its coefficient \(1-s\) is positive. Hence both
sides equal \(+\infty\), proving \textup{(ii)} in every case.
By \textup{(i)},
\begin{align*}
G(r)
&=
H(r)-\HC(r)\\
&=
\Hb(s_r)
+s_r\HS(r)
+(1-s_r)\HC(r)-\HC(r)\\
&=
\Hb(s_r)
+s_r\bigl(\HS(r)-\HC(r)\bigr).
\end{align*}
We now establish the entropy extrema. Let \(A\) be a nonempty finite
set with \(|A|=\ell\), and let
\(a\in\overline{\Delta}(A)\). Since
\[
-a_x\log a_x\geq0,
\]
one has \(H(a)\geq0\). Equality holds if and only if \(a\) is a point
mass.
By concavity of the logarithm,
\begin{align*}
H(a)
&=
\sum_{x:a_x>0}
a_x\log\frac{1}{a_x}\\
&\leq
\log\left(
\sum_{x:a_x>0}
a_x\frac{1}{a_x}
\right)\\
&=
\log|\operatorname{supp}(a)|
\leq
\log\ell.
\end{align*}
Equality in the first inequality holds if and only if \(a\) is
uniform on its support. Therefore \(H(a)=\log\ell\) if and only if
\(\operatorname{supp}(a)=A\) and \(a\) is uniform on \(A\).
Fix \(\sigma\in(0,1)\) and
\(c\in\overline{\Delta}(\contC)\). For
\(r\in\mathcal P_{\sigma,c}\),
\[
G(r)
=
\Hb(\sigma)
+\sigma\HS(r)
-\sigma H(c).
\]
Since
\[
0\leq\HS(r)\leq\log m,
\]
the stated minimum and maximum follow. Their equality
characterizations follow from the entropy equality cases above.
For \(r\in\mathcal P_{\sigma}\),
\[
0\leq\HS(r)\leq\log m,
\qquad
0\leq\HC(r)\leq\log k.
\]
Hence
\[
G(r)-\bigl(\Hb(\sigma)-\sigma\log k\bigr)
=
\sigma\HS(r)
+\sigma\bigl(\log k-\HC(r)\bigr)
\geq0,
\]
and
\[
\Hb(\sigma)+\sigma\log m-G(r)
=
\sigma\bigl(\log m-\HS(r)\bigr)
+\sigma\HC(r)
\geq0.
\]
The lower equality holds exactly when
\[
\HS(r)=0,
\qquad
\HC(r)=\log k,
\]
and the upper equality holds exactly when
\[
\HS(r)=\log m,
\qquad
\HC(r)=0.
\]
These are precisely the equality conditions stated in
\textup{(iii)}.
It remains to prove attainment. For arbitrary
\[
a\in\overline{\Delta}(\scafS),
\qquad
c\in\overline{\Delta}(\contC),
\]
define
\[
r_v
=
\begin{cases}
\sigma a_v,&v\in\scafS,\\
(1-\sigma)c_v,&v\in\contC.
\end{cases}
\]
Then
\[
r\in\Delta_{\mathrm{gi}}(V),
\qquad
s_r=\sigma,
\qquad
r^{\scafS}=a,
\qquad
r^{\contC}=c.
\]
Taking \(a\) to be a point mass or uniform proves attainment of the
fixed-content extrema. Taking
\[
(a,c)
=
(\text{point mass},\text{uniform})
\]
or
\[
(a,c)
=
(\text{uniform},\text{point mass})
\]
proves attainment of the global extrema. If a block is a singleton,
its unique probability law is simultaneously a point mass and
uniform. This proves all sharpness and equality claims.
For \textup{(iv)}, apply \textup{(i)} to each \(p_j\), multiply by
\(\alpha_j\), and sum over \(j\). This proves the deterministic
identity.
For the random identity, note first that \(H\), \(\Hb(s_\cdot)\),
\(s_\cdot\HS\), and \((1-s_\cdot)\HC\) are continuous on
\(\Delta_{\mathrm{gi}}(V)\), so every summand is measurable, and
\(\sum_{j\leq N}=\sum_{j\geq1}\mathbf1\{N\geq j\}\cdot(\cdot)\)
preserves measurability. The same equality then holds pathwise.
Furthermore,
\[
0\leq H(P_j)\leq\log(m+k),
\qquad
0\leq\Hb(S_j)\leq\log2,
\]
\[
0\leq S_j\HS(P_j)\leq\log m,
\qquad
0\leq(1-S_j)\HC(P_j)\leq\log k.
\]
Thus
\[
\mathbb E\left[
\sum_{j=1}^{N}|W_j|
\right]<\infty
\]
implies absolute integrability of all four stopped weighted sums,
each dominated by \(\log(m+k)\sum_{j\leq N}|W_j|\). Taking
expectations of the pathwise identity is therefore justified by
linearity of expectation for integrable summands.
\end{proof}

\begin{remark}[Boundary and open-simplex cases]
The entropy identity extends to a gate mass of \(0\) or \(1\) by
assigning weight zero to the entropy of a zero-mass block.
At gate mass \(0\), the content law is the restriction of \(r\) to
\(\contC\), and
\[
G(r)=0.
\]
The numerical surplus bounds then collapse to equality for every
\(r\), so their gate-interior equality characterizations do not
extend.
At gate mass \(1\), the ratio defining \(r^{\contC}\) is undefined.
If \(k\geq2\), no boundary value of \(G\) at such \(r\) is
consistent with continuity. Indeed, choose
\(a\in\overline{\Delta}(\scafS)\) and content laws
\(c_1,c_2\in\overline{\Delta}(\contC)\) with
\(H(c_1)\neq H(c_2)\), and define
\[
r_{\varepsilon,v}^{(i)}
=
\begin{cases}
(1-\varepsilon)a_v,&v\in\scafS,\\
\varepsilon(c_i)_v,&v\in\contC.
\end{cases}
\]
Then \(r_\varepsilon^{(1)}\) and \(r_\varepsilon^{(2)}\) converge to
the same boundary distribution, while
\[
G(r_\varepsilon^{(i)})
\longrightarrow
H(a)-H(c_i).
\]
Thus \(G\) has no path-independent continuous extension in general.
If \(k=1\), the unique content law has entropy zero, so
\[
G(r):=H(r)
\]
gives a canonical boundary extension.
Boundary values in the KL identity require a separate formulation
with explicit conventions for zero-mass conditional laws. They are
not included in the theorem.
If \(p\) and \(q\) arise from ordinary softmax maps with finite logits
on both nonempty blocks, then \(p,q\in\Delta_{\mathrm{gi}}(V)\).
For a nonempty finite set \(A\), define its relative open simplex by
\[
\Delta^\circ(A)
:=
\left\{
a\in\overline{\Delta}(A):
a_x>0\ \text{for every }x\in A
\right\}.
\]
If the manuscript restricts \(r\) to \(\Delta^\circ(V)\), all chain
rule identities remain valid. In the fixed-content statement one
must additionally require
\[
c\in\Delta^\circ(\contC);
\]
otherwise the constrained class is empty. For admissible \(c\), any
extremum requiring a point mass on a block of cardinality greater
than one is not attained in \(\Delta^\circ(V)\), but the same constant
remains the sharp infimum or supremum and is approached by strictly
positive distributions. On a singleton block, the unique conditional
law is both uniform and a point mass.
\end{remark}
\paragraph{Standing notation.}
For every nonempty finite set \(A\), define
\[
\overline{\Delta}_A
:=
\left\{
x\in[0,\infty)^A:
\sum_{a\in A}x_a=1
\right\},
\qquad
\Delta_A^\circ
:=
\left\{
x\in(0,\infty)^A:
\sum_{a\in A}x_a=1
\right\}.
\]
Thus, if \(|A|=1\), then \(\Delta_A^\circ=\{(1)\}\).
For \(x\in\overline{\Delta}_A\), define
\[
H(x):=-\sum_{a:x_a>0}x_a\log x_a.
\]
Equivalently, \(0\log0:=0\).
For \(p\in\Delta_V^\circ\) and nonempty \(B\subseteq V\), write
\[
p(B):=\sum_{v\in B}p_v,
\qquad
p^B
:=
\left(
\frac{p_v}{p(B)}
\right)_{v\in B}
\in\Delta_B^\circ.
\]
For \(z\in\mathbb R^V\), define
\[
P_v(z)
:=
\frac{e^{z_v}}{\sum_{u\in V}e^{z_u}},
\qquad
P^B(z):=(P(z))^B
=
\left(
\frac{e^{z_v}}{\sum_{u\in B}e^{z_u}}
\right)_{v\in B}.
\]
For \(x\in\Delta_A^\circ\), define
\[
\operatorname{clr}_A(x)_a
:=
\log x_a
-\frac1{|A|}\sum_{b\in A}\log x_b,
\]
and, for \(x,y\in\Delta_A^\circ\),
\[
d_{\mathrm A}(x,y)
:=
\left\|
\operatorname{clr}_A(x)-\operatorname{clr}_A(y)
\right\|_2.
\]
When \(|A|=1\), both centered log-ratio vectors and hence
\(d_{\mathrm A}\) are identically zero. For \(\alpha>0\), define
Aitchison powering by
\[
(\alpha\odot x)_a
:=
\frac{x_a^\alpha}{\sum_{b\in A}x_b^\alpha}.
\]

\subsection{Lemma~\ref{lem:pyth}: three-way Pythagorean split}

\begin{proof}
\emph{(classical construction \citep{egozcue2003ilr,egozcue2005balances};
we verify it in coordinates).} Define $u_b\in\mathbb{R}^{D}$ by
$(u_b)_i=\sqrt{k/(mD)}$ for $i\in\scafS$ and
$(u_b)_i=-\sqrt{m/(kD)}$ for $i\in\contC$. Then
$\lVert u_b\rVert^2=m\cdot\frac{k}{mD}+k\cdot\frac{m}{kD}=1$ and
$u_b^{\top}\mathbf{1}=m\sqrt{k/(mD)}-k\sqrt{m/(kD)}=0$, so $u_b$ is a
unit vector of the CLR hyperplane. Writing
$\bar\ell_{\scafS},\bar\ell_{\contC},\bar\ell$ for the means of
$\log p$ over $\scafS$, $\contC$, and $V$,
\[
u_b^{\top}\clr(p)
=\sqrt{\tfrac{k}{mD}}\,m(\bar\ell_{\scafS}-\bar\ell)
-\sqrt{\tfrac{m}{kD}}\,k(\bar\ell_{\contC}-\bar\ell)
=\sqrt{\tfrac{mk}{D}}\big(\bar\ell_{\scafS}-\bar\ell_{\contC}\big)=b(p).
\]
Complete $u_b$ with any orthonormal basis $\{v_j\}_{j<m}$ of the
sum-zero vectors supported on $\scafS$ and $\{w_j\}_{j<k}$ supported on
$\contC$; the three groups are mutually orthogonal and count
$1+(m-1)+(k-1)=D-1$ vectors, an orthonormal basis of the hyperplane.
For $v_j$ supported on $\scafS$ with $v_j^{\top}\mathbf{1}=0$,
\[
v_j^{\top}\clr(p)=v_j^{\top}\log p\big|_{\scafS}
=v_j^{\top}\big(\log \pS+\log s\,\mathbf{1}\big)\big|_{\scafS}
=v_j^{\top}\clr_{\scafS}(\pS),
\]
i.e.\ the within-$\scafS$ coordinates of $p$ \emph{are} ILR coordinates
of the subcomposition $\pS$, and likewise for $\contC$. The map
$p\mapsto(b,\ilr(\pS),\ilr(\pC))$ is therefore an isometric linear
change of ILR basis, and $\dA^2$, the Euclidean norm in any ILR basis,
splits as claimed.
\end{proof}

\begin{remark}[Amalgamation versus balances]
\label{rem:amalg}
The entropy identity is written in the \emph{amalgamated mass} $s$
(an arithmetic sum of parts), the geometric identity in the
\emph{balance} $b$ (a log-ratio of geometric means). Amalgamation is not
a linear operation of Aitchison geometry, a known tension in
compositional data analysis \citep{pawlowsky2015,greenacre2022reappraisal};
the two coordinate systems answer ``how much probability sits on the
scaffold'' and ``where the scaffold sits in log-ratio terms''
respectively. Both are exact; we never mix them within one identity,
and the companion's singleton-sink case (where $m=1$, $\HS\equiv0$, and
$b$ is the sink balance) is the unique setting where the two stories
collapse into one \citep{anonymous2026rows}.
\end{remark}

\subsection{Proposition~\ref{prop:invariance}: content maximal
  invariance and temperature behavior}

\begin{proof}
The transformations in \(G_{\scafS}\) form a group because
\[
g_{w,c}\circ g_{\widetilde w,\widetilde c}
=
g_{w+\widetilde w,c+\widetilde c},
\]
and \(w+\widetilde w\) still vanishes on \(\contC\). The identity is
\(g_{0,0}\), and
\[
g_{w,c}^{-1}=g_{-w,-c}.
\]
\emph{Proof of \textup{(i)}.}
For \(v\in\contC\), one has \(w_v=0\). Therefore
\[
\begin{aligned}
P_v^{\contC}(g_{w,c}z)
&=
\frac{e^{z_v+c}}
     {\sum_{u\in\contC}e^{z_u+c}}\\
&=
\frac{e^{z_v}}
     {\sum_{u\in\contC}e^{z_u}}
=
P_v^{\contC}(z).
\end{aligned}
\]
Thus \(P^{\contC}\) is invariant.
Suppose that
\[
P^{\contC}(z)=P^{\contC}(z').
\]
Let
\[
Z_{\contC}(z):=\sum_{u\in\contC}e^{z_u}.
\]
For every \(v\in\contC\),
\[
\frac{e^{z'_v}}{Z_{\contC}(z')}
=
\frac{e^{z_v}}{Z_{\contC}(z)}.
\]
Taking logarithms gives
\[
z'_v-z_v
=
\log Z_{\contC}(z')
-
\log Z_{\contC}(z)
=:c,
\qquad v\in\contC.
\]
Define
\[
w:=z'-z-c\mathbf1_V.
\]
Then \(w|_{\contC}=0\), and
\[
z'=z+w+c\mathbf1_V=g_{w,c}(z).
\]
The converse follows from the invariance already proved. Hence the
fibers of \(z\mapsto P^{\contC}(z)\) are exactly the
\(G_{\scafS}\)-orbits.
The logit and simplex actions agree because
\[
\begin{aligned}
P_v(g_{w,c}z)
&=
\frac{e^{z_v+w_v+c}}
     {\sum_{u\in V}e^{z_u+w_u+c}}\\
&=
\frac{e^{w_v}P_v(z)}
     {\sum_{u\in V}e^{w_u}P_u(z)}
=
\bigl(\Phi_w(P(z))\bigr)_v.
\end{aligned}
\]
Thus the induced action is independent of \(c\).
Let \(p,q\in\Delta_V^\circ\), and set
\[
z:=\log p,
\qquad
z':=\log q
\]
coordinatewise. Positivity gives
\[
P(z)=p,
\qquad
P(z')=q.
\]
If \(p^{\contC}=q^{\contC}\), then
\(P^{\contC}(z)=P^{\contC}(z')\). By maximality on logit space,
\(z'=g_{w,c}(z)\) for some \(w|_{\contC}=0\), and hence
\[
q=P(z')=\Phi_w(p).
\]
The converse follows from content-subcomposition invariance.
Suppose that \(F\) satisfies \textup{(a)}. Then \(F\) is constant on
every fiber of \(p\mapsto p^{\contC}\). Fix
\[
\rho\in(0,1)
\qquad\text{and}\qquad
r\in\Delta_{\scafS}^\circ.
\]
For \(q\in\Delta_{\contC}^\circ\), define
\[
\iota(q)_v
:=
\begin{cases}
\rho r_v,&v\in\scafS,\\[1mm]
(1-\rho)q_v,&v\in\contC.
\end{cases}
\]
Then \(\iota(q)\in\Delta_V^\circ\) and
\[
\iota(q)^{\contC}=q.
\]
Define
\[
\widetilde F(q):=F(\iota(q)).
\]
For every \(p\in\Delta_V^\circ\), the vectors
\(p\) and \(\iota(p^{\contC})\) have the same content
subcomposition and hence lie in the same orbit. Therefore
\[
F(p)
=
F(\iota(p^{\contC}))
=
\widetilde F(p^{\contC}).
\]
This proves existence. Since \(p\mapsto p^{\contC}\) is surjective,
the factor \(\widetilde F\) is unique. Conversely, every statistic
of the form
\[
F(p)=\widetilde F(p^{\contC})
\]
is invariant because \(p^{\contC}\) is invariant.
\emph{Proof of \textup{(ii)}.}
Let \(e_a\) denote the coordinate unit vector corresponding to
\(a\in V\). Choose \(a\in\scafS\), let
\[
z=0,
\qquad
z'=(\log2)e_a,
\]
and set
\[
p:=P(z),
\qquad
q:=P(z').
\]
Then
\[
p_i=\frac1D
\quad(i\in V),
\]
whereas
\[
q_a=\frac2{D+1},
\qquad
q_i=\frac1{D+1}
\quad(i\neq a).
\]
The two distributions belong to the same orbit and have the same
content subcomposition. Nevertheless,
\[
H(p)=\log D,
\qquad
H(q)
=
\log(D+1)-\frac2{D+1}\log2
<\log D.
\]
The strict inequality also follows from uniqueness of the uniform
entropy maximizer. Thus raw entropy is not invariant.
Moreover,
\[
\operatorname{VarEnt}(p)=0.
\]
Under \(q\), the random variable \(-\log q_i\) takes two values whose
difference is \(\log2\), with aggregate probabilities
\[
\frac2{D+1}
\qquad\text{and}\qquad
\frac{D-1}{D+1}.
\]
Hence
\[
\operatorname{VarEnt}(q)
=
\frac{2(D-1)}{(D+1)^2}(\log2)^2
>0.
\]
Thus raw varentropy is not invariant.
For the margin, choose \(b\in\contC\) and logits \(z\) satisfying
\[
e^{z_b}=4,
\qquad
e^{z_a}=2,
\qquad
e^{z_i}=1
\quad(i\notin\{a,b\}).
\]
Set
\[
z':=z+(\log6)e_a.
\]
Before the transformation, the two largest unnormalized weights are
\(4\) and \(2\); afterward they are \(12\) and \(4\). Therefore
\[
\operatorname{mar}(P(z))
=
\frac2{D+4},
\qquad
\operatorname{mar}(P(z'))
=
\frac8{D+14},
\]
and
\[
\frac8{D+14}-\frac2{D+4}
=
\frac{6D+4}{(D+14)(D+4)}
>0.
\]
Thus the raw top-two probability margin is not invariant.
Fix \(k\in\{1,\ldots,D-1\}\). Choose \(b\in\contC\) and
\(R\subseteq V\setminus\{a,b\}\) with \(|R|=k-1\). Define
\[
z_i=
\begin{cases}
2,&i\in R,\\
1,&i=b,\\
0,&i=a,\\
-1,&\text{otherwise},
\end{cases}
\qquad
z':=z+3e_a.
\]
There is no tie at the \(k\)-th selection boundary, and softmax
preserves coordinate ordering. Hence
\[
\operatorname{Top}_k(P(z))
=
R\cup\{b\},
\qquad
\operatorname{Top}_k(P(z'))
=
R\cup\{a\}.
\]
Thus every nontrivial raw top-\(k\) map is non-invariant.
For \(\lambda\ge0\), define
\[
z(\lambda):=\lambda e_a,
\qquad
p_\lambda:=P(z(\lambda)).
\]
All \(z(\lambda)\) lie in the same \(G_{\scafS}\)-orbit, and
\[
H(p_\lambda)
=
\log(e^\lambda+D-1)
-
\lambda\frac{e^\lambda}{e^\lambda+D-1}.
\]
For \(\lambda>0\),
\[
\frac{d}{d\lambda}H(p_\lambda)
=
-\lambda
\frac{e^\lambda(D-1)}
     {(e^\lambda+D-1)^2}
<0.
\]
Furthermore,
\[
H(p_0)=\log D,
\qquad
\lim_{\lambda\to\infty}H(p_\lambda)=0.
\]
Therefore, for every \(h\in(0,\log D)\), there exists a finite
\(\lambda\) such that
\[
H(p_\lambda)<h<H(p_0).
\]
Thus
\[
E_h(p_\lambda)\neq E_h(p_0).
\]
All content-only invariances follow from part~\textup{(i)}.
\emph{Proof of \textup{(iii)}.}
For \(i\in B\),
\[
\begin{aligned}
P_i^B(z/\tau)
&=
\frac{e^{z_i/\tau}}
     {\sum_{j\in B}e^{z_j/\tau}}\\
&=
\frac{(P_i^B(z))^{1/\tau}}
     {\sum_{j\in B}(P_j^B(z))^{1/\tau}}.
\end{aligned}
\]
Thus
\[
P^B(z/\tau)=(1/\tau)\odot P^B(z).
\]
For every \(\alpha>0\),
\[
\operatorname{clr}_B(\alpha\odot x)
=
\alpha\,\operatorname{clr}_B(x).
\]
Consequently,
\[
\begin{aligned}
d_{\mathrm A}\!\left(
P^{\contC}(z/\tau),P^{\contC}(z'/\tau)
\right)
&=
\left\|
\frac1\tau\operatorname{clr}_{\contC}(P^{\contC}(z))
-
\frac1\tau\operatorname{clr}_{\contC}(P^{\contC}(z'))
\right\|_2\\
&=
\frac1\tau
d_{\mathrm A}\!\left(
P^{\contC}(z),P^{\contC}(z')
\right).
\end{aligned}
\]
Because \(1/\tau>0\) is common to all pairs, every distance ranking
and every distance tie is preserved.
For \(i,j\in\contC\),
\[
\frac{P_i^{\contC}(z/\tau)}
     {P_j^{\contC}(z/\tau)}
=
\exp\!\left(\frac{z_i-z_j}{\tau}\right).
\]
Its comparison with \(1\) is independent of \(\tau>0\), proving
preservation of within-content coordinate order and ties.
We now prove the negative rank-preservation statements. First take
\[
D=3,
\qquad
\scafS=\{1\},
\qquad
\contC=\{2,3\}.
\]
For \(x\in\mathbb R^n\), define
\[
H_\tau(x)
:=
H\!\left(
\frac{e^{x/\tau}}{\sum_{i=1}^n e^{x_i/\tau}}
\right).
\]
Let
\[
\psi_x(\beta):=\log\sum_{i=1}^n e^{\beta x_i},
\qquad
\beta:=\frac1\tau.
\]
Then
\[
H_\tau(x)
=
\psi_x(\beta)-\beta\psi_x'(\beta).
\]
If \(U_n\) is uniform on the coordinates of \(x\), Taylor expansion
at \(\beta=0\) gives
\[
\psi_x(\beta)
=
\log n
+
\beta\,\mathbb E_{U_n}[x]
+
\frac{\beta^2}{2}\operatorname{Var}_{U_n}(x)
+
O(\beta^3),
\]
and
\[
\psi_x'(\beta)
=
\mathbb E_{U_n}[x]
+
\beta\,\operatorname{Var}_{U_n}(x)
+
O(\beta^2).
\]
It follows that
\[
H_\tau(x)
=
\log n
-
\frac{\operatorname{Var}_{U_n}(x)}{2\tau^2}
+
O(\tau^{-3}),
\qquad
\tau\to\infty.
\tag{1}
\]
For raw entropy, take
\[
x=(0,0,1),
\qquad
x'=(0,2,2).
\]
As \(\tau\downarrow0\), softmax converges to the uniform distribution
on the maximizing coordinates, so
\[
H_\tau(x)\longrightarrow0,
\qquad
H_\tau(x')\longrightarrow\log2.
\]
Moreover,
\[
\operatorname{Var}_{U_3}(x)=\frac29,
\qquad
\operatorname{Var}_{U_3}(x')=\frac89.
\]
Thus (1) gives
\[
H_\tau(x)-H_\tau(x')
=
\frac1{3\tau^2}
+
O(\tau^{-3})
>0
\]
for all sufficiently large \(\tau\). Hence there exist finite
\(0<\tau_{\mathsf H,-}<\tau_{\mathsf H,+}\) at which the entropy
ordering is reversed.
For the gate mass, take
\[
y=(1,0,0),
\qquad
y'=(3,0,3).
\]
As \(\tau\downarrow0\),
\[
s(y/\tau)\longrightarrow1,
\qquad
s(y'/\tau)\longrightarrow\frac12.
\]
Writing \(\beta=1/\tau\),
\[
s(y/\tau)
=
\frac{e^\beta}{e^\beta+2}
=
\frac13+\frac{2\beta}{9}+O(\beta^2),
\]
while
\[
s(y'/\tau)
=
\frac{e^{3\beta}}{2e^{3\beta}+1}
=
\frac13+\frac{3\beta}{9}+O(\beta^2).
\]
Therefore
\[
s(y/\tau)-s(y'/\tau)
=
-\frac1{9\tau}
+
O(\tau^{-2})
<0
\]
for all sufficiently large \(\tau\). Hence the gate-mass ordering
reverses between two finite temperatures.
For the surplus, take
\[
u=(1,0,0),
\qquad
u'=(3,0,1).
\]
As \(\tau\downarrow0\),
\[
G(u/\tau)\longrightarrow-\log2,
\qquad
G(u'/\tau)\longrightarrow0.
\]
Applying (1) to the full three-coordinate distribution and its
two-coordinate content subcomposition gives
\[
G(x/\tau)
=
\log\frac32
-
\frac{
\operatorname{Var}_{U_3}(x)
-
\operatorname{Var}_{U_2}(x|_{\contC})
}{2\tau^2}
+
O(\tau^{-3}).
\tag{2}
\]
For \(u\),
\[
\operatorname{Var}_{U_3}(u)
-
\operatorname{Var}_{U_2}(u|_{\contC})
=
\frac8{36},
\]
whereas, for \(u'\),
\[
\operatorname{Var}_{U_3}(u')
-
\operatorname{Var}_{U_2}(u'|_{\contC})
=
\frac{47}{36}.
\]
Consequently,
\[
G(u/\tau)-G(u'/\tau)
=
\frac{13}{24\tau^2}
+
O(\tau^{-3})
>0
\]
for all sufficiently large \(\tau\). Thus the surplus ordering
reverses between two finite temperatures.
It remains to extend these witnesses to every admissible partition
with \(D\ge3\).
If \(|\contC|\ge2\), choose
\[
a\in\scafS,
\qquad
b_1,b_2\in\contC,
\qquad
b_1\neq b_2,
\]
and assign each displayed three-coordinate witness to
\((a,b_1,b_2)\).
If \(|\contC|=1\), then \(|\scafS|\ge2\). Write
\[
\contC=\{c\},
\]
and choose distinct \(a_1,a_2\in\scafS\). Assign the entropy witness
to \((c,a_1,a_2)\). For the gate witness, assign \(y,y'\) to
\((c,a_1,a_2)\). The actual scaffold mass is then one minus the
first-coordinate probability, so both comparison signs are reversed
and a strict ranking reversal remains. Finally,
\(P^{\contC}(z)\) is the one-point distribution, so
\[
H(P^{\contC}(z))=0,
\qquad
G(z)=\mathsf H(z).
\]
The entropy witness therefore also supplies the surplus reversal.
For \(D>3\), first choose finite temperatures at which the relevant
three-coordinate inequalities are strict. Set every unused
coordinate in both logit vectors equal to \(-L\). At either fixed
temperature, as \(L\to\infty\), the resulting full softmax
distributions and their block subcompositions converge to the
corresponding embedded sub-block distributions. Shannon entropy, gate mass,
and content entropy are continuous under this convergence. Therefore
a sufficiently large finite \(L\), one value serving both fixed
temperatures, preserves all strict inequalities.
All final logits and temperatures are finite.
This proves all three parts.
\end{proof}

\begin{remark}[Necessity of the interior domain]
The maximal-invariance statement does not extend verbatim to the
closed simplex under finite logit reweightings. Such reweightings
preserve zero coordinates, and the content subcomposition is
undefined when the content block has zero total mass. A boundary
version must therefore be stratified by support. The dimension bound
in part (iii) is also necessary: for \(D=2\) both blocks are
singletons, \(G=\mathsf H\), and \(\mathsf H\) and \(s\) are
strictly monotone in \((z_1-z_2)/\tau\), so no common-temperature
reversal exists.
\end{remark}

\subsection{Proposition~\ref{prop:dynamics}: channel dynamics}

\begin{proof}
Fix \(t\in I\) and suppress the dependence on \(t\) until the
natural-gradient specialization. Since all logits are finite,
\[
p_a>0\qquad(a\in V).
\]
Because both blocks are nonempty,
\[
0<s<1,
\qquad
m_{\scafS}>0,
\qquad
m_{\contC}>0.
\]
Thus all conditional distributions and logarithms appearing below are
well defined.
Put
\[
\bar u:=\mathbb E_{a\sim p}[u_a].
\]
Differentiating the softmax gives
\begin{align*}
\dot p_a
&=
\frac{e^{z_a}u_a}{\sum_v e^{z_v}}
-
\frac{e^{z_a}\sum_v e^{z_v}u_v}
     {(\sum_v e^{z_v})^2}\\
&=
p_a\left(u_a-\sum_vp_vu_v\right)\\
&=
p_a(u_a-\bar u).
\end{align*}
Consequently,
\[
\sum_{a\in V}\dot p_a
=
\sum_a p_a u_a-\bar u\sum_a p_a
=
0.
\]
Therefore
\begin{align*}
\frac{d}{dt}H(p)
&=
-\sum_{a\in V}(1+\log p_a)\dot p_a\\
&=
-\sum_{a\in V}\log p_a\,\dot p_a\\
&=
-\sum_{a\in V}
p_a(u_a-\bar u)\log p_a\\
&=
-\operatorname{Cov}_{a\sim p}(\log p_a,u_a).
\end{align*}
For \(B\in\{\scafS,\contC\}\),
\begin{align*}
\dot m_B
&=
\sum_{a\in B}\dot p_a\\
&=
\sum_{a\in B}p_a(u_a-\bar u)\\
&=
m_B(\bar u_B-\bar u).
\end{align*}
Furthermore,
\[
\bar u
=
s\bar u_{\scafS}
+(1-s)\bar u_{\contC}.
\]
Taking \(B=\scafS\), we obtain
\begin{align*}
\dot s
&=
s(\bar u_{\scafS}-\bar u)\\
&=
s(1-s)
\bigl(\bar u_{\scafS}-\bar u_{\contC}\bigr).
\end{align*}
For \(a\in B\), the quotient rule gives
\begin{align*}
\dot p_a^B
&=
\frac{\dot p_a}{m_B}
-\frac{p_a\dot m_B}{m_B^2}\\
&=
p_a^B(u_a-\bar u)
-p_a^B(\bar u_B-\bar u)\\
&=
p_a^B(u_a-\bar u_B).
\end{align*}
It follows that
\[
\sum_{a\in B}\dot p_a^B=0
\]
and hence
\begin{align*}
\dot H^B
&=
-\sum_{a\in B}(1+\log p_a^B)\dot p_a^B\\
&=
-\sum_{a\in B}\log p_a^B\,\dot p_a^B\\
&=
-\sum_{a\in B}
p_a^B(u_a-\bar u_B)\log p_a^B\\
&=
-\operatorname{Cov}_{a\sim p^B}
\bigl(\log p_a^B,u_a\bigr).
\end{align*}
For \(a\in B\), \(p_a=m_Bp_a^B\). Consequently,
\begin{align*}
-\sum_{a\in B}p_a\log p_a
&=
-m_B\sum_{a\in B}
p_a^B\bigl(\log m_B+\log p_a^B\bigr)\\
&=
-m_B\log m_B+m_BH^B.
\end{align*}
Summing over the two blocks yields
\[
H(p)
=
\Hb(s)+sH^{\scafS}+(1-s)H^{\contC}.
\]
Since
\[
\Hb'(s)=\log\frac{1-s}{s},
\]
differentiating the grouping identity gives
\begin{align*}
\dot H(p)
&=
\left[
\log\frac{1-s}{s}
+H^{\scafS}-H^{\contC}
\right]\dot s\\
&\qquad
+s\dot H^{\scafS}
+(1-s)\dot H^{\contC}\\
&=
\Gamma_{\mathrm{gate}}
+\Gamma_{\scafS}
+\Gamma_{\contC}.
\end{align*}
This also proves
\[
\Gamma_{\mathrm{gate}}
=
K(p)\dot s
=
\dot\Hb(s)
+\bigl(H^{\scafS}-H^{\contC}\bigr)\dot s.
\]
If \(u_a\) is constant on \(B\), then
\(u_a-\bar u_B=0\) for every \(a\in B\), so
\[
\dot p_a^B=0
\qquad\text{and}\qquad
\Gamma_B=0.
\]
We now prove the natural-gradient specialization. Restore the time
arguments. For a categorical softmax distribution,
\[
\nabla_z\log p_a(t)=e_a-p(t).
\]
Thus
\begin{align*}
g(t)
&=
\mathbb E_{a\sim p(t)}
\bigl[A_a(t)(e_a-p(t))\bigr]\\
&=
\operatorname{diag}(p(t))A(t)
-p(t)\bigl(p(t)^\top A(t)\bigr)\\
&=
F(t)A(t).
\end{align*}
For every \(x\in\mathbb R^V\),
\begin{align*}
x^\top F(t)x
&=
\sum_{a\in V}p_a(t)x_a^2
-
\left(\sum_{a\in V}p_a(t)x_a\right)^2\\
&=
\operatorname{Var}_{a\sim p(t)}(x_a).
\end{align*}
Because \(p_a(t)>0\) for every \(a\), this variance vanishes if and
only if \(x\) is constant. Therefore
\[
\ker F(t)=\operatorname{span}\{\mathbf1\}.
\]
Since \(F(t)\) is symmetric,
\[
\operatorname{range}F(t)
=
\ker F(t)^\perp
=
\mathbf1^\perp.
\]
Since \(F(t)^\dagger F(t)\) is the orthogonal projector onto
\(\operatorname{range}(F(t)^\top)=\mathbf1^\perp\),
\[
F(t)^\dagger F(t)
=
I_{|V|}
-\frac{1}{|V|}\mathbf1\mathbf1^\top.
\]
Using the assumed natural-gradient flow,
\begin{align*}
\dot z(t)
&=
u^{\mathrm{NG}}(t)\\
&=
\eta F(t)^\dagger g(t)\\
&=
\eta F(t)^\dagger F(t)A(t)\\
&=
\eta
\left(
A(t)-\langle A(t)\rangle_{\mathrm{unif}}\mathbf1
\right).
\end{align*}
Thus the Moore--Penrose velocity is of the form
\[
u^{\mathrm{NG}}(t)=\eta A(t)+c(t)\mathbf1,
\qquad
c(t)=-\eta\langle A(t)\rangle_{\mathrm{unif}}.
\]
Since the common additive term cancels from the softmax derivative,
\begin{align*}
\dot p_a(t)
&=
p_a(t)
\left(
\dot z_a(t)
-\mathbb E_{v\sim p(t)}[\dot z_v(t)]
\right)\\
&=
\eta p_a(t)
\bigl(A_a(t)-\bar A(t)\bigr).
\end{align*}
Applying the already proved block formulas gives
\[
\dot p_a^B(t)
=
\eta p_a^B(t)
\bigl(A_a(t)-\bar A_B(t)\bigr)
\]
and
\[
\dot s(t)
=
\eta s(t)(1-s(t))
\bigl(\bar A_{\scafS}(t)-\bar A_{\contC}(t)\bigr).
\]
Because covariance with a constant is zero,
\[
\Gamma_B(t)
=
-\eta m_B(t)
\operatorname{Cov}_{a\sim p^B(t)}
\bigl(\log p_a^B(t),A_a(t)\bigr),
\]
while substitution into the general gate formula yields
\[
\Gamma_{\mathrm{gate}}(t)
=
\eta K(p(t))s(t)(1-s(t))
\bigl(\bar A_{\scafS}(t)-\bar A_{\contC}(t)\bigr).
\]
If \(A_a(t)=\alpha_B(t)\) for every \(a\in B\), then
\(\bar A_B(t)=\alpha_B(t)\), and hence
\[
\dot p_a^B(t)=0
\qquad(a\in B).
\]
If this holds for every \(t\in J\), then
\(\dot p_a^B(t)=0\) throughout \(J\). Since \(J\) is an interval,
\(p^B(t)\) is constant on \(J\).
Finally, \(\eta>0\) and \(s(t)(1-s(t))>0\). The displayed formulas for
\(\dot s(t)\) and \(\Gamma_{\mathrm{gate}}(t)\) therefore imply the
stated equivalences and sign identity.
\end{proof}

\begin{proof}[Proof of Corollary~\ref{cor:reinforce}]
For the sampled update $\Delta z_a=\eta(\mathbf{1}[a=a']-p_a)\hat A(a')$
with $a'\sim p$, taking expectations gives
$\mathbb{E}[\Delta z_a]=\eta\,p_a(\hat A_a-\bar{\hat A})$. A key
point, which an earlier version of this corollary got wrong and a
careful reader will check: this velocity is \emph{not} blockwise
constant even for block-constant $\hat A$, because of the factor $p_a$.
Take $\hat A_a=A_{\scafS}$ on $\scafS$ and $A_{\contC}$ on $\contC$,
$\Delta A=A_{\scafS}-A_{\contC}$, so
$\bar{\hat A}=sA_{\scafS}+(1-s)A_{\contC}$ and
\[
u_a=\eta\,p_a(1-s)\Delta A=\eta\,s(1-s)\Delta A\,\pS_a\ \ (a\in\scafS),
\qquad
u_a=-\eta\,p_a\,s\,\Delta A=-\eta\,s(1-s)\Delta A\,\pC_a\ \ (a\in\contC).
\]
Within each block $u$ is proportional to $p^{B}$, so
$\operatorname{Cov}_{p^{B}}(\log p^{B},u)=c_B
\operatorname{Cov}_{p^{B}}(\log p^{B},p^{B})=c_B V_B$ with
$c_{\scafS}=\eta s(1-s)\Delta A$, $c_{\contC}=-\eta s(1-s)\Delta A$,
and plugging into $\Gamma_B=-m_B\operatorname{Cov}_{p^{B}}(\log p^{B},u)$
gives the displayed formulas. $V_B\ge0$ because $x\mapsto\log x$ and
$x\mapsto x$ are comonotone functions of $p^{B}_a$ (Chebyshev's sum
inequality), with equality iff $p^{B}$ is uniform. A numerical
counterexample to the gate-only reading is
$p=(0.4,0.1\mid0.25,0.25)$, $\hat A=(1,1\mid-1,-1)$ gives
$u=(0.4,0.1\mid-0.25,-0.25)$ and $\Gamma_{\scafS}=-0.03327\ne0$. Under
natural gradient $u=\eta\hat A$ the velocity IS blockwise constant for
block-level rewards, and only the between-block term moves
(Proposition~\ref{prop:dynamics}).
\end{proof}

\begin{remark}[Scope]
The natural-gradient specialization concerns exact undamped ascent in
the local categorical Fisher geometry at a fixed prefix. It assumes
that the actual logit path satisfies the displayed natural-gradient
ODE. It does not assert that a shared-parameter neural-network update,
a damped or approximate natural gradient, clipping, Adam, or a
REINFORCE/GRPO update induces the same local logit velocity.
For the standard exact advantage
\[
A_a(t)
=
Q_a(t)-\mathbb E_{v\sim p(t)}[Q_v(t)],
\]
one has \(\bar A(t)=0\). The formulas retain \(\bar A(t)\) to make
their invariance to action-independent baselines explicit.
Blockwise constancy is a condition on the conditional action score or
advantage. An immediate reward that is constant within a block need
not yield a blockwise-constant advantage in a sequential model,
because continuation values may differ among actions in that block.
Finally, the proposition is an exact differential identity. At a
fixed finite logit vector \(z\in\mathbb R^V\), let
\(p=\operatorname{softmax}(z)\). Since
\(z\mapsto H(\operatorname{softmax}(z))\) is smooth,
\[
H(\operatorname{softmax}(z+\Delta z))
-
H(\operatorname{softmax}(z))
=
-\operatorname{Cov}_{a\sim p}
\bigl(\log p_a,\Delta z_a\bigr)
+
O(\|\Delta z\|_2^2)
\]
as \(\|\Delta z\|_2\to0\). The implicit constant may depend on
\(z\) and \(V\).
\end{remark}
\providecommand{\scafS}{\mathcal S}
\providecommand{\contC}{\mathcal C}
\providecommand{\KL}{\operatorname{KL}}
\providecommand{\HS}{H^{\scafS}}
\providecommand{\HC}{H^{\contC}}
\providecommand{\Hb}{H_{\mathrm b}}

\subsection{Theorem~\ref{thm:reversal}: fork-verdict reversal}

\paragraph{Simplex, entropy, and selection conventions.}
For every nonempty finite set \(A\), define
\[
\Delta_A
:=
\left\{
x\in[0,\infty)^A:
\sum_{a\in A}x_a=1
\right\},
\qquad
\Delta_A^\circ
:=
\left\{
x\in(0,\infty)^A:
\sum_{a\in A}x_a=1
\right\}.
\]
Thus
\[
\Delta_A^\circ=\operatorname{ri}(\Delta_A).
\]
In particular, if \(|A|=1\), then
\[
\Delta_A^\circ=\{(1)\}.
\]
For \(d\geq1\), write
\[
[d]:=\{1,\ldots,d\},
\qquad
\Delta_d^\circ:=\Delta_{[d]}^\circ.
\]
For every probability vector \(x\), define
\[
H(x):=-\sum_a x_a\log x_a,
\]
with the convention \(0\log0=0\). For \(s\in[0,1]\), define
\[
\Hb(s):=-s\log s-(1-s)\log(1-s),
\]
using the same convention at \(s=0,1\). For every nonzero
nonnegative vector \(z\), define
\[
\clo(z):=\frac{z}{\sum_a z_a}.
\]
For a score vector \(a=(a_1,\ldots,a_N)\in\mathbb R^N\) and
\(1\leq K\leq N\), define
\[
\operatorname{Top}_K(a)
:=
\left\{
i\in[N]:
\#\left\{
j\in[N]:
(a_j,-j)>_{\mathrm{lex}}(a_i,-i)
\right\}<K
\right\}.
\]
The pairs \((a_i,-i)\) are distinct, so this rule selects exactly
\(K\) indices. Larger scores are selected first, and smaller indices
break score ties.

\begin{proof}
For every \(r\in\Delta_V^\circ\), the definition of \(G\) gives
\[
H(r)=\HC(r)+G(r).
\]
Therefore, for every \(p,q\in\Delta_V^\circ\),
\[
\begin{aligned}
H(p)-H(q)
&=
\HC(p)+G(p)-\HC(q)-G(q) \\
&=
\delta(p,q)+\gamma(p,q).
\end{aligned}
\]
The two strict rankings are opposite precisely when the two nonzero
gaps
\[
\delta(p,q)
\qquad\text{and}\qquad
\delta(p,q)+\gamma(p,q)
\]
have opposite signs. Two nonzero real numbers have opposite signs if
and only if their product is negative. This proves the ranking
criterion in part~{\rm(i)}.
We next construct an interior strict reversal for every permitted
pair \((m,k)\). Fix arbitrary orderings of the coordinates in
\(\scafS\) and \(\contC\). For \(d\geq2\) and
\(0<\varepsilon<1/d\), define
\[
r_{d,\varepsilon}
:=
\left(
1-(d-1)\varepsilon,
\underbrace{\varepsilon,\ldots,\varepsilon}_{d-1}
\right)
\in\Delta_d^\circ,
\]
and set
\[
r_{1,\varepsilon}:=(1)\in\Delta_1^\circ.
\]
Also let
\[
u_k:=\left(\frac1k,\ldots,\frac1k\right)\in\Delta_k^\circ.
\]
Choose
\[
0<\varepsilon<\frac{1}{\max\{m,k\}}
\]
and define
\[
p_\varepsilon
:=
\left(
(1-\varepsilon)r_{m,\varepsilon},
\ \varepsilon u_k
\right),
\qquad
q_\varepsilon
:=
\left(
\frac12r_{m,\varepsilon},
\ \frac12r_{k,\varepsilon}
\right).
\]
For every \(d\in\{m,k\}\) with \(d\geq2\), the choice of
\(\varepsilon\) implies \(\varepsilon<1/d\), and hence
\[
1-(d-1)\varepsilon
>
1-\frac{d-1}{d}
=
\frac1d
>
0.
\]
If \(d=1\), then \(r_{1,\varepsilon}=(1)\), whose sole coordinate is
positive. Thus all coordinates of \(r_{m,\varepsilon}\) and
\(r_{k,\varepsilon}\) are positive. Since their coordinates sum to
one, it follows that
\[
p_\varepsilon,q_\varepsilon\in\Delta_V^\circ.
\]
Their normalized content laws are
\[
\clo(\varepsilon u_k)=u_k,
\qquad
\clo\left(\frac12r_{k,\varepsilon}\right)
=
r_{k,\varepsilon}.
\]
Hence
\[
\HC(p_\varepsilon)=H(u_k)=\log k,
\qquad
\HC(q_\varepsilon)=H(r_{k,\varepsilon}).
\]
For probability vectors \(a,b\) and \(s\in(0,1)\), direct expansion
gives the two-block grouping identity
\[
H\bigl(sa,(1-s)b\bigr)
=
\Hb(s)+sH(a)+(1-s)H(b).
\]
Applying this identity to \(p_\varepsilon\), and using
\(\Hb(1-\varepsilon)=\Hb(\varepsilon)\), gives
\[
H(p_\varepsilon)
=
\Hb(\varepsilon)
+
(1-\varepsilon)H(r_{m,\varepsilon})
+
\varepsilon\log k.
\]
Applying it to \(q_\varepsilon\) gives
\[
H(q_\varepsilon)
=
\log2
+
\frac12H(r_{m,\varepsilon})
+
\frac12H(r_{k,\varepsilon}).
\]
For every fixed \(d\geq2\),
\[
\begin{aligned}
H(r_{d,\varepsilon})
={}&
-\bigl(1-(d-1)\varepsilon\bigr)
 \log\bigl(1-(d-1)\varepsilon\bigr) \\
&\quad
-(d-1)\varepsilon\log\varepsilon
\longrightarrow0
\qquad
(\varepsilon\downarrow0).
\end{aligned}
\]
Indeed, the first term converges to
\[
-1\log1=0
\]
by continuity, while the second converges to zero because
\[
\varepsilon\log\varepsilon\longrightarrow0.
\]
For \(d=1\),
\[
H(r_{1,\varepsilon})=0
\]
identically. Also,
\[
\Hb(\varepsilon)\longrightarrow0.
\]
Consequently,
\[
H(p_\varepsilon)-H(q_\varepsilon)
\longrightarrow
-\log2<0,
\]
while
\[
\HC(p_\varepsilon)-\HC(q_\varepsilon)
\longrightarrow
\log k>0.
\]
Therefore, there exists \(\varepsilon_0>0\) such that, whenever
\[
0<\varepsilon<
\min\left\{
\varepsilon_0,
\frac{1}{\max\{m,k\}}
\right\},
\]
one has simultaneously
\[
H(p_\varepsilon)<H(q_\varepsilon),
\qquad
\HC(p_\varepsilon)>\HC(q_\varepsilon).
\label{eq:universal-strict-reversal}
\]
For any such \(\varepsilon\), set
\[
p_\star:=p_\varepsilon,
\qquad
q_\star:=q_\varepsilon.
\]
This proves part~{\rm(ii)}. Moreover,
\[
(p_\varepsilon,q_\varepsilon)\in\mathcal R,
\]
so \(\mathcal R\) is nonempty.
It remains to establish openness and positive measure. Define
\[
\Phi(p,q)
:=
\bigl(\HC(p)-\HC(q)\bigr)
\bigl(H(p)-H(q)\bigr).
\]
Entropy is continuous. On \(\Delta_V^\circ\), the content mass
\(p(\contC)\) is positive, so
\[
p\longmapsto
p^{\contC}
=
\frac{p|_{\contC}}{p(\contC)}
\]
is continuous. Hence \(p\mapsto\HC(p)\) and \(\Phi\) are continuous.
Therefore,
\[
\mathcal R
=
\Phi^{-1}((-\infty,0))
\]
is relatively open in \((\Delta_V^\circ)^2\). Since
\((\Delta_V^\circ)^2\) is relatively open in
\(\mathcal A_V\times\mathcal A_V\), it follows that \(\mathcal R\)
is relatively open in \(\mathcal A_V\times\mathcal A_V\).
The affine space \(\mathcal A_V\) has dimension \(D-1\). Therefore,
\(\mathcal A_V\times\mathcal A_V\) has dimension \(2(D-1)\).
Since \(\mathcal R\) is nonempty and relatively open, it contains a
relative Euclidean ball. It consequently has positive induced
\(2(D-1)\)-dimensional Lebesgue measure, equivalently positive
\(2(D-1)\)-dimensional Hausdorff measure. This completes
part~{\rm(i)}.
For the finite-population claim in part~{\rm(iii)}, fix an
\(\varepsilon\) satisfying
\eqref{eq:universal-strict-reversal}. Let \(\rho\in(0,\tfrac12]\), and
let \(N\) be an even positive integer with \(\rho N\in\mathbb N\);
define
\[
p_i
:=
\begin{cases}
q_\varepsilon,
&1\leq i\leq N/2,\\
p_\varepsilon,
&N/2<i\leq N.
\end{cases}
\]
If \(N\) pairwise distinct positions are desired, replace the two
repeated values by \(N\) distinct points drawn from small disjoint
neighborhoods of \(p_\varepsilon\) and \(q_\varepsilon\) inside the
open region of part (i); all strict inequalities below persist.
Every \(q_\varepsilon\)-position has strictly larger raw entropy than
every \(p_\varepsilon\)-position, while every
\(p_\varepsilon\)-position has strictly larger content entropy than
every \(q_\varepsilon\)-position. Because
\[
\rho N\leq\frac{N}{2},
\]
the \(\rho N\) largest raw-entropy scores all come from the first
group, while the \(\rho N\) largest content-entropy scores all come
from the second group. The two selected index sets are therefore
disjoint. Ties occur only within the two groups and are resolved by
the stated index rule. Both selected sets have \(\rho N>0\) elements,
so their Jaccard similarity is zero. (The instantiation
\(\rho=\tfrac15\), \(N\) divisible by \(10\), is the top-\(20\%\) case.)
It remains to prove the population formulas. Let \(Z\) be a
real-valued random variable, let \(u\in(0,1)\), and define
\[
q:=F_Z^{-1}(u).
\]
Writing
\[
F_Z(q-)
:=
\lim_{t\uparrow q}F_Z(t)
=
\mathbb P(Z<q),
\]
the definition of the generalized inverse and the right-continuity
of \(F_Z\) imply
\[
F_Z(q-)\leq u\leq F_Z(q).
\]
If \(\mathbb P(Z=q)=0\), then
\[
F_Z(q)-F_Z(q-)=\mathbb P(Z=q)=0.
\]
It follows that
\[
F_Z(q-)=F_Z(q)=u,
\qquad
\mathbb P(Z\geq q)=1-F_Z(q-)=1-u.
\]
The relevant quantiles are finite because
\[
0\leq X\leq\log k,
\qquad
0\leq Y\leq\log D.
\]
Applying the preceding fact with
\[
(Z,u)=(X,1-\rho)
\qquad\text{and}\qquad
(Z,u)=(Y,1-\rho)
\]
gives
\[
\mathbb P(A_C)=\mathbb P(A_H)=\rho.
\]
Since \(Y=X+W\),
\[
A_C\cap A_H
=
\{X\geq q_C,\ X+W\geq q_H\}.
\]
By the definition of \(\nu=\mathcal L(X,W)\),
\[
\omega_\rho
=
\nu\left(
\left\{
(x,w)\in\mathbb R^2:
x\geq q_C,\quad x+w\geq q_H
\right\}
\right).
\]
Moreover,
\[
\mathbb P(A_C\cup A_H)
=
\mathbb P(A_C)+\mathbb P(A_H)
-\mathbb P(A_C\cap A_H)
=
2\rho-\omega_\rho.
\]
Therefore,
\[
\mathbb P(A_C\mid A_H)
=
\mathbb P(A_H\mid A_C)
=
\frac{\omega_\rho}{\rho},
\]
and
\[
J_\rho
=
\frac{\omega_\rho}{2\rho-\omega_\rho}.
\]
Finally,
\[
\omega_\rho
\leq
\min\{\mathbb P(A_C),\mathbb P(A_H)\}
=
\rho,
\]
while inclusion-exclusion gives
\[
\omega_\rho
\geq
\mathbb P(A_C)+\mathbb P(A_H)-1
=
2\rho-1.
\]
Combining this with \(\omega_\rho\geq0\) yields
\[
\max\{0,2\rho-1\}
\leq
\omega_\rho
\leq
\rho.
\]
In particular,
\[
2\rho-\omega_\rho\geq\rho>0,
\]
so every conditional-probability and Jaccard denominator is strictly
positive. This completes the proof.
\end{proof}

\subsection{Theorem~\ref{prop:exploration}: exploration bound
  with exact leakage}

\begin{proof}
Because \(Z_{<t}\) is a deterministic function of \(Y_{<t}\),
\begin{align}
I(Z_t;Y_{<t}\mid Z_{<t})
&=
H(Z_t\mid Z_{<t})
-
H(Z_t\mid Y_{<t},Z_{<t})
\notag\\
&=
H(Z_t\mid Z_{<t})
-
H(Z_t\mid Y_{<t}).
\label{eq:exploration-one-step-mi}
\end{align}
Fix a positive-probability history \(h\in V^{t-1}\). Conditionally on
\(Y_{<t}=h\), the variable \(Z_t\) assigns probability \(s_t(h)\) to
\(\dagger\) and probability
\[
(1-s_t(h))p_t^{\contC}(v\mid h)
\]
to each \(v\in\contC\). The entropy grouping identity, including the
boundary cases \(s_t(h)\in\{0,1\}\), therefore gives
\[
H(Z_t\mid Y_{<t}=h)
=
\Hb(s_t(h))
+
(1-s_t(h))H_t^{\contC}(h).
\]
Averaging over \(Y_{<t}\) yields
\[
H(Z_t\mid Y_{<t})
=
\mathbb E\!\left[
\Hb(s_t)+(1-s_t)H_t^{\contC}
\right].
\]
Rearranging \eqref{eq:exploration-one-step-mi} gives
\[
H(Z_t\mid Z_{<t})
=
\mathbb E\!\left[
\Hb(s_t)+(1-s_t)H_t^{\contC}
\right]
+
I(Z_t;Y_{<t}\mid Z_{<t}).
\]
Summing over \(t\) and applying the entropy chain rule,
\begin{align*}
H(Z_{1:T})
&=
\sum_{t=1}^{T}H(Z_t\mid Z_{<t})\\
&=
B_{\mathrm{gc}}+\lek.
\end{align*}
This proves \eqref{eq:exploration-censored-entropy}. Since conditional
mutual information is nonnegative,
\[
\lek\geq0.
\]
For a positive-probability history \(h\), the conditional law of
\(Y_t\) is the disjoint-support mixture
\[
p_t(\,\cdot\mid h)
=
\left(
s_t(h)p_t^{\scafS}(\,\cdot\mid h),
(1-s_t(h))p_t^{\contC}(\,\cdot\mid h)
\right).
\]
A second application of the entropy grouping identity gives
\[
H(Y_t\mid Y_{<t}=h)
=
\Hb(s_t(h))
+s_t(h)H_t^{\scafS}(h)
+(1-s_t(h))H_t^{\contC}(h).
\]
After averaging over \(Y_{<t}\), summing over \(t\), and applying the
entropy chain rule,
\begin{align}
H(Y_{1:T})
&=
\sum_{t=1}^{T}H(Y_t\mid Y_{<t})
\notag\\
&=
B_{\mathrm{gc}}+B_{\mathrm{sc}}.
\label{eq:exploration-full-entropy}
\end{align}
Because \(Z_{1:T}\) is a deterministic function of \(Y_{1:T}\),
\[
H(Y_{1:T})
=
H(Z_{1:T})
+
H(Y_{1:T}\mid Z_{1:T}).
\]
Substituting \eqref{eq:exploration-censored-entropy} and
\eqref{eq:exploration-full-entropy} gives
\[
B_{\mathrm{gc}}+B_{\mathrm{sc}}
=
B_{\mathrm{gc}}+\lek
+
H(Y_{1:T}\mid Z_{1:T}).
\]
Therefore,
\[
\lek
=
B_{\mathrm{sc}}
-
H(Y_{1:T}\mid Z_{1:T}),
\]
which proves \eqref{eq:exploration-leakage-slack}. Since conditional
entropy is nonnegative,
\[
\lek\leq B_{\mathrm{sc}}.
\]
Moreover,
\[
\lek=B_{\mathrm{sc}}
\quad\Longleftrightarrow\quad
H(Y_{1:T}\mid Z_{1:T})=0.
\]
Because \(Y_{1:T}\) and \(Z_{1:T}\) have finite alphabets, this is
equivalent to almost-sure recoverability of \(Y_{1:T}\) from
\(Z_{1:T}\).
The two chain-rule decompositions of the joint entropy of
\((\alpha,Z_{1:T})\) give
\[
H(\alpha)+H(Z_{1:T}\mid\alpha)
=
H(Z_{1:T})+H(\alpha\mid Z_{1:T}).
\]
All quantities are finite because \(Z_{1:T}\) has a finite alphabet
and \(H(\alpha)<\infty\). Substituting
\eqref{eq:exploration-censored-entropy} and rearranging yields
\[
H(\alpha)
=
B_{\mathrm{gc}}+\lek
+H(\alpha\mid Z_{1:T})
-H(Z_{1:T}\mid\alpha),
\]
which proves \eqref{eq:exploration-answer-identity}. Dropping the
nonnegative term \(H(Z_{1:T}\mid\alpha)\) proves
\[
H(\alpha)
\leq
B_{\mathrm{gc}}+\lek+H(\alpha\mid Z_{1:T}).
\]
Equality holds precisely when
\[
H(Z_{1:T}\mid\alpha)=0.
\]
The zero-conditional-entropy characterization for discrete random
variables proves all stated functional-recoverability equivalences.
Every summand defining \(\lek\) is nonnegative. Hence
\[
\lek=0
\quad\Longleftrightarrow\quad
I(Z_t;Y_{<t}\mid Z_{<t})=0
\quad\text{for every }t.
\]
For finite trajectory alphabets,
\[
I(Z_t;Y_{<t}\mid Z_{<t})=0
\quad\Longleftrightarrow\quad
Z_t\mathrel{\perp\!\!\!\perp}Y_{<t}\mid Z_{<t}.
\]
Since \(Z_{<t}\) is a function of \(Y_{<t}\), this conditional
independence is equivalent to the existence of a kernel \(K_t\) such
that
\[
\mathbb P(Z_t=z\mid Y_{<t})
=
K_t(z\mid Z_{<t})
\]
almost surely for every \(z\in\mathcal Z\). This proves the inertness
characterization.
It remains to establish simultaneous sharpness. Assume \(T\geq2\).
Choose distinct nonterminal scaffold symbols
\(w_0,w_1\in\scafS\) and distinct content symbols
\(a_0,a_1\in\contC\). Define
\[
\mathbb P(Y_1=w_i)=\frac12,
\qquad
\mathbb P(Y_2=a_i\mid Y_1=w_i)=1,
\qquad i\in\{0,1\},
\]
and, when \(T>2\), set
\[
Y_t=a_0
\qquad\text{almost surely for }t=3,\ldots,T.
\]
Let
\[
\alpha:=Y_2.
\]
Then
\[
Z_1=\dagger,
\qquad
Z_2=\alpha,
\qquad
Z_t=a_0\quad(t\geq3).
\]
At time \(1\), the block is deterministically scaffold and the
within-scaffold entropy equals \(\log2\). At every later time, both the
block and the within-block token are conditionally deterministic.
Therefore,
\[
B_{\mathrm{gc}}=0,
\qquad
B_{\mathrm{sc}}=\log2.
\]
Furthermore,
\[
\lek
=
I(Z_2;Y_1\mid Z_1)
=
\log2,
\]
while every other summand in \(\lek\) is zero. The visible sequence
identifies the hidden choice \(w_i\), and \(\alpha\) identifies the
entire visible sequence. Consequently,
\[
H(Y_{1:T}\mid Z_{1:T})=0,
\qquad
H(\alpha\mid Z_{1:T})=0,
\qquad
H(Z_{1:T}\mid\alpha)=0.
\]
Thus
\[
H(\alpha)
=
\lek
=
B_{\mathrm{sc}}
=
\log2,
\]
so the upper leakage bound and the answer bound are simultaneous
equalities.
Finally, consider full-support constraints. Fix versions
\(p_t^0(\cdot\mid h)\) of the sharpness-witness kernels, extended
arbitrarily to null histories (EOS-absorbing where applicable),
and let \(u_V\) be uniform on \(V\). In the fixed-length case, define
\[
p_t^\varepsilon(\,\cdot\mid h)
:=
(1-\varepsilon)p_t^0(\,\cdot\mid h)
+\varepsilon u_V,
\qquad 0<\varepsilon<1.
\]
Every resulting kernel has full support.
In the absorbing-EOS case, use the same definition when \(h\) does
not contain \(\mathrm{EOS}\), and set
\[
p_t^\varepsilon(\,\cdot\mid h)
:=
\delta_{\mathrm{EOS}}
\]
when \(h\) contains \(\mathrm{EOS}\). Thus every pre-termination kernel
has full support and EOS remains absorbing. In both cases, define
\[
\alpha^\varepsilon:=Y_2^\varepsilon.
\]
Because the horizon and vocabulary are finite, each joint probability
is a finite product of kernel entries, each converging pointwise as
\(\varepsilon\downarrow0\); hence the induced joint laws
of
\[
(\alpha^\varepsilon,Y_{1:T}^\varepsilon,Z_{1:T}^\varepsilon)
\]
converge in total variation to the sharpness witness as
\(\varepsilon\downarrow0\). Shannon entropy and mutual information are
continuous on finite probability simplices. Hence
\[
B_{\mathrm{sc}}^\varepsilon-\lek^\varepsilon
=
H(Y_{1:T}^\varepsilon\mid Z_{1:T}^\varepsilon)
\longrightarrow0
\]
and
\begin{align*}
&B_{\mathrm{gc}}^\varepsilon+\lek^\varepsilon
+H(\alpha^\varepsilon\mid Z_{1:T}^\varepsilon)
-H(\alpha^\varepsilon)\\
&\hspace{4em}
=
H(Z_{1:T}^\varepsilon\mid\alpha^\varepsilon)
\longrightarrow0.
\end{align*}
Thus both equalities can be approached arbitrarily closely under the
stated full-support requirement while preserving EOS absorption.
\end{proof}

\begin{remark}[Random prompts]
\label{rem:exploration-random-prompts}
Let \(X\) be a random prompt taking values in a standard Borel space,
and assume \(H(\alpha)<\infty\). All conditional quantities below are
computed from regular conditional distributions given \((X,Y_{<t})\),
which exist because \(V\) is finite and \(X\) is standard Borel. Define
\[
s_t
:=
\mathbb P(Y_t\in\scafS\mid X,Y_{<t}),
\]
and define \(H_t^{\scafS}\) and \(H_t^{\contC}\) from the corresponding
conditional within-block laws given \((X,Y_{<t})\), using the same
zero-mass conventions as above. Put
\[
B_{\mathrm{gc}}^{X}
:=
\sum_{t=1}^{T}
\mathbb E\!\left[
\Hb(s_t)+(1-s_t)H_t^{\contC}
\right],
\]
\[
B_{\mathrm{sc}}^{X}
:=
\sum_{t=1}^{T}
\mathbb E\!\left[s_tH_t^{\scafS}\right],
\]
and
\[
\lek_X
:=
\sum_{t=1}^{T}
I(Z_t;Y_{<t}\mid X,Z_{<t}).
\]
Then
\[
H(Z_{1:T}\mid X)
=
B_{\mathrm{gc}}^{X}+\lek_X,
\]
\[
\lek_X
=
B_{\mathrm{sc}}^{X}
-
H(Y_{1:T}\mid X,Z_{1:T}),
\]
and
\[
H(\alpha\mid X)
=
B_{\mathrm{gc}}^{X}+\lek_X
+H(\alpha\mid X,Z_{1:T})
-H(Z_{1:T}\mid X,\alpha).
\]
Consequently,
\[
H(\alpha\mid X)
\leq
B_{\mathrm{gc}}^{X}+\lek_X
+H(\alpha\mid X,Z_{1:T}).
\]
Moreover,
\[
\lek_X=0
\quad\Longleftrightarrow\quad
Z_t\mathrel{\perp\!\!\!\perp}Y_{<t}
\mid (X,Z_{<t})
\quad\text{for every }t.
\]
For unconditional answer entropy,
\begin{align*}
H(\alpha)
&=
B_{\mathrm{gc}}^{X}+\lek_X
+H(\alpha\mid X,Z_{1:T})\\
&\quad
-H(Z_{1:T}\mid X,\alpha)
+I(\alpha;X).
\end{align*}
Thus conditioning all quantities on \(X\) gives a statement about
\(H(\alpha\mid X)\). A bound for \(H(\alpha)\) additionally requires
the prompt-information term \(I(\alpha;X)\).
\end{remark}

\subsection{Definition~\ref{def:attr} is exact, and where it is not}
Writing $\Phi(s,x,y)=\Hb(s)+sx+(1-s)y$, the two telescopings
$\Delta H=[\Phi(s_1,\HS_1,\HC_0)-\Phi_0]+[\Phi_1-\Phi(s_1,\HS_1,\HC_0)]$
and the reverse order give content parts $(1-s_1)\Delta\HC$ and
$(1-s_0)\Delta\HC$; their average is
$(1-\bar s)\Delta\HC$ with $\bar s=\tfrac{s_0+s_1}{2}$, and the
scaffold part is defined as the exact remainder, so the two parts sum
to $\Delta H$ identically. This is the
two-player Shapley value of the coalition game that assigns to
$\{$content$\}$ the entropy move realized by updating $\HC$ alone.
\textbf{Scope.} Exactness is per matched context, and pointwise
attributions average exactly over any \emph{fixed} context set (as in
the SYN-2 constructions, where the population is literally reused, and
in the protocol's teacher-forced attribution set). Applying the formula
to separately averaged $(\bar s,\bar\HS,\bar\HC)$ is NOT exact, since
$\mathbb{E}\,\Phi\ne\Phi(\mathbb{E}\,\cdot)$; and on-policy evals at
different checkpoints sample different contexts, an occupancy shift no
per-position formula can attribute. The analysis therefore never
applies the formula to unaligned means and reports linear channel
aggregates for on-policy series.

\subsection{Proofs for \S\ref{sec:firstpass}: first passage and routing
information}

\begin{proof}[Proof of the recursion \eqref{eq:firstpass}]
Condition on the next token. If $Y_t\in\contC$ (probability $1-s_t$),
then $\tau_t=t$ and $Y_{\tau_t}\sim\pC_t$. If $Y_t=w\in\scafS$
(probability $s_t\pS_t(w)$), then $\tau_t=\tau_{t+1}$ for the extended
prefix $(h_t,w)$, whose eventual-content law is $q^{(w)}_{t+1}$ by
definition. Total probability gives
$q_t=(1-s_t)\pC_t+s_t\sum_w\pS_t(w)\,q^{(w)}_{t+1}$. (For the
recursion to be well posed, $\tau_t<\infty$ a.s.; with the EOS-padding
convention of Theorem~\ref{prop:exploration}, trajectories that
never emit content contribute an absorbing symbol treated as its own
content outcome, and the estimator reports the mass that hits
the truncation budget instead of hiding it.)
\end{proof}

\begin{proof}[Proof of the $\kappa_t$ identity]
Conditionally on $h_t$ and $Y_t\in\scafS$, the pair
$(Y_t,Y_{\tau_{t+1}})$ has joint law $\pS_t(w)\,q^{(w)}_{t+1}(v)$.
Mutual information of a mixture is the Jensen gap of entropy:
$I=H(\sum_w\pS_t(w)q^{(w)}_{t+1})-\sum_w\pS_t(w)H(q^{(w)}_{t+1})\ge0$,
with equality iff all mixture components with positive weight coincide,
which is exactly the statement that the eventual next content token is
independent of which scaffold token was chosen: one-step scaffold
inertness at $h_t$. If this holds at every prefix (and likewise for the
block flag), the conditional law of $Z_{t'}$ given $Y_{<t'}$ is
measurable in $Z_{<t'}$, so every conditional mutual information in
$\lek$ vanishes (Theorem~\ref{prop:exploration}); conversely a
prefix with $\kappa_t>0$ witnesses non-inertness. Estimation: the
estimator replaces $q^{(w)}_{t+1}$ by empirical first-passage
samples ($k$ rollouts per substituted $w$); plug-in entropy is biased
downward by $O((\text{support})/k)$, the bias partially cancels in the
Jensen gap, and problem-clustered bootstrap intervals are reported; we
therefore present measured $\kappa$ as an estimate with intervals, not
as an exact quantity.
\end{proof}

\begin{remark}[$\kappa_t$ witnesses $\lek$ but does not determine it]
\label{rem:kappa-not-lambda}
The implication in the preceding proof runs one way only, and the
converse is false. Take $\scafS=\{w_0,w_1\}$, $\contC=\{a,b,c\}$ and
$T=3$: let $Y_1$ be a fair coin on $\{w_0,w_1\}$, let $Y_2=a$ under both
branches, and let $Y_3=b$ after $w_0$ and $Y_3=c$ after $w_1$. The first
content token after step $1$ is $a$ with probability one under either
connective, so $\kappa_1=0$, and the timing-aware refinement is zero as
well because the waiting time is also identical. Yet
$I(Z_3;Y_{<3}\mid Z_{<3})=\log 2$, so $\lek=\log 2$. Routing that acts
only beyond the first content token is invisible to $\kappa_t$.
Consequently $\kappa_t>0$ certifies $\lek>0$, and a measured
$\hat\kappa$ bounds scaffold inertness away, but no numerical relation
between $\hat\kappa$ and $\lek$ is available from these results. An
estimator of $\lek$ would target its own summands
$I(Z_t;Y_{<t}\mid Z_{<t})$, or a routing statistic whose outcome is the
whole censored suffix; we do not attempt either here.
\end{remark}

\subsection{Proofs for \S\ref{sec:instruments}: the instrument calculus}

\begin{proof}[Proof of Theorem~\ref{thm:instruments}]
\emph{(i).} The midpoint $m=\tfrac{p+q}2$ has gate mass $\bar s$ and
within-block channels equal to the skew mixtures
$\lambda_S\pS+(1-\lambda_S)q^{\scafS}$ and
$\lambda_C\pC+(1-\lambda_C)q^{\contC}$ (collect coefficients blockwise:
the scaffold block of $m$ is $\tfrac{s\pS+t\,q^{\scafS}}2$, whose mass is
$\bar s$ and whose normalization is the $\lambda_S$-mixture). Apply the
three-channel KL chain rule (Proposition~\ref{thm:chain}(ii)) to
$\KL(p\|m)$ and $\KL(q\|m)$ and average: the binary parts assemble
$\JS_b(s,t)$; the scaffold parts give
$\tfrac{s}2\KL(\pS\|m^{\scafS})+\tfrac{t}2\KL(q^{\scafS}\|m^{\scafS})
=\bar s\,\JS_{\lambda_S}(\pS,q^{\scafS})$, and likewise for content.
With $|\scafS|=1$ the scaffold term vanishes and the singleton law is recovered.
\emph{(ii).} At $s=t$, a coordinate $v\in\scafS$ contributes
$t\,q^{\scafS}_v f\big(\tfrac{s\pS_v}{t\,q^{\scafS}_v}\big)
=s\,q^{\scafS}_vf(\pS_v/q^{\scafS}_v)$ to $D_f$; summing blockwise gives
the mixture law. R\'enyi follows from
$\sum_v p_v^{\alpha}q_v^{1-\alpha}
=s\sum_{\scafS}(\pS)^{\alpha}(q^{\scafS})^{1-\alpha}
+(1-s)\sum_{\contC}(\pC)^{\alpha}(q^{\contC})^{1-\alpha}$.
Bhattacharyya splits blockwise by the same computation; substituting
colatitudes $\arccos\sqrt s$, $\arccos\sqrt t$ and the two block angles
turns it into the displayed two-dihedral law.
\emph{(iii).} Under scaffold locking, (ii) gives
$D_f(p,q)=(1-s)D_f(\pC,q^{\contC})$ exactly ($\gamma=1$; Hellinger
$=\sqrt{1-s}\,\mathrm{Hel}(\pC,q^{\contC})$, $\gamma=\nicefrac12$);
cosine follows the companion's expansion with $\|\pS\|_2^2$ in the
prefactor ($\gamma=2$); the channel-resolved Aitchison distances do not
involve $s$ at all ($\gamma=0$), by Lemma~\ref{lem:pyth}. All claims were
verified numerically, including the $|\scafS|=1$ reductions.
\end{proof}

\begin{proof}[Proof of Corollary~\ref{cor:klpenalty}]
Immediate from Proposition~\ref{thm:chain}(ii): at matched gate and locked
scaffold the binary and scaffold terms vanish and the content term
carries weight $1-s$. The general display is the chain rule itself.
\end{proof}

\begin{proof}[In the proof of Theorem~\ref{thm:downstream}]
\emph{(i)}the only
property used is $\mathbf 1^{\top}\dvec=0$, unaffected by the sink being
a set. \emph{(ii).} Coordinatewise: on $\scafS$,
$\Delta s\,m^{\scafS}_i+\bar s\,\Delta^{\scafS}_i
=\tfrac{(s-t)(\pS_i+q^{\scafS}_i)+(s+t)(\pS_i-q^{\scafS}_i)}2
=s\pS_i-t q^{\scafS}_i=\dvec_i$; on $\contC$,
$-\Delta s\,m^{\contC}_i+(1-\bar s)\Delta^{\contC}_i
=(1-s)\pC_i-(1-t)q^{\contC}_i=\dvec_i$. Expanding
$\dvec^{\top}G\dvec$ bilinearly in the three directions gives six terms;
$q_G(\kappa)=\|E^{\top}\kappa\|^2=\|\bar v_{\scafS}-\bar v_{\contC}\|^2$
with the mixture-weighted block means. \emph{(iii)} is (ii) at $d=1$;
the between-block coefficient $\Delta s(\bar A_{\scafS}-\bar A_{\contC})$
is the advantage-gap structure of Proposition~\ref{prop:dynamics}.
\end{proof}

\subsection{Proposition~\ref{prop:unmatched-sink-regimes}: unmatched-sink regimes}

\begin{proof}
Fix \(n\) and suppress its subscript, writing
\[
(u,v,a,\lambda,\pi,\rho,p,q,R,B,C)
:=
(u_n,v_n,a_n,\lambda_n,\pi_n,\rho_n,p_n,q_n,R_n,B_n,C_n).
\]
Put
\[
\mu
:=
\mu_\lambda(\pi,\rho)
=
\lambda\pi+(1-\lambda)\rho.
\]
Since
\[
u=2a\lambda,
\qquad
v=2a(1-\lambda),
\]
the equal-weight mixture of \(p\) and \(q\) is
\[
r
:=
\frac{p+q}{2}
=
\left(
1-a,
\frac{u\pi+v\rho}{2}
\right)
=
(1-a,a\mu).
\]
Expanding the KL divergences coordinatewise gives
\begin{align}
\operatorname{KL}(p\|r)
={}&
(1-u)\log\frac{1-u}{1-a}
+
\sum_{i\in\mathcal I}
u\pi_i
\log\frac{u\pi_i}{a\mu_i}
\nonumber\\
={}&
(1-u)\log\frac{1-u}{1-a}
+
u\log\frac{u}{a}
+
u\operatorname{KL}(\pi\|\mu),
\label{eq:unmatched-kl-p}
\end{align}
and
\begin{align}
\operatorname{KL}(q\|r)
={}&
(1-v)\log\frac{1-v}{1-a}
+
\sum_{i\in\mathcal I}
v\rho_i
\log\frac{v\rho_i}{a\mu_i}
\nonumber\\
={}&
(1-v)\log\frac{1-v}{1-a}
+
v\log\frac{v}{a}
+
v\operatorname{KL}(\rho\|\mu).
\label{eq:unmatched-kl-q}
\end{align}
Taking one half of the sum of
\eqref{eq:unmatched-kl-p} and \eqref{eq:unmatched-kl-q} yields
\begin{align}
\operatorname{JS}(p,q)
={}&
R
+
\frac12
\left[
u\log\frac{u}{a}
+
v\log\frac{v}{a}
\right]
\nonumber\\
&+
\frac12
\left[
u\operatorname{KL}(\pi\|\mu)
+
v\operatorname{KL}(\rho\|\mu)
\right].
\label{eq:unmatched-js-three-terms}
\end{align}
Because
\[
\frac{u}{2}=a\lambda,
\qquad
\frac{v}{2}=a(1-\lambda),
\]
the content contribution is
\begin{align*}
\frac12
\left[
u\operatorname{KL}(\pi\|\mu)
+
v\operatorname{KL}(\rho\|\mu)
\right]
&=
a
\left[
\lambda\operatorname{KL}(\pi\|\mu)
+
(1-\lambda)\operatorname{KL}(\rho\|\mu)
\right]\\
&=
a\operatorname{JS}_\lambda(\pi,\rho).
\end{align*}
The remaining nonsink-mass contribution is
\begin{align*}
\frac12
\left[
u\log\frac{u}{a}
+
v\log\frac{v}{a}
\right]
&=
a
\left[
\lambda\log(2\lambda)
+
(1-\lambda)\log(2(1-\lambda))
\right]\\
&=
a\bigl[\log 2-h_b(\lambda)\bigr].
\end{align*}
Substitution into \eqref{eq:unmatched-js-three-terms} proves
\eqref{eq:unmatched-js-exact}.
Applying the same calculation to
\[
(1-u,u)
\quad\text{and}\quad
(1-v,v)
\]
gives
\[
B
=
a\bigl[\log 2-h_b(\lambda)\bigr]+R.
\]
Therefore,
\[
\operatorname{JS}(p,q)
=
B+a\operatorname{JS}_\lambda(\pi,\rho)
=
B+C,
\]
which proves \eqref{eq:unmatched-js-channel-decomposition}.
We next bound \(R\). Define
\[
F(x):=(1-x)\log(1-x),
\qquad
d:=\frac{u-v}{2}=a(2\lambda-1).
\]
Since
\[
u=a+d,
\qquad
v=a-d,
\]
equation \eqref{eq:unmatched-js-remainder} becomes
\[
R
=
\frac{F(a+d)+F(a-d)}{2}
-
F(a).
\]
On \([0,1)\),
\[
F''(x)=\frac{1}{1-x}>0.
\]
Hence \(F\) is convex and Jensen's inequality gives \(R\geq0\).
If \(d=0\), then \(R=0\). Suppose \(d\neq0\). Taylor's theorem,
applied separately to \(F(a+d)\) and \(F(a-d)\), gives points
\(\xi_+\) between \(a\) and \(a+d\), and \(\xi_-\) between \(a\)
and \(a-d\), such that
\[
F(a+d)
=
F(a)+dF'(a)+\frac{d^2}{2}F''(\xi_+),
\]
and
\[
F(a-d)
=
F(a)-dF'(a)+\frac{d^2}{2}F''(\xi_-).
\]
Consequently,
\[
R
=
\frac{d^2}{4}
\left[
F''(\xi_+)+F''(\xi_-)
\right].
\]
The interval with endpoints \(a-d=v\) and \(a+d=u\) is contained in
\([0,2a]\). Thus, whenever \(a<1/2\),
\begin{align*}
R
&\leq
\frac{d^2}{2}
\sup_{0\leq x\leq2a}F''(x)\\
&=
\frac{d^2}{2(1-2a)}
=
\frac{a^2(2\lambda-1)^2}{2(1-2a)}.
\end{align*}
This proves \eqref{eq:unmatched-js-remainder-bound}.
Since \(a_n\to0\), eventually \(a_n\leq1/4\), and then
\[
0\leq R_n\leq a_n^2.
\]
Thus \(R_n=O(a_n^2)\) with an absolute constant, proving
\eqref{eq:unmatched-js-first-order}.
We next establish a uniform bound for the skew divergence.
Coordinatewise,
\[
\mu\geq\lambda\pi,
\qquad
\mu\geq(1-\lambda)\rho.
\]
Therefore,
\[
\operatorname{KL}(\pi\|\mu)
\leq-\log\lambda,
\qquad
\operatorname{KL}(\rho\|\mu)
\leq-\log(1-\lambda),
\]
and hence
\begin{equation}
0
\leq
\operatorname{JS}_\lambda(\pi,\rho)
\leq
-\lambda\log\lambda
-(1-\lambda)\log(1-\lambda)
=
h_b(\lambda).
\label{eq:unmatched-skew-js-bound}
\end{equation}
This bound is independent of the content pair.
Since \(h_b(\lambda)\leq\log 2\) (giving the first display for every
\(n\)) and \(R_n\leq a_n^{2}/(2(1-2a_n))\leq a_n^{2}\) whenever
\(a_n\leq1/4\), hence for all sufficiently large \(n\),
\[
0\leq C_n\leq a_n\log 2,
\]
and
\[
0
\leq
B_n
\leq
a_n\log 2+a_n^2.
\]
This proves the uniform \(O(a_n)\) bounds.
Because both weights in skew JS are strictly positive and KL is
definite,
\[
\operatorname{JS}_\lambda(\pi,\rho)=0
\quad\Longleftrightarrow\quad
\pi=\rho.
\]
Thus, for every \(n\),
\[
C_n=0
\quad\Longleftrightarrow\quad
\pi_n=\rho_n.
\]
Similarly, definiteness of ordinary JS gives
\[
B_n=0
\quad\Longleftrightarrow\quad
u_n=v_n
\quad\Longleftrightarrow\quad
\lambda_n=\frac12.
\]
Suppose now that
\[
\min\{\lambda_n,1-\lambda_n\}\longrightarrow0.
\]
Then \(h_b(\lambda_n)\to0\), and
\[
0
\leq
\frac{C_n}{a_n}
\leq
h_b(\lambda_n)
\longrightarrow0
\]
uniformly over all content sequences. Hence \(C_n=o(a_n)\).
Moreover, whenever \(a_n<1/2\), hence eventually,
\[
\left|
\frac{B_n}{a_n}-\log 2
\right|
\leq
h_b(\lambda_n)
+
\frac{a_n}{2(1-2a_n)}
\longrightarrow0.
\]
Thus
\[
B_n=a_n\log 2+o(a_n).
\]
Combining these conclusions with
\(\operatorname{JS}(p_n,q_n)=B_n+C_n\) proves
\eqref{eq:unmatched-js-rate-mismatch}.
Under the separate comparable-rate condition
\[
\varepsilon
\leq
\lambda_n
\leq
1-\varepsilon
\]
eventually, for some \(\varepsilon\in(0,1/2)\), we have
\[
0
\leq
\frac{C_n}{a_n}
=
\operatorname{JS}_{\lambda_n}(\pi_n,\rho_n)
\leq
\log 2.
\]
Therefore,
\[
C_n=\Theta(a_n)
\quad\Longleftrightarrow\quad
\liminf_{n\to\infty}
\operatorname{JS}_{\lambda_n}(\pi_n,\rho_n)>0.
\]
If \(\pi_n\equiv\pi\), \(\rho_n\equiv\rho\), and
\(\pi\neq\rho\), then each coordinate of
\(\mu_\lambda(\pi,\rho)\) is strictly positive (by the open-simplex
convention \(\mathcal S^{D-1}\subset(0,1)^{D-1}\)) and continuous in
\(\lambda\). Hence
\[
\lambda\longmapsto
\operatorname{JS}_\lambda(\pi,\rho)
\]
is continuous on \((0,1)\). It is strictly positive there because
skew JS is definite for every \(\lambda\in(0,1)\). It therefore has
a strictly positive minimum on
\([\varepsilon,1-\varepsilon]\), proving
\(C_n=\Theta(a_n)\).
For the sink channel, define
\[
g(\lambda):=\log 2-h_b(\lambda).
\]
The function \(g\) is continuous, nonnegative, and has the unique
zero \(\lambda=1/2\). Moreover, whenever \(a_n<1/2\), hence
eventually,
\[
\frac{B_n}{a_n}
=
g(\lambda_n)
+
\frac{R_n}{a_n},
\qquad
0
\leq
\frac{R_n}{a_n}
\leq
\frac{a_n}{2(1-2a_n)}
\longrightarrow0.
\]
If
\[
\liminf_{n\to\infty}
\left|
\lambda_n-\frac12
\right|>0,
\]
then continuity and compactness give
\[
\liminf_{n\to\infty}g(\lambda_n)>0,
\]
and hence \(B_n=\Theta(a_n)\).
Conversely, if
\[
\liminf_{n\to\infty}
\left|
\lambda_n-\frac12
\right|=0,
\]
there exists a subsequence \((n_k)\) such that
\[
\lambda_{n_k}\longrightarrow\frac12.
\]
Along that subsequence,
\[
\frac{B_{n_k}}{a_{n_k}}
=
g(\lambda_{n_k})
+
\frac{R_{n_k}}{a_{n_k}}
\longrightarrow0.
\]
Therefore \(B_n\neq\Theta(a_n)\), which proves
\eqref{eq:unmatched-sink-theta}. Combining
\eqref{eq:unmatched-sink-theta} with
\eqref{eq:unmatched-content-theta}, whose right-hand side holds for
fixed \(\pi\neq\rho\) as just shown, yields the final two-channel
equivalence. This completes the proof.
\end{proof}

\subsection{Theorem~\ref{thm:no-free-lunch-corrected}: no free lunch
  under compositional axioms}
\label{app:nflnotation}

\paragraph{Companion notation.} For $p,q$ in the open simplex
$\mathcal S^m$: perturbation $(c\oplus p)=\clo(c_ip_i)_i$; powering
$(\alpha\odot p)=\clo(p_i^\alpha)_i$;
$\clr p=\log p-\overline{\log p}\,\mathbf 1$;
$d_{A,m}(p,q)=\|\clr p-\clr q\|_2$; the CLR hyperplane is
$\mathcal H_m=\{x\in\mathbb R^m:\mathbf1^\top x=0\}$; $g$ is the
geometric mean. With the sink at coordinate $0$, the content
subcomposition is $p^{(-0)}=\clo\bigl((p_i)_{i\geq1}\bigr)$ and the
sink balance is
$b_0(p)=\sqrt{\tfrac{D-1}{D}}\,\log\bigl(p_0/g(p_1,\dots,p_{D-1})\bigr)$, with the geometric mean taken on the \emph{raw} non-sink subvector (normalizing it first shifts $b_0$ by $\log(1-p_0)$ and breaks the identity $b_0(p_s^\pi)=\sqrt{(D-1)/D}\,(\operatorname{logit} s-\log g(\pi))$ used in the proof),
with $b_j$ and $p^{(-j)}$ defined analogously at coordinate $j$. A
dissimilarity $\delta$ is \emph{sink-separable at coordinate $j$} if
$\delta(p,q)^2=\varphi_j\bigl(b_j(p),b_j(q)\bigr)+\Psi_j\bigl(p^{(-j)},q^{(-j)}\bigr)$
for finite real functions $\varphi_j,\Psi_j$, and \emph{Sep-A} if
this holds at every coordinate. The axioms are (S1)
$\oplus$-invariance $\delta(c\oplus p,c\oplus q)=\delta(p,q)$; (S2)
$\odot$-homogeneity $\delta(\alpha\odot p,\alpha\odot
q)=|\alpha|\,\delta(p,q)$; (S3) symmetry and definiteness; (S4)
Sep-A. 

\begin{proof}
\textit{Part \textnormal{(i)}.}
Because \(\delta\not\equiv0\), there are
\(p^\star,q^\star\in\mathcal S^D\) and \(\varepsilon>0\) such that
\(\delta(p^\star,q^\star)=\varepsilon\).  For \(r>0\), let
\[
c_r=\mathcal C(r,1,\ldots,1).
\]
For every \(p\in\mathcal S^D\), closure cancels the normalizing
constant in \(c_r\), and therefore
\[
c_r\oplus p
=
\frac{(rp_0,p_1,\ldots,p_{D-1})}
{rp_0+\sum_{i=1}^{D-1}p_i}
\longrightarrow e_0
\qquad(r\to\infty).
\]
Perturbation invariance gives
\[
\delta(c_r\oplus p^\star,c_r\oplus q^\star)=\varepsilon
\qquad(r>0).
\]
If \(\bar\delta\) were a continuous extension, then for every
\(x\in\overline{\mathcal S^D}\) an interior sequence \(z_m\to x\)
would give
\(\bar\delta(x,x)=\lim_m\delta(z_m,z_m)=0\).  In particular,
\(\bar\delta(e_0,e_0)=0\).  Consequently, the
left-hand side would converge to
\(\bar\delta(e_0,e_0)=0\), which is impossible.
For the second assertion, suppose that a jointly continuous
extension exists.  For every \(\pi\in\mathcal S^{D-1}\),
\[
\|p_s^\pi-e_0\|_2^2
=(1-s)^2+(1-s)^2\|\pi\|_2^2
\le 2(1-s)^2.
\]
Thus \((p_s^\pi,p_s^\rho)\to(e_0,e_0)\) uniformly over
\((\pi,\rho)\), and continuity of \(\bar\delta\) at
\((e_0,e_0)\) proves the uniform limit.
No universal polynomial rate follows from continuity alone.  Define
\(h:[0,\infty)\to[0,\infty)\) by \(h(0)=0\) and
\[
h(u)=
\begin{cases}
\displaystyle\frac1{\log(e/u)},&0<u\le e^{-1},\\[1.2ex]
\displaystyle\frac12+\frac e4\bigl(u-e^{-1}\bigr),&u>e^{-1}.
\end{cases}
\]
The two branches have the same value and derivative at \(e^{-1}\).
On \((0,e^{-1}]\),
\[
h''(u)=
\frac{2-\log(e/u)}{u^2\log(e/u)^3}\le0.
\]
Thus \(h\) is increasing and concave on \([0,\infty)\), with
\(h(0)=0\).  If \(a,b>0\), concavity with
\(\theta=a/(a+b)\) gives
\(h(a)\ge\theta h(a+b)\) and
\(h(b)\ge(1-\theta)h(a+b)\); hence \(h\) is subadditive.
Therefore
\[
\delta_0(p,q)=h(\|p-q\|_2)
\]
is a metric and extends continuously to the closed simplex.  Since
\(\delta_0(p_s^\pi,p_s^\rho)=h\bigl((1-s)\|\pi-\rho\|_2\bigr)\) and
\(h(u)/u^\gamma=1/\bigl(u^{\gamma}\log(e/u)\bigr)\to\infty\) as
\(u\downarrow0\) for every \(\gamma>0\), for
every \(\gamma>0\) and every \(\pi\ne\rho\),
\[
\frac{\delta_0(p_s^\pi,p_s^\rho)}{(1-s)^\gamma}
\longrightarrow\infty
\qquad(s\uparrow1).
\]
\medskip
\textit{Part \textnormal{(ii)(a)}.}
For \(p_s^\pi=(s,(1-s)\pi)\),
\[
b_0(p_s^\pi)
=
\sqrt{\frac{D-1}{D}}
\left(\operatorname{logit}(s)-\log g(\pi)\right).
\]
Moreover, \((p_s^\pi)^{(-0)}=\pi\).
Since \(g(\sigma\pi)=g(\pi)\), the ordered balance pair for
\((p_s^\pi,p_s^{\sigma\pi})\) is the same as that for
\((p_s^\pi,p_s^\pi)\).  Subtracting the two separability
identities yields
\begin{align*}
&\delta(p_s^\pi,p_s^{\sigma\pi})^2
-\delta(p_s^\pi,p_s^\pi)^2\\
&\hspace{2cm}
=
\Psi_0(\pi,\sigma\pi)-\Psi_0(\pi,\pi).
\end{align*}
Because \(\delta(p,p)=0\),
\[
\delta(p_s^\pi,p_s^{\sigma\pi})^2
=
\Psi_0(\pi,\sigma\pi)-\Psi_0(\pi,\pi),
\]
which is independent of \(s\).  Since \(\delta\ge0\), the
unsquared curve is constant as well.  The positive-then-vanishing
obstruction follows immediately.
Perturbation invariance gives a stronger conclusion without
separability.  For \(s,t\in(0,1)\), put
\[
r=\frac{s(1-t)}{t(1-s)},
\qquad
c=\mathcal C(r,1,\ldots,1).
\]
A direct calculation gives
\[
c\oplus p_t^\pi=p_s^\pi,
\qquad
c\oplus p_t^\rho=p_s^\rho,
\]
and hence
\[
\delta(p_s^\pi,p_s^\rho)
=
\delta(p_t^\pi,p_t^\rho).
\]
For an \(f\)-divergence,
\[
D_f(p_s^\pi\Vert p_s^\rho)
=(1-s)D_f(\pi\Vert\rho).
\]
To verify the required positive permutation witness for every
finite non-affine convex generator, subtract a supporting affine
function \(a(t-1)\) at \(1\).  This does not change the divergence,
because
\[
\sum_i q_i a\!\left(\frac{p_i}{q_i}-1\right)=0,
\]
and gives the convex function
\[
\widetilde f(t)=f(t)-a(t-1)\ge0,
\qquad
\widetilde f(1)=0.
\]
Since \(f\) is non-affine, \(\widetilde f(r)>0\) for some
\(r\ne1\).  Choose a strictly positive \(\pi\) with two
coordinates \(i,j\) satisfying \(\pi_i/\pi_j=r\), and let
\(\sigma\) interchange those coordinates.  All unchanged coordinates
contribute \(\widetilde f(1)=0\), while
\[
D_f(\pi\Vert\sigma\pi)
=D_{\widetilde f}(\pi\Vert\sigma\pi)
=\pi_j\widetilde f(r)+\pi_i\widetilde f(r^{-1})>0.
\]
The matched-sink divergence converges to zero.  For every
\(\alpha>0\), \((D_f)^\alpha\) has the same
positive-then-vanishing property.
For Euclidean distance,
\[
\|p_s^\pi-p_s^{\sigma\pi}\|_2
=(1-s)\|\pi-\sigma\pi\|_2.
\]
If \(\sigma\pi\ne\pi\), Fisher--Rao distance and cosine
dissimilarity are also positive for every \(s<1\) and converge to
zero as \(s\uparrow1\): the former follows from
\[
\operatorname{BC}(p_s^\pi,p_s^{\sigma\pi})
=
s+(1-s)\operatorname{BC}(\pi,\sigma\pi)\uparrow1,
\]
with \(\operatorname{BC}(\pi,\sigma\pi)<1\), and the latter from
the exact formula
\[
1-\frac{s^2+(1-s)^2\langle\pi,\sigma\pi\rangle}
{\sqrt{s^2+(1-s)^2\|\pi\|_2^2}\,
 \sqrt{s^2+(1-s)^2\|\sigma\pi\|_2^2}},
\]
which is positive for \(s<1\) and tends to zero as
\(s\uparrow1\).  This completes part~\textnormal{(ii)(a)}.
\medskip
\textit{Part \textnormal{(ii)(b)}: exact split expansion.}
Fix \(k\ge2\), let \(E=E_{D,k}^{(j)}\), and put
\(n=D+k-1\).  Set
\[
x=\operatorname{clr}_D(p)-\operatorname{clr}_D(q)
\in\mathcal H_D,
\qquad
\mathbf1_D^\top x=0.
\]
The raw log-ratio vector \(\log p-\log q\) differs from \(x\) by
a constant vector.  After splitting, that constant is repeated in
every output coordinate and is again removed by centering.
Consequently, the split Aitchison square may be computed by taking
one copy of \(x_i\) for \(i\ne j\) and \(k\) copies of \(x_j\):
\begin{align*}
d_{A,n}(Ep,Eq)^2
&=
\sum_{i\ne j}x_i^2+kx_j^2
-
\frac{\left(\sum_{i\ne j}x_i+kx_j\right)^2}{n}\\
&=
\|x\|_2^2+
\frac{D(k-1)}{n}x_j^2,
\end{align*}
where the last equality uses
\(\sum_{i\ne j}x_i=-x_j\).
Let
\[
w_j=e_j-\frac1D\mathbf1_D.
\]
Since \(x\in\mathcal H_D\),
\[
x_j=\langle x,w_j\rangle,
\qquad
\|w_j\|_2^2=\frac{D-1}{D}.
\]
Cauchy--Schwarz therefore gives
\[
x_j^2\le\frac{D-1}{D}\|x\|_2^2,
\]
with equality if and only if \(x\) is proportional to \(w_j\).
It follows that
\[
1
\le
\frac{d_{A,n}(Ep,Eq)^2}{d_{A,D}(p,q)^2}
\le
1+\frac{D(k-1)}n\frac{D-1}{D}
=
\frac{kD}{D+k-1}.
\]
Every \(x\in\mathcal H_D\) is realizable by taking
\[
q=D^{-1}\mathbf1_D,
\qquad
p=\mathcal C(e^{x_0},\ldots,e^{x_{D-1}}).
\]
Thus the upper bound is attained by \(x=w_j\).  Since \(D\ge3\),
the lower bound is attained by \(x=e_a-e_b\) with
\(a,b\ne j\).  Since the denominator is positive definite on
\(\mathcal H_D\), the maximal squared Rayleigh quotient is the top
generalized eigenvalue of the two forms restricted to
\(\mathcal H_D\).  Taking square roots proves the asserted supremum.
\textit{The scale loophole.}
For each \(m\ge3\) there is \(\kappa_m>0\) such
that \(\delta_m=\kappa_m d_{A,m}\).  Let
\(M:\mathcal S^n\to\mathcal S^D\) merge the \(k\) split cells, so
that \((My)_j\) is their sum and every unsplit coordinate is left
unchanged.  Both \(E\) and \(M\) are admissible Markov maps and
\(ME=I\).  If the
family satisfied DPI, then for every \(p,q\),
\[
\delta_D(p,q)
=
\delta_D(MEp,MEq)
\le
\delta_n(Ep,Eq)
\le
\delta_D(p,q).
\]
Hence equality would hold throughout.  Choosing
\(x=e_a-e_b\), \(a,b\ne j\), for which the Aitchison ratio is
one, forces \(\kappa_n=\kappa_D\).  Choosing \(x=w_j\) would then
force \(\sqrt{kD/n}=1\), a contradiction.
\textit{A fixed-dimensional positive counterexample.}
Let
\[
p=\frac1{25}(16,1,8),
\qquad
q=\frac1{25}(8,1,16),
\]
and
\[
T=
\begin{pmatrix}
\frac{299}{600}&\frac1{300}&\frac1{300}\\
\frac{299}{600}&\frac1{300}&\frac1{300}\\
\frac1{300}&\frac{149}{150}&\frac{149}{150}
\end{pmatrix}.
\]
Every entry of \(T\) is positive and each column sums to one.  A
direct multiplication gives
\[
Tp=\frac1{7500}(2401,2401,2698),
\qquad
Tq=\frac1{7500}(1213,1213,5074).
\]
The input log-ratio vector is
\((\log2,0,-\log2)\), which is centered, so
\[
d_A(p,q)^2=2(\log2)^2.
\]
For a three-vector of the form \((a,a,b)\), its centered squared
norm is \(\frac23(a-b)^2\).  Hence
\[
d_A(Tp,Tq)^2
=
\frac23
\left[
\log\left(\frac{6091337}{1636337}\right)
\right]^2.
\]
Moreover,
\[
8\cdot6091337=48730696
>
44181099=27\cdot1636337,
\]
so \(6091337/1636337>27/8\).  Also
\[
\left(\frac{27}{8}\right)^4
=
\frac{531441}{4096}
>
\frac{524288}{4096}
=2^7,
\qquad
\frac74>\sqrt3.
\]
Therefore
\[
\frac{6091337}{1636337}
>
\frac{27}{8}
>
2^{7/4}
>
2^{\sqrt3},
\]
and thus
\[
d_A(Tp,Tq)^2
>
\frac23(\sqrt3\log2)^2
=
2(\log2)^2
=d_A(p,q)^2.
\]
This proves fixed-dimensional expansion without relying on rounded
numerical values.
\textit{The Markov-monotone side.}
For any \(f\)-divergence and any equal split,
\begin{align*}
D_f(Ep\Vert Eq)
&=
\sum_{i\ne j}q_i f(p_i/q_i)
+k\frac{q_j}{k}
f\!\left(\frac{p_j/k}{q_j/k}\right)\\
&=D_f(p\Vert q).
\end{align*}
If \(T\) is column-stochastic with no zero row, define
\[
\lambda_{ai}=\frac{T_{ai}q_i}{(Tq)_a}.
\]
For each \(a\), these weights are nonnegative and sum to one, and
\((Tp)_a/(Tq)_a=\sum_i\lambda_{ai}(p_i/q_i)\).  Jensen's
inequality, followed by summation over \(a\), gives
\begin{align*}
D_f(Tp\Vert Tq)
&\le
\sum_{a,i}T_{ai}q_i f(p_i/q_i)\\
&=D_f(p\Vert q).
\end{align*}
Writing
\(\operatorname{BC}(p,q)=\sum_i\sqrt{p_iq_i}\), an equal split
also gives
\[
\operatorname{BC}(Ep,Eq)=\operatorname{BC}(p,q),
\]
and hence exact invariance of
\(\operatorname{FR}(p,q)=2\arccos\operatorname{BC}(p,q)\).
For a general \(T\), Cauchy--Schwarz gives, row by row,
\[
\sqrt{(Tp)_a(Tq)_a}
\ge
\sum_iT_{ai}\sqrt{p_iq_i}.
\]
Summing over \(a\) yields
\[
\operatorname{BC}(Tp,Tq)
\ge
\operatorname{BC}(p,q).
\]
Since \(2\arccos(\cdot)\) is decreasing on \([0,1]\),
\[
\operatorname{FR}(Tp,Tq)
\le
\operatorname{FR}(p,q).
\]
Part~\textnormal{(ii)(a)} shows that these nondegenerate attenuating
families are not Sep-A.  No assertion that every Markov-monotone
family is matched-sink noninvariant is made or needed.  Generic
\(f\)-divergences need not be symmetric; they are used here as
divergence functionals and are not claimed to satisfy
\textnormal{(S3)}.
\medskip
\textit{Part \textnormal{(iii)}.}
Let \(N(z)=\operatorname{osc}(z)\).  Then
\[
N(-z)=N(z),
\qquad
N(\alpha z)=|\alpha|N(z),
\qquad
N(z+w)\le N(z)+N(w).
\]
Moreover, \(N(\log p-\log q)=0\) implies that \(p_i/q_i\) is
constant in \(i\); since \(p\) and \(q\) both sum to one, this
implies \(p=q\).  Finally,
\begin{align*}
d_H(p,r)
&=N\!\left((\log p-\log q)+(\log q-\log r)\right)\\
&\le d_H(p,q)+d_H(q,r).
\end{align*}
Thus \(d_H\) is a metric on the open simplex.
For perturbation, if \(Z_p(c)=\sum_i c_ip_i\), then
\[
\log\frac{(c\oplus p)_i}{(c\oplus q)_i}
=
\log\frac{p_i}{q_i}
+
\log\frac{Z_q(c)}{Z_p(c)}.
\]
The last term is independent of \(i\), so oscillation is unchanged.
Similarly, if \(S_p(\alpha)=\sum_i p_i^\alpha\), then
\[
\log\frac{(\alpha\odot p)_i}{(\alpha\odot q)_i}
=
\alpha\log\frac{p_i}{q_i}
+
\log\frac{S_q(\alpha)}{S_p(\alpha)},
\]
which proves
\[
d_{H,D}(\alpha\odot p,\alpha\odot q)
=|\alpha|d_{H,D}(p,q).
\]
The case \(\alpha=0\) is included because both powered
compositions are uniform.  Permutation invariance follows because
a permutation merely reorders the log ratios.
Let \(M\ge2\), and let \(E\in[0,\infty)^{M\times D}\) be a congruent Markov
embedding: its columns are stochastic, their supports are pairwise
disjoint, and it has no zero row.  For every output coordinate
\(a\), there is a unique input coordinate \(\iota(a)\) with
\(E_{a,\iota(a)}>0\), and
\[
\frac{(Ep)_a}{(Eq)_a}
=
\frac{p_{\iota(a)}}{q_{\iota(a)}}.
\]
Every input ratio occurs at least once, so the maximum and minimum
ratios are unchanged.  Therefore
\[
d_{H,M}(Ep,Eq)=d_{H,D}(p,q).
\]
Now let \(T\in[0,\infty)^{M\times D}\) be admissible.  Put
\(r_i=p_i/q_i\) and
\[
\lambda_{ai}=\frac{T_{ai}q_i}{(Tq)_a}.
\]
Then
\[
\frac{(Tp)_a}{(Tq)_a}=\sum_i\lambda_{ai}r_i
\]
is a convex combination of \(r_1,\ldots,r_D\).  Hence
\[
\min_i r_i
\le
\frac{(Tp)_a}{(Tq)_a}
\le
\max_i r_i,
\]
which proves \(d_{H,M}(Tp,Tq)\le d_{H,D}(p,q)\).
If \(T\) is entrywise positive, its projective diameter is
\[
\Delta(T)
=
\log
\max_{a,b,i,j}
\frac{T_{ai}T_{bj}}{T_{aj}T_{bi}}
<\infty.
\]
The Birkhoff--Hopf contraction theorem gives
\[
d_{H,M}(Tp,Tq)
\le
\tanh\!\left(\frac{\Delta(T)}4\right)d_{H,D}(p,q).
\]
Thus the coefficient
\(\tau(T)=\tanh(\Delta(T)/4)\) is strictly smaller than one for
each fixed positive \(T\).  It cannot be chosen uniformly over all
positive stochastic matrices: for
\[
T_\varepsilon
=(1-\varepsilon)I
+\varepsilon D^{-1}\mathbf1_D\mathbf1_D^\top,
\qquad 0<\varepsilon<1,
\qquad \varepsilon\downarrow0,
\]
the maps approach the identity.  Hence, for every fixed \(p\ne q\),
\[
\frac{d_{H,D}(T_\varepsilon p,T_\varepsilon q)}{d_{H,D}(p,q)}
\longrightarrow1,
\]
which rules out a common coefficient strictly below one.
To disprove sink separability, fix
\(q=D^{-1}\mathbf1_D\) and define
\[
p_{A,w}
=
\mathcal C\!\left(
e^A,e^w,e^{-w},\underbrace{1,\ldots,1}_{D-3}
\right),
\]
where the final block is empty when \(D=3\).
Then
\[
b_0(p_{A,w})-b_0(q)
=
\sqrt{\frac{D-1}{D}}A,
\]
while
\[
p_{A,w}^{(-0)}
=
\mathcal C(e^w,e^{-w},1,\ldots,1)
\]
is independent of \(A\).  If \(d_H\) were sink-separable at
coordinate \(0\), the function
\[
F(A,w)=d_H(p_{A,w},q)^2
\]
would have the form \(f(A)+g(w)\), and hence would have zero mixed
rectangle difference.  Normalization adds only a common constant
to the log-ratio vector, so
\[
F(A,w)
=
\operatorname{osc}(A,w,-w,0,\ldots,0)^2.
\]
At \(A,w\in\{0,1\}\),
\[
F(1,1)=4,
\quad F(0,1)=4,
\quad F(1,0)=1,
\quad F(0,0)=0.
\]
Thus
\[
F(1,1)-F(0,1)-F(1,0)+F(0,0)=-1\ne0,
\]
contradicting sink separability.
Finally, in CLR coordinates \(d_H\) is generated on
\(\mathcal H_D\) by \(N(x)=\operatorname{osc}(x)\).  The displayed
homogeneity and triangle inequality hold on \(\mathcal H_D\), and
if \(N(x)=0\), then \(x=c\mathbf1_D\); since
\(x\in\mathcal H_D\), this forces \(c=0\).  Thus \(N\) is a norm.
Let
\[
u=(1,-1,0,\ldots,0),
\qquad
v=(1,0,-1,0,\ldots,0).
\]
Then
\[
N(u)=N(v)=2,
\qquad
N(u+v)=3,
\qquad
N(u-v)=2.
\]
Therefore
\[
N(u+v)^2+N(u-v)^2
=13
\ne16
=2N(u)^2+2N(v)^2.
\]
The parallelogram law fails, so the Jordan--von Neumann theorem
implies that \(N\) is not induced by an inner product.
\end{proof}

\begin{remark}[Scope of the Kohlberg--Pratt comparison]
\label{rem:kohlberg-pratt-scope}
Kohlberg and Pratt prove, under the full hypotheses of their
characterization of projective metrics for which every strictly
positive linear map acts as a contraction, that the resulting metric
in a fixed dimension \(D\) has the form
\[
\mathfrak d_D=f_D\circ d_{H,D},
\]
where \(f_D:[0,\infty)\to[0,\infty)\) is continuous and strictly
increasing, with \(f_D(0)=0\).  This conclusion uses all the
hypotheses of their characterization; contraction alone does not
imply that \(\mathfrak d_D\) is a scalar multiple of \(d_{H,D}\).
If \(\mathfrak d_D\) is also powering-homogeneous, then
\[
f_D(\alpha t)=\alpha f_D(t)
\qquad(\alpha,t\ge0).
\]
Indeed, every \(t\ge0\) is realized as a Hilbert distance, since
\[
d_{H,D}\!\left(
\mathcal C(e^t,1,\ldots,1),D^{-1}\mathbf1_D
\right)=t.
\]
Setting \(t=1\) and writing \(u=\alpha\) gives
\[
f_D(u)=c_Du,
\qquad
u\ge0,
\qquad
c_D=f_D(1)>0.
\]
A dimension-independent constant requires an additional
cross-dimensional invariance assumption.
\end{remark}

\section{The scaffold set}
\label{app:scaffold}

\subsection{Scaffold-set definitions}
Every word in the exact lists is matched over tokenizer variants
(leading space/newline, capitalization, fused trailing comma/period),
and only single-token matches are kept, so $\scafS$ is a set of vocabulary
ids specific to each tokenizer, reported and hashable per run.

\textbf{\textsc{fixed} (default), 71 connectives $+$ 11 format
strings.} Reflection and uncertainty management (\emph{wait, hmm,
okay, ok, alright, well, oh, ah, actually, alternatively, instead,
anyway, right}); sequencing (\emph{first, firstly, second, secondly,
third, next, then, now, finally, lastly, meanwhile, overall});
inference (\emph{so, therefore, thus, hence, because, since,
consequently, accordingly}); contrast/addition (\emph{but, however,
although, though, yet, still, also, moreover, furthermore,
additionally, besides, similarly, likewise, nevertheless, nonetheless,
whereas}); self-dialogue markers frequent in R1-style traces
(\emph{let, lets, note, notice, recall, remember, check, verify,
confirm, yes, no, maybe, perhaps, indeed, sure, great, good, hold,
looking, considering, thinking}). Format strings: think delimiters,
chat-role delimiters, \texttt{\#\#\#\#}/\texttt{\#\#\#}/\texttt{\#\#},
\emph{step, answer, solution, final}; plus all pure
whitespace/punctuation tokens and tokenizer special tokens.
\textbf{Mathematical operators ($=,+,-,\times,/,\hat{\ },\%$) are
deliberately excluded}: in reasoning traces they carry computation, and
classifying them as scaffold would bake the conclusion of
\citet{candussio2026demystifying} into the definition instead of testing
it.

\textbf{\textsc{pos}.} \textsc{fixed} plus an 81-word closed-class list
(articles, prepositions, auxiliaries, pronouns, conjunctions), a
deterministic tagger-free proxy for a part-of-speech definition.

\textbf{\textsc{data}.} The most frequent non-numeric, non-lexical
token types of a corpus of the \emph{model's own} CoTs (frequency-ranked;
numerals and operator strings excluded by regex; cap $M=64$ ids), a
model-relative definition in the spirit of perplexity-based importance
\citep{pan2024llmlingua2}.

\subsection{Sensitivity protocol}
Every Stage-1 aggregate is recomputed under \textsc{fixed},
\textsc{pos}, \textsc{data}, and under \textsc{fixed} with $25\%$ of
the connective list dropped at random (3 draws), reporting (a) the
Jaccard overlap of the resulting id sets, (b) the stability of each
principal cell of Table~\ref{tab:frozen}, and (c) the realized scaffold
token share, which audits whether $\scafS$ is a minority class as
intended. Cells: the \textsc{fixed}/\textsc{pos}/\textsc{data}
comparison is measured in App.~\ref{app:additional}
(scaffold-definition sensitivity). The random list-drop ablation uses
a drop fraction of $0.25$ and seeds $\{0,1,2\}$, and the tier split of
Remark~\ref{rem:partition} uses \textsc{format}-only and \textsc{core}-only
scaffold sets, yielding the three-block formatting/strategic/content
refinement. Exploration analyses (Prop.~\ref{prop:exploration}) use
$\scafS$ with answer-bearing words (\emph{yes, no, answer, final,
solution}) excluded, per the answer-sufficiency residual. The identities themselves are
$\scafS$-independent, so sensitivity here concerns the \emph{empirical}
magnitudes only.

\section{Experimental details}
\label{app:expdetails}

\subsection{Target-checked position population (SYN-1, Figure~\ref{fig:wedge})}
$n=50{,}000$ positions per setting, 3 seeds, $m=30$ scaffold
coordinates. A fraction $f=0.20$ of positions are
\emph{boundary-eligible}: gate mass
$s\sim\mathrm{Beta}(6\bar s,6(1-\bar s))$ with $\bar s$ the swept
parameter; interior positions $s\sim\mathrm{Beta}(1.5,60)$ (mean
$\approx0.024$). Scaffold-internal entropy
$\HS=\log(m)\cdot\mathrm{Beta}(2.5,2.5)$. Content uses an explicit
effective-support mixture (an earlier tiny-shape Dirichlet
parameterization underflowed in float64 and silently produced spurious
uniform-content positions; the effective-support formulation below fixes this issue): each
position has $k_{\mathrm{eff}}$ plausible continuations with
$\mathrm{Dirichlet}(\mathbf{1})$ weights, with
$k_{\mathrm{eff}}$-mixtures $(0.72,0.14,0.08,0.06)$ over
$\{1\},\{2\},\{3..8\},\{9..64\}$ for interior positions and
$(0.55,0.25,0.15,0.05)$ over $\{1\},\{2\},\{3..8\},\{6..16\}$ at
boundaries (more decided: after the connective decision the
continuation is constrained). The stated target, a low-entropy majority
with a boundary-concentrated high-entropy minority in the spirit of
\citet{wang2025beyond,zhao2026shorthand}, is \emph{verified in each run}:
at $\bar s{=}0.4$, $65\%$ of positions
fall below $0.7$ nats, the top-20\% threshold is $1.38$ nats, and
boundary slots occupy $58\%$ of the raw top-20\% against a $20\%$ base
rate. These remain author-chosen simulation parameters; nothing here is
a model measurement. Flip rates use $n$ random position pairs restricted
to strict non-ties (the tie fraction is reported); retention uses
largest-first top-20\% sets with index tie-breaking;
``scaffold-driven'' means $\Hb+s\HS>(1-s)\HC$ at the position.

\subsection{Constructed collapse trajectories (SYN-2,
Figure~\ref{fig:collapse})}
One population ($\bar s=0.35$, $f=0.35$, content concentration
$\log\alpha_c\sim\mathcal{N}(\log 0.006,1.2^2)$) is drawn once and
transformed along 12 checkpoints by the nuisance group only:
\emph{phantom}: boundary-slot gate shift $\beta:0\to4.5$ logits with
scaffold powering $a_{\scafS}:1\to10$ and content untouched;
\emph{masked}: $\beta:0\to1.5$, scaffold powering $1\to0.18$
(diversification), content powering $1\to8$. Because the content
draws are literally reused, ``content never moved'' is exact in the
phantom arm, and Definition~\ref{def:attr} (pointwise over the aligned
population, hence exact) attributes $-0.49$ vs $0.00$ nats (phantom)
and $+0.42$ vs $-0.32$ nats (masked). Gate shifts and scaffold powering
are elements of the nuisance group $G_{\scafS}$; content powering is
not, and is used only in the masked arm, whose point is a genuine
content collapse. Content compositions use the same underflow-safe
effective-support representation as SYN-1.

\subsection{In-silico GRPO (SYN-3, Figure~\ref{fig:insilico})}
Vocabulary $10$ scaffold $+30$ content tokens; $4$ instances; episodes
of $T=7$ positions with style slots $\{0,2,3,5,6\}$ and content forks
$\{1,4\}$; each fork of each instance accepts \emph{two} correct
content tokens (drawn without replacement); reward $=\tfrac12$ per
solved fork $+\lambda$ per style slot filled with any scaffold token.
Policy: tabular logits per (instance, position); GRPO-style updates:
group size $8$, advantages group-normalized, base step $0.9$
(fork positions scaled $\times1.0$); optional entropy intervention adds
the exact gradient $\partial H/\partial z_a=-p_a(\log p_a+H)$ at fork
positions with coefficient $0.4$ (the Clip-Cov/KL-Cov class acts
similarly on high-covariance tokens \citep{cui2025entropy}).
Conditions: $\lambda\in\{0,0.35\}\times$ intervention
$\in\{\text{off},\text{on}\}$, 4 seeds each, 220 steps, evals every 10.
Exploration metric, computed exactly from the policy (no sampling):
\emph{worst-case preference-shift pass@8}
$=\min_{o\in\{1,2\}}\big[1-\big(1-\prod_{\text{forks}}p(\text{token}_o)\big)^{8}\big]$,
which is near its maximum iff both correct tokens stayed alive at every
fork. The per-step channel prediction of
Proposition~\ref{prop:dynamics} is compared with realized entropy
differences; median relative error $0.0\%$ without and $3.2\%$ with
the intervention arm (larger steps), reported exactly as such.

\subsection{Stage 1: frozen models}
\textbf{Models.} EleutherAI Pythia $\{70\text{M},160\text{M},410\text{M},
1\text{B},1.4\text{B}\}$ and GPT-2 (teacher-forced);
Qwen2.5-1.5B-Instruct and DeepSeek-R1-Distill-Qwen-1.5B (generation;
optional 7B replication). \textbf{Data.} GSM8K test
\citep{cobbe2021gsm8k} with calculator annotations
\texttt{<<...>>} stripped; MATH-500 \citep{hendrycks2021math};
optional AIME. $n=512$ (teacher) / $256$ (GSM8K gen) / $200$ (MATH
gen) problems, 3 data-subsample seeds. \textbf{Teacher mode} scores
solution positions only, conditioning on the question.
\textbf{Generation} uses the chat template with ``reason step by step
\dots \verb|\boxed{}|'', $T{=}0.6$, top-$p{=}0.95$ (R1 model card);
sanity accuracy is recorded. Channel statistics use RAW
full-vocabulary logits, with a teacher-forced re-forward fallback;
post-processed generation scores (top-$p$ inserts $-\infty$) are never used for
measurement. Headline statistics carry problem-clustered bootstrap
CIs. \textbf{Channels} are streamed from
full-vocabulary logits: with $A_B=\sum_{v\in B}p_v\log p_v$,
$H=-(A_{\scafS}+A_{\contC})$, $\HS=-A_{\scafS}/s+\log s$,
$\HC=-A_{\contC}/(1-s)+\log(1-s)$, two masked sums per position, no
logit storage; per-token scalars
$(s,H,\HS,\HC,\Hb,G,\text{block},\text{seq},\text{pos})$ are retained for Stage 3.
\textbf{Reported} (Table~\ref{tab:frozen}): means, channel shares, the
pointwise variance split
$\operatorname{Var}H=\operatorname{Var}G+\operatorname{Var}\HC
+2\operatorname{Cov}(G,\HC)$, fork-flip rate on random pairs,
top-$\rho$ retention/Jaccard ($\rho\in\{0.1,0.2\}$), scaffold-driven
and realized-scaffold shares of the raw top set.

\subsection{Stage 2: GRPO}
TRL's GRPO trainer \citep{vonwerra2020trl}, Qwen2.5-0.5B-Instruct,
GSM8K train prompts ($n=2048$), exact-match reward on
\verb|\boxed{}|/\texttt{\#\#\#\#}, group $8$, completion cap $512$,
lr $10^{-6}$, no KL penalty, 300 steps, bf16, seeds $\{0,1,2\}$
budget-permitting. Every 20 steps the policy is frozen and measured on
64 held-out problems: channel decomposition along one sampled CoT per
problem plus pass@1/pass@8 at $T{=}1.0$; consecutive evals additionally
record Definition~\ref{def:attr}. The same measurement is applied to
the full checkpoint series, with completed measurements retained after
interruptions.

\subsection{Stage 3: re-audits}
\textbf{Selection} consumes the recorded Stage-1 per-token statistics (no model needed).
\textbf{Compression}: by default on the model's OWN
generated CoTs ($T{=}0.6$) with any \verb|\boxed{...}| span deleted;
reference solutions have the final-answer
line stripped, so the gold answer never appears in the compressed
prompt. Every arm keeps EXACTLY the same number of tokens per
problem (budget $=$ the matched feasible budget across arms). The
final protocol runs ten arms: raw surprisal
$-\log p(y_t)$, raw entropy $H(p_t)$, receiver attention, content
surprisal $-\log \pC(y_t)=-\log p(y_t)+\log(1-s_t)$ and content
entropy $\HC_t$ (scaffold tokens ineligible), random scattered keep,
random contiguous block, head block, recency (last tokens), and the
uncompressed chain; budgets $\{0.25,0.5,0.75\}$ on R1-Distill-1.5B
GSM8K plus $0.5$ cells on R1-Distill MATH-500 and Qwen2.5-1.5B-It
GSM8K; $n{=}512$ problems per seed ($256$ on MATH-500), three seeds,
generation cached once per (model, dataset, seed). Answers are decoded
greedily from ``\dots\texttt{The answer is}'' after the compressed
chain; comparisons use the exact McNemar sign test on pooled
discordant pairs plus a problem-clustered bootstrap. Three audits
extend the protocol: answer survival is reported per positional arm
and budget as the fraction of kept text still containing the gold
answer; the factorial adds raw score with content-only eligibility and
gate-adjusted score with all tokens eligible; and the stripping
condition deletes every exact, comma-free, and numeric-form occurrence
of the gold answer plus the chain's final sentence before compression.
\textbf{Anchors}: sentence spans by regex; importance $=$ mean
attention received from all later tokens, layer/head-averaged, with
rows either raw or renormalized over content columns; Kendall $\tau$
and top-1 flips per problem. \textbf{Eviction}: token-level attention
received, both conventions, keep-set Jaccard at budgets
$\{0.3,0.5\}$. Anchors and eviction defaults: R1-Distill-Qwen-1.5B,
GSM8K, $n=128$, 3 seeds.

\subsection{Answer distributions and numerics}
For the answer-distribution analysis, $k=32$ samples per problem ($T{=}0.7$),
$n=64$ problems, answers clustered by normalized final string; we report
per-question diversity-verdict flip rates between keep- and
drop-the-majority conventions for base vs R1-distilled models.
\textbf{Numerics.} CPU analyses in float64; GPU logits cast to float32
and all channel quantities computed \emph{blockwise in log space}: with
logits $z$, $\log Z_{\scafS}{=}\operatorname{lse}(z|_{\scafS})$,
$\log Z_{\contC}{=}\operatorname{lse}(z|_{\contC})$,
$\log Z{=}\operatorname{lse}(z)$ give
$\log s=\log Z_{\scafS}-\log Z$ and $\log(1{-}s)=\log Z_{\contC}-\log Z$
exactly; $\HS$ and $\HC$ are entropies of the within-block log-softmaxes
$z|_{\scafS}-\log Z_{\scafS}$ and $z|_{\contC}-\log Z_{\contC}$; $\Hb$
uses $x\log x$ with its $0$-limit. No clamping of $s$ occurs and no
boundary convention is needed for the generic position: even where
$1-s$ underflows in float32 (reasoning models place numerically all
mass on scaffold ids at $1$--$3\%$ of positions), $\HC$ is the entropy
of the correct conditional distribution, computed stably from logit
differences. The estimator still flags the residual positions whose
content-block mass underflows beyond what the blockwise computation
can certify; these are excluded from rankings and their fraction
reported in the main text (at most $0.2\%$ teacher-forced, $4.9\%$
generated), so the exclusions of \S\ref{sec:protocol} and the
stability claim here describe the same estimator.  An
earlier draft used a clamped parameterization with an $\HC{:=}0$
convention at underflow; that convention is \emph{not} a continuity
limit ($\HC$ has no limit as $s\to1$:
$p_\varepsilon=(1{-}\varepsilon,\varepsilon/2,\varepsilon/2)$ keeps
$\HC=\log 2$ for every $\varepsilon$), so all generate-mode cells are
re-measured with the exact estimator; the convention-era estimates were
retained solely for audit comparisons.
All GPU experiments were run on a workstation equipped with two NVIDIA
GeForce RTX 4090 GPUs (24 GB of memory each) and 64 logical CPUs.
Runtimes on 2$\times$RTX 4090: Stage 1 $\approx$4--6 h (3 seeds),
Stage 2 $\approx$6--10 h per seed, Stage 3 $\approx$2--3 h.

\section{Additional results}
\label{app:additional}

\begin{figure}[t]
\centering
\includegraphics[width=0.8\textwidth]{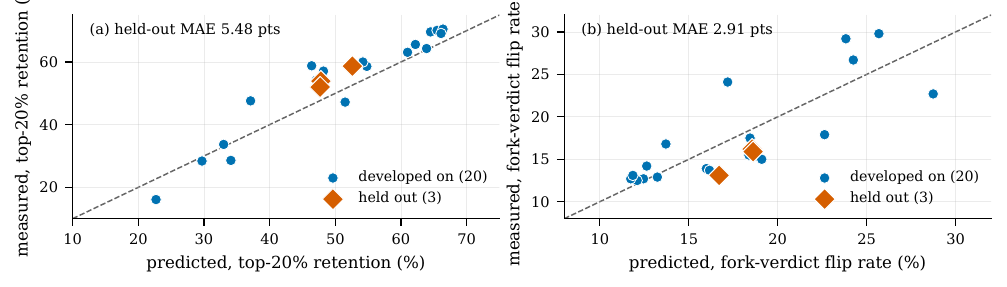}
\caption{\textbf{Forecast against measurement, twenty-three cells, three held out.} Each point is one (model, dataset, mode, scaffold) cell; the forecast uses only $r=\operatorname{corr}(H,\HC)$ through the Gaussian specialization of Theorem~\ref{thm:reversal}(iii); dashed line is $y=x$; diamonds are the cells first measured after registration. (a) Retention carries no fitted parameter. (b) Flip carries the one frozen constant $c=0.5988$.}
\label{fig:forecast}
\end{figure}

\begin{figure}[t]
\centering
\includegraphics[width=0.72\textwidth]{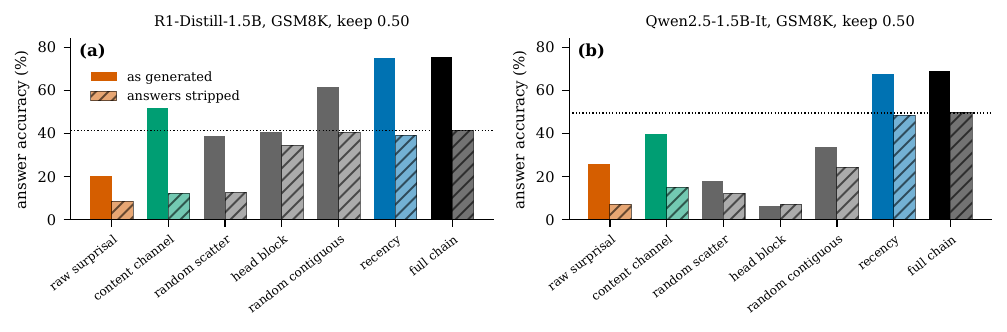}
\caption{\textbf{What stripping the restated answers does to the compression arms} (GSM8K, keep $0.50$, $512$ problems $\times$ three seeds; solid bars as generated, hatched bars after every gold-answer restatement and the final sentence are deleted; dotted line is the stripped full-chain ceiling). (a) On the distilled model the ceiling halves and the recency arm collapses onto the random contiguous block: its tail premium was largely copyable answer text. (b) On the instruct model a real tail premium survives stripping. Numbers, budgets, and clustered intervals in Table~\ref{tab:compress} and \S\ref{sec:decisions}.}
\label{fig:decisions}
\end{figure}

\textbf{Variance occupancy.} The pointwise identity $H=G+\HC$ gives
$\operatorname{Var}(H)=\operatorname{Var}(G)+\operatorname{Var}(\HC)
+2\operatorname{Cov}(G,\HC)$ per corpus; measured cells (mean over 3
seeds, all $\times\operatorname{Var}H$), as $\operatorname{Var}G$ /
$\operatorname{Var}\HC$ / $2\operatorname{Cov}$: Pythia-70M
$0.39/1.10/{-}0.49$; Pythia-160M $0.46/1.17/{-}0.63$; Pythia-410M
$0.47/1.31/{-}0.78$; Pythia-1B $0.47/1.25/{-}0.73$; Pythia-1.4B
$0.62/1.51/{-}1.13$; GPT-2 $0.41/1.17/{-}0.59$; and, for
the generating reasoning models (exact blockwise estimator): Qwen2.5-1.5B-It
$3.89/4.14/{-}7.02$, R1-Distill-1.5B (GSM8K) $1.96/2.34/{-}3.30$,
R1-Distill-1.5B (MATH) $2.37/2.75/{-}4.11$; each triple sums to
$1.00$ exactly, as the identity requires.  For reasoning-tuned models
the raw statistic is thus the \emph{smaller residual of two larger,
strongly anticorrelated channels}
($\operatorname{corr}(G,\HC)=-0.87/-0.77/-0.81$ respectively): per-token
surplus and content entropy each carry $2.0$--$4.1\times$ the variance
of $H$ itself, so a raw-$H$ monitor understates channel movement
several-fold, a measured form of the conflation claim, and the
mechanism behind the low top-20\% retention of Table~\ref{tab:frozen}.
(A convention-era draft reported $17$--$141\times$ here; those
magnitudes were artifacts of the clamped $\HC{:=}0$ estimator at
scaffold-certain positions; the exact estimator of
App.~\ref{app:expdetails} yields the figures above, reanalysis under the
earlier convention reproduces the old values, and the correction is reported explicitly.) In the target-checked
population at $\bar s=0.4$:
$\operatorname{Var}(G)=0.29\operatorname{Var}(H)$,
$\operatorname{Var}(\HC)=0.80\operatorname{Var}(H)$,
$2\operatorname{Cov}=-0.08\operatorname{Var}(H)$
($\operatorname{corr}(G,\HC)=-0.09$): the surplus carries roughly a
third of the raw statistic's variance and is nearly uncorrelated with
content entropy, so it reshuffles the raw ranking without inflating its
spread, which is why similar-looking $H$ and $\HC$ histograms
nonetheless select visibly different top-20\% sets
(Figure~\ref{fig:wedge}b).

Because $G\equiv H-\HC$ constrains the three terms above to sum to
$\operatorname{Var}(H)$, large near-cancelling terms are
\emph{equivalent} to $\operatorname{corr}(G,\HC)\to-1$; the finding is
the magnitudes, not the cancellation. The constraint-free view
decomposes $H=\Hb+s\HS+(1-s)\HC$ into three \emph{additive} terms.
The corresponding analysis reports $\operatorname{Var}(\Hb)$,
$\operatorname{Var}(s\HS)$, $\operatorname{Var}((1-s)\HC)$, and twice
the summed covariances, each as a share of $\operatorname{Var}(H)$,
together with the correlation between the scaffold side $\Hb+s\HS$ and
the content side $(1-s)\HC$. Measured (exact estimator, 3 seeds):
for the generating reasoning models the additive content term carries
$58$--$66\%$ of $\operatorname{Var}(H)$ (Qwen-It $0.66$, R1 GSM8K
$0.61$, R1 MATH $0.58$), the scaffold-side terms
$\operatorname{Var}(\Hb)+\operatorname{Var}(s\HS)$ carry $15$--$20\%$
($0.09{+}0.07$, $0.08{+}0.12$, $0.07{+}0.13$), and the summed
covariances $+17$--$21\%$, with
$\operatorname{corr}(\Hb{+}s\HS,\,(1{-}s)\HC)=+0.15$--$+0.16$; for
teacher-forced base models the content term is $80$--$95\%$ and the
scaffold-side terms $4$--$8\%$.  In this constraint-free coordinate
system no cancellation occurs: the scaffold side contributes an additional
$15$--$20\%$ of variance for tuned models ($2$--$4\times$ the base
share), and the mild \emph{positive} coupling replaces the mechanical
$-0.8$ correlation of the $(G,\HC)$ parameterization.

\textbf{Scaffold-definition sensitivity.} Measured on Pythia-410M
(GSM8K, teacher-forced, 3 seeds; App.~\ref{app:scaffold} protocol), same columns as
Table~\ref{tab:frozen}, now joined by the same
\textsc{pos}/\textsc{data} sweep measured on the two generating
reasoning models of Table~\ref{tab:frozen} (R1-Distill-1.5B and
Qwen2.5-1.5B-It, three seeds each, exact estimator). The \textsc{fixed} rows repeat the repaired
Table~\ref{tab:frozen} cells verbatim so the sweep is comparable
within one estimator:

{\centering\scriptsize\setlength{\tabcolsep}{3.5pt}
\begin{tabular}{lccccccccc}
\toprule
scaffold def. & $\bar s$ & $\bar H$ & gate\% & scaf\% & cont\% & flip\% & ret@20\% & scaf-drv\% & real-$\scafS$\% \\
\midrule
Pythia-410M \textsc{fixed} & $0.18{\scriptstyle\pm 0.0}$ & $2.25{\scriptstyle\pm 0.0}$ & $9.8{\scriptstyle\pm 0.0}$ & $8.1{\scriptstyle\pm 0.0}$ & $82.0{\scriptstyle\pm 0.1}$ & $12.5{\scriptstyle\pm 0.2}$ & $70.1{\scriptstyle\pm 0.1}$ & $0.4{\scriptstyle\pm 0.0}$ & $13.4{\scriptstyle\pm 0.2}$ \\
Pythia-410M \textsc{pos} & $0.38{\scriptstyle\pm 0.0}$ & $2.25{\scriptstyle\pm 0.0}$ & $13.1{\scriptstyle\pm 0.0}$ & $26.1{\scriptstyle\pm 0.1}$ & $60.7{\scriptstyle\pm 0.1}$ & $24.1{\scriptstyle\pm 0.2}$ & $47.2{\scriptstyle\pm 0.3}$ & $28.9{\scriptstyle\pm 0.6}$ & $29.9{\scriptstyle\pm 0.4}$ \\
Pythia-410M \textsc{data} & $0.25{\scriptstyle\pm 0.0}$ & $2.25{\scriptstyle\pm 0.0}$ & $13.0{\scriptstyle\pm 0.0}$ & $10.7{\scriptstyle\pm 0.1}$ & $76.4{\scriptstyle\pm 0.1}$ & $16.8{\scriptstyle\pm 0.1}$ & $63.1{\scriptstyle\pm 0.2}$ & $1.5{\scriptstyle\pm 0.2}$ & $26.0{\scriptstyle\pm 0.6}$ \\
\midrule
Qwen-It gen \textsc{fixed} & $0.26{\scriptstyle\pm 0.0}$ & $0.26{\scriptstyle\pm 0.0}$ & $20.0{\scriptstyle\pm 0.1}$ & $11.6{\scriptstyle\pm 0.1}$ & $68.5{\scriptstyle\pm 0.1}$ & $17.9{\scriptstyle\pm 0.2}$ & $47.6{\scriptstyle\pm 0.3}$ & $32.9{\scriptstyle\pm 0.2}$ & $22.8{\scriptstyle\pm 0.1}$ \\
Qwen-It gen \textsc{pos} & $0.40{\scriptstyle\pm 0.0}$ & $0.26{\scriptstyle\pm 0.0}$ & $25.8{\scriptstyle\pm 0.1}$ & $32.8{\scriptstyle\pm 0.3}$ & $41.5{\scriptstyle\pm 0.4}$ & $29.8{\scriptstyle\pm 0.3}$ & $28.4{\scriptstyle\pm 0.6}$ & $63.0{\scriptstyle\pm 0.4}$ & $45.4{\scriptstyle\pm 0.4}$ \\
Qwen-It gen \textsc{data} & $0.35{\scriptstyle\pm 0.0}$ & $0.25{\scriptstyle\pm 0.0}$ & $29.7{\scriptstyle\pm 0.7}$ & $15.1{\scriptstyle\pm 0.3}$ & $55.1{\scriptstyle\pm 0.9}$ & $29.2{\scriptstyle\pm 0.7}$ & $28.6{\scriptstyle\pm 1.3}$ & $49.6{\scriptstyle\pm 1.0}$ & $30.9{\scriptstyle\pm 0.9}$ \\
R1 gen \textsc{fixed} & $0.28{\scriptstyle\pm 0.0}$ & $0.56{\scriptstyle\pm 0.0}$ & $20.3{\scriptstyle\pm 0.1}$ & $17.5{\scriptstyle\pm 0.1}$ & $62.2{\scriptstyle\pm 0.1}$ & $15.5{\scriptstyle\pm 0.1}$ & $57.1{\scriptstyle\pm 0.3}$ & $34.6{\scriptstyle\pm 0.2}$ & $31.8{\scriptstyle\pm 0.1}$ \\
R1 gen \textsc{pos} & $0.48{\scriptstyle\pm 0.0}$ & $0.56{\scriptstyle\pm 0.0}$ & $22.6{\scriptstyle\pm 0.2}$ & $46.2{\scriptstyle\pm 0.1}$ & $31.3{\scriptstyle\pm 0.1}$ & $26.7{\scriptstyle\pm 0.2}$ & $33.7{\scriptstyle\pm 0.4}$ & $73.9{\scriptstyle\pm 0.1}$ & $63.8{\scriptstyle\pm 0.1}$ \\
R1 gen \textsc{data} & $0.29{\scriptstyle\pm 0.0}$ & $0.55{\scriptstyle\pm 0.0}$ & $26.2{\scriptstyle\pm 0.2}$ & $16.7{\scriptstyle\pm 1.0}$ & $57.1{\scriptstyle\pm 0.9}$ & $17.5{\scriptstyle\pm 0.5}$ & $55.6{\scriptstyle\pm 1.0}$ & $44.1{\scriptstyle\pm 2.1}$ & $39.6{\scriptstyle\pm 1.2}$ \\
\bottomrule
\end{tabular}\par}

\smallskip
\noindent As the paper states, the identities are
$\scafS$-independent while the empirical magnitudes are not: widening
$\scafS$ to the closed-class superset (\textsc{pos}) doubles $\bar s$
($0.18\to0.38$) and the flip rate ($12.5\to24.1\%$) and raises the
scaffold-driven share of the raw top-20\% to $28.9\%$ even for a base
model, while the model-relative \textsc{data} definition moves the
cells far less.  $\bar H$ is unchanged by construction.  All principal
comparisons in the main text therefore hold the definition
(\textsc{fixed}) constant across models; the base-vs-reasoning
contrast of Table~\ref{tab:frozen} is a statement at fixed
convention, not a convention-free constant.

\textbf{Base-model rows omitted from Table~\ref{tab:frozen}} (same columns, \textsc{fixed} scaffold, teacher-forced GSM8K, mean$\pm$sem over 3 seeds):

{\centering\scriptsize\setlength{\tabcolsep}{3.5pt}
\begin{tabular}{lccccccccc}
\toprule
model & $\bar s$ & $\bar H$ & gate\% & scaf\% & cont\% & flip\% & ret@20\% & scaf-drv\% & real-$\scafS$\% \\
\midrule
Pythia-70M & $0.16$ & $3.08$ & $7.9{\scriptstyle\pm 0.0}$ & $5.9{\scriptstyle\pm 0.0}$ & $86.2{\scriptstyle\pm 0.0}$ & $12.7{\scriptstyle\pm 0.1}$ & $70.5{\scriptstyle\pm 0.3}$ & $0.0{\scriptstyle\pm 0.0}$ & $9.1{\scriptstyle\pm 0.2}$ \\
Pythia-160M & $0.17$ & $2.85$ & $7.9{\scriptstyle\pm 0.0}$ & $6.2{\scriptstyle\pm 0.0}$ & $85.9{\scriptstyle\pm 0.0}$ & $14.2{\scriptstyle\pm 0.1}$ & $64.3{\scriptstyle\pm 0.2}$ & $0.0{\scriptstyle\pm 0.0}$ & $8.1{\scriptstyle\pm 0.1}$ \\
Pythia-1B & $0.18$ & $2.06$ & $10.1{\scriptstyle\pm 0.0}$ & $8.5{\scriptstyle\pm 0.0}$ & $81.4{\scriptstyle\pm 0.0}$ & $12.7{\scriptstyle\pm 0.2}$ & $69.6{\scriptstyle\pm 0.1}$ & $0.6{\scriptstyle\pm 0.0}$ & $15.4{\scriptstyle\pm 0.2}$ \\
\bottomrule
\end{tabular}\par}

\smallskip
\textbf{Stage-2b intervention arms} (Qwen2.5-0.5B, GSM8K, $300$ steps, matched update counts; the table \S\ref{sec:training} summarizes):

{\centering\scriptsize\setlength{\tabcolsep}{4pt}
\begin{tabular}{llcccccc}
\toprule
select-by & ent-target & pass@1\% & pass@8\% & $\bar H$ & $\bar\HC$ & $\bar s$ & seeds \\
\midrule
none & content & $8.9{\scriptstyle\pm 1.9}$ & $38.0{\scriptstyle\pm 4.5}$ & $0.65{\scriptstyle\pm 0.03}$ & $1.14{\scriptstyle\pm 0.01}$ & $0.23{\scriptstyle\pm 0.00}$ & 3 \\
none & none & $11.5{\scriptstyle\pm 1.9}$ & $39.6{\scriptstyle\pm 5.1}$ & $0.58{\scriptstyle\pm 0.02}$ & $1.06{\scriptstyle\pm 0.04}$ & $0.24{\scriptstyle\pm 0.01}$ & 3 \\
none & raw & $9.4{\scriptstyle\pm 0.9}$ & $43.2{\scriptstyle\pm 8.1}$ & $0.61{\scriptstyle\pm 0.03}$ & $1.16{\scriptstyle\pm 0.03}$ & $0.24{\scriptstyle\pm 0.01}$ & 3 \\
\bottomrule
\end{tabular}\par}

\smallskip
\textbf{Answer-distribution object.} The instance is specified and implemented, but is not measured in this submission:
majority mass $m$, $H$ vs
$H(\pi)$, and the keep/drop-the-majority verdict-flip rate, connecting
the pass@$k$-inversion literature \citep{yue2025limit} to the
convention question.

\section{MoE routers as compositions (cross-object generality II)}
\label{app:router}

Router softmax rows over $E$ experts are compositions with documented
dominant/shared-expert structure (routing collapse, hot/cold experts).
Designating the shared or empirically dominant expert(s) as the block
$\scafS$, Proposition~\ref{thm:chain} splits router entropy, the
load-balancing signal, into a gate channel (shared-vs-routed mass), a
within-block channel, and a routed-expert channel, and
Proposition~\ref{prop:invariance} identifies which
expert-specialization statistics are convention-independent.
The implemented analysis hooks every linear gate or router module
(OLMoE-1B-7B; Qwen1.5-MoE), streams router log-softmaxes over GSM8K
CoTs, and reports per-layer $(s,\Hb,\HS,\HC)$ plus the share of
adjacent-layer ``routing collapse'' transitions on which $H$ and $\HC$
disagree in sign. This instance is specified and implemented, and is not measured in this submission. The point of the section is
structural: one identity, three measurement literatures (attention
rows in the companion; CoT tokens and answers here; routers), one
convention audit.

\section{Relation to the companion manuscript}
\label{app:portmap}

For reviewers of both papers, stated once and plainly. The companion
studies attention rows, whose sink is a single canonical coordinate, and
proves sink/content separation for similarity, entropy, taxonomies, and
pruning. This paper studies next-token
policies, whose ``sink'' is a chosen \emph{set} with internal structure,
and shares with the companion exactly the static skeleton: the grouping
identity (there $H=\Hb(s)+(1-s)H(\pi)$, here with the extra $s\HS$
term), the partition split of $\dA$, and a reversal theorem. Everything
that makes this paper's principal claims, the leakage bound and its
tight coin example, the routing information $\kappa_t$ and the
first-passage estimand $q_t$, the optimizer-resolved dynamics of
Proposition~\ref{prop:dynamics} and Corollary~\ref{cor:reinforce}, and
the intervention protocol of Stage 2b, has no attention-row
counterpart, because attention rows do not train against rewards and do
not carry answers. We keep the shared machinery deliberately identical
so that a reader of either paper can audit the other, and we flag every
ported statement as classical or ported at the point of use.
Two statements are the companion's general theorems reproduced here
with full proofs because the CoT protocol uses them directly:
Proposition~\ref{prop:unmatched-sink-regimes} (unmatched-sink
regimes, applied to unlocked-gate checkpoint drift) and
Theorem~\ref{thm:no-free-lunch-corrected} (no free lunch, applied to
cross-tokenizer comparisons); neither is claimed as new to this paper,
and the scaffold-\emph{set} content in \S\ref{sec:instruments} is
the exact lumping reduction that makes them apply.

\end{document}